 \documentclass[final,3p,times]{elsarticle}

\usepackage{amsfonts}
\usepackage{graphicx}
\usepackage{times}
\usepackage{color}
\usepackage{amssymb}
\usepackage{natbib}
\usepackage{amsmath}
\usepackage{bm}
\usepackage{bbold}
\usepackage[bbgreekl]{mathbbol}
\usepackage{breqn}
\usepackage{algorithm}
\usepackage{algpseudocode}
\usepackage{accents}
\usepackage{multicol}
\usepackage{subcaption}

\font\ppppppcarac=ptmr8y at 4pt
\font\pppppcarac=ptmr8y at 5pt
\font\ppppcarac=ptmr8y at 6pt
\font\pppcarac=ptmr8y at 7pt

\font\pcarac=ptmr8y at 9pt

\newcommand{\bfA}{{\bm{A}}}

\newcommand{\bfG}{{\bm{G}}}
\newcommand{\bfH}{{\bm{H}}}

\newcommand{\bfL}{{\bm{L}}}

\newcommand{\bfQ}{{\bm{Q}}}

\newcommand{\bfU}{{\bm{U}}}
\newcommand{\bfV}{{\bm{V}}}
\newcommand{\bfW}{{\bm{W}}}
\newcommand{\bfX}{{\bm{X}}}
\newcommand{\bfY}{{\bm{Y}}}

\newcommand{\bfzero}{{ \hbox{\bf 0} }}

\newcommand{\bfa}{{\bm{a}}}
\newcommand{\bfb}{{\bm{b}}}

\newcommand{\bff}{{\bm{f}}}

\newcommand{\bfh}{{\bm{h}}}

\newcommand{\bfq}{{\bm{q}}}

\newcommand{\bfu}{{\bm{u}}}
\newcommand{\bfw}{{\bm{w}}}
\newcommand{\bfx}{{\bm{x}}}
\newcommand{\bfy}{{\bm{y}}}
\newcommand{\bfz}{{\bm{z}}}

\newcommand{\bfalpha}{{\bm{\alpha}}}
\newcommand{\bfbeta}{{\bm{\beta}}}

\newcommand{\bfeta}{{\bm{\eta}}}

\newcommand{\bflambda}{{\bm{\lambda}}}

\newcommand{\bfomega}{{\bm{\omega}}}

\newcommand{\bfzeta}{{\bm{\zeta}}}

\newcommand{\AAeclair}{{\mathbb{A}}}

\newcommand{\CC}{{\mathbb{C}}}

\newcommand{\FF}{{\mathbb{F}}}

\newcommand{\HH}{{\mathbb{H}}}

\newcommand{\LL}{{\mathbb{L}}}
\newcommand{\MM}{{\mathbb{M}}}
\newcommand{\NN}{{\mathbb{N}}}

\newcommand{\RR}{{\mathbb{R}}}

\newcommand{\UU}{{\mathbb{U}}}
\newcommand{\VV}{{\mathbb{V}}}
\newcommand{\WW}{{\mathbb{W}}}

\newcommand{\ff}{{\mathbb{f}}}
\newcommand{\nn}{{\mathbb{n}}}

\newcommand{\sigmasigma}{{\mathbb{\sigma}}}
\newcommand{\epsilonepsilon}{{\mathbb{\epsilon}}}

\DeclareMathAlphabet{\mathonebb}{U}{bbold}{m}{n}
\def\11{{\ensuremath{\mathonebb{1}}}}

\newcommand{\curB}{{\mathcal{B}}}
\newcommand{\curD}{{\mathcal{D}}}

\newcommand{\curG}{{\mathcal{G}}}
\newcommand{\curH}{{\mathcal{H}}}

\newcommand{\curP}{{\mathcal{P}}}

\newcommand{\curT}{{\mathcal{T}}}

\newcommand{\curV}{{\mathcal{V}}}

\newcommand{\ad}{{\hbox{{\ppppcarac ad}}}}
\newcommand{\pad}{{\hbox{{\pppppcarac ad}}}}
\newcommand{\st}{{\hbox{{\ppppcarac st}}}}
\newcommand{\SB}{{\hbox{{\pppppcarac SB}}}}

\newcommand{\pexp}{{\hbox{{\ppppcarac targ}}}}

\newcommand{\pobs}{{\hbox{{\pppcarac obs}}}}
\newcommand{\ppobs}{{\hbox{{\ppppcarac obs}}}}
\newcommand{\ppost}{{\hbox{{\ppppcarac post}}}}
\newcommand{\pppost}{{\hbox{{\pppppcarac post}}}}
\newcommand{\pQA}{{\hbox{{\ppppcarac QA}}}}

\newcommand{\cov}{{\hbox{{\textrm cov}}}}

\newcommand{\perr}{{\hbox{{\pcarac err}}}}

\newcommand{\pmax}{{\hbox{{\ppppcarac max}}}}

\newcommand{\pMC}{{\hbox{{\pppppcarac MC}}}}

\newcommand{\pPCA}{{\hbox{{\pppppcarac PCA}}}}

\newcommand{\tr}{{\hbox{{\textrm tr}}}}

\newcommand{\error}{{\hbox{{err}}}}

\newcommand{\psol}{{\hbox{{\ppppcarac sol}}}}
\newcommand{\ppsol}{{\hbox{{\ppppppcarac sol}}}}

\newcommand{\wien}{{\hbox{{\ppppcarac wien}}}}
\newcommand{\prelax}{{\hbox{{\ppppcarac relax}}}}

\newcommand{\bfHbflambdai}{{\bfH_{\!\bflambda^{\,i}}}}

\newcommand{\RRNr}{{{\RR^{N_r}}}}
\newcommand{\RRnu}{{{\RR^{\nu}}}}
\newcommand{\RRnx}{{{\RR^{n_x}}}}
\newcommand{\RRnq}{{{\RR^{n_q}}}}
\newcommand{\RRnw}{{{\RR^{n_w}}}}

\newcommand{\varphitilde} {{\underaccent{\thicksim}{\varphi}}}
\newcommand{\phitilde} {{\underaccent{\thicksim}{\phi}}}
\newcommand{\sigmatilde} {{\underaccent{\thicksim}{\sigma}}}

\usepackage{mathrsfs}
\newcommand{\curC}{{\mathscr{C}}}       % curvilinear font different from \matcal C

\font\bf=ptmb8y at 10pt
\font\pbf=ptmb8y at 7pt
\font\vcarac=ptmr8y at 10pt
\font\pvcarac=ptmr8y at 7pt
\newcommand{\bfv}{{\hbox{\bf{v}}}}      % boldface lower case letter \bfv must be very different from greek letter \nu
\newcommand{\pbfv}{{\hbox{\pbf{v}}}}    % boldface lower case letter \bfv must be very different from greek letter \nu
\newcommand{\vc}{{\hbox{\vcarac{v}}}}   % lower case letter \bfv must be very different from greek letter \nu
\newcommand{\pvc}{{\hbox{\pvcarac{v}}}} % ower case letter \bfv must be very different from greek letter \nu
\newcommand{\bfi}{{\hbox{\bf{i}}}}
\newcommand{\pbfi}{{\hbox{\pbf{i}}}}
\newcommand{\bfj}{{\hbox{\bf{j}}}}
\newcommand{\pbfj}{{\hbox{\pbf{j}}}}
\newtheorem{theorem}{Theorem}
\newdefinition{definition}{Definition}
\newtheorem{lemma}{Lemma}

\newproof{proof}{Proof}
\newdefinition{remark}{Remark}
\newdefinition{hypothesis}{Hypothesis}
\newdefinition{notation}{Notation}
\newproof{example}{Example}
\numberwithin{equation}{section}

\journal{arXiv}
\begin{document}

\begin{frontmatter}

\title{Probabilistic learning constrained by realizations using a weak \\
  formulation of Fourier transform of probability measures}

%\titlerunning{Probabilistic learning constrained by realizations}

\author[1]{Christian Soize \corref{cor1}}
\ead{christian.soize@univ-eiffel.fr}
\cortext[cor1]{Corresponding author: C. Soize, christian.soize@univ-eiffel.fr}
\address[1]{Universit\'e Gustave Eiffel, MSME UMR 8208 CNRS, 5 bd Descartes, 77454 Marne-la-Vall\'ee, France}

%\authorrunning{Christian Soize}

\begin{abstract}
This paper deals with the taking into account a given set of realizations as constraints in the Kullback-Leibler minimum principle, which is used as a probabilistic learning algorithm. This permits the effective integration of data into predictive models.
We consider the probabilistic learning of a random vector that is made up of either a  quantity of interest (unsupervised case) or the couple of the quantity of interest and a control parameter (supervised case). A training set of independent realizations of this random vector is assumed to be given and to be generated with a prior probability measure that is unknown. A target set of realizations of the QoI is available for the two considered cases. The framework is the one of non-Gaussian problems in high dimension. A functional approach is developed on the basis of a weak formulation of the Fourier transform of probability measures (characteristic functions). The construction makes it possible to take into account the target set of realizations of the QoI in the Kullback-Leibler minimum principle. The proposed approach allows for estimating the posterior probability measure of the QoI (unsupervised case)  or of the posterior joint probability measure of the QoI with the control parameter (supervised case). The existence and the uniqueness of the posterior probability measure is analyzed for the two cases. The numerical aspects are detailed in order to facilitate the implementation of the proposed method. The presented application in high dimension demonstrates the efficiency and the robustness of the proposed algorithm.
\end{abstract}


\end{frontmatter}

\section{Introduction}
\label{sec:Section1}
This paper deals with a probabilistic learning inference that permits the effective integration of data (target set) into predictive models. The target set is constituted of realizations/samples of the quantity of interest (QoI) and the training set is constituted of a small number of points, each point being a realization of the pair made up of the random QoI (output) and the random control parameter (input).
Taking into account constraints in learning algorithms remains a very important question and an active research topic. Bayesian updating provides a rational framework for integrating data into predictive models (see \cite{Bernardo2000,Kennedy2001,Spall2005,Congdon2007,Carlin2008,Gentle2009,Marin2012,Givens2013,Scott2016,Ghanem2017} for general aspects,
 \cite{Kaipio2005,Marzouk2007,Stuart2010,Soize2011,Matthies2016,Bilionis2017,Dashti2017,Arnst2017,Spantini2017,Picchini2018,Perrin2020} for specific aspects related to statistical inverse problems, \cite{Shen2012,Depraetere2017} for variational Bayesian methods,
\cite{Golightly2006} for Bayesian sequential inference, or \cite{Fearnhead2006} for Bayesian inference for changepoint problems).
Bayesian inferences have also been considered in the framework of machine learning \cite{Neil2007,Sambasivan2020} and probabilistic learning for small data sets and in high dimension \cite{Soize2020b}. Bayesian inference is therefore a powerful statistical tool for integrating raw data but requires that targets be given in the form of realizations, which is not the hypothesis introduced in  this paper. Note also that Bayesian inference can remains tricky to use \cite{Owhadi2015}, in particular for the high dimension.
However, in many instances, relevant information is available in the form of sample statistics, such as statistical moments, rather than raw data; this is the case when the statistical moments have been estimated with realizations (samples) that are no longer available. In these settings, the Kullback-Leibler divergence minimum principle \cite{Kullback1951,Kapur1992,Cover2006,Gentle2009} can be used for estimating a posterior probability measure  given its prior probability measure and the constraints related to the statistical moments.
This principle has extensively been used over the last three decades for imposing constraints in the framework of learning with statistical models (see for instance \cite{Kapur1992,Vasconcelos2004,Zhang2007,Cappe2013,Saleem2018}), in particular for reinforcement learning \cite{Filippi2010} and for probabilistic learning \cite{Soize2020a,Soize2021a}). However, the use of this principle requires that the constraints (related to the target set) be expressed as the mathematical expectation of a random variable that is the transformation  of the quantity of interest by a measurable mapping.

In this paper,  we present a novel method, which makes it possible to use the Kullback-Leibler divergence minimum principle when the constraints are not defined by statistical moments but when a target set of realizations is directly integrate to define the constraints. We then obtain  a  probabilistic learning algorithm that allows for integrating raw data into predictive models.
\subsection{Framework of the considered problem, objectives of the paper, and methodology proposed}
\label{sec:Section1.1}
{\textit{(i) First case referred as the unsupervised case}}.
The quantity of interest is a $\RRnq$-valued random variable $\bfQ$, defined on a probability space $(\Theta,\curT,\curP)$, whose prior probability measure is $P_\bfQ(d\bfq)$ on $\RRnq$. This prior probability measure is unknown but is the underlying probability measure that has been used to generate the training set $D_d=\{\bfq_d^1,\ldots,\bfq_d^{N_d}\}$ constituted of $N_d$ independent realizations $\{\bfq_d^j\in\RRnq , j=1,\ldots , N_d\}$ of $\bfQ$ (the subscript "$d$" is introduced to reference the "data" of the training set).
It is assumed that $n_q$ is big (high-dimension problem). For instance, $\bfq_d^j$ can be the realizations of the discretization of a random field indexed by a bounded part of $\RR^d$ with $d \geq 2$.
Related to $\bfQ$, a target set $D_\pexp= \{\bfq_\pexp^1,\ldots , \bfq_\pexp^{N_r}\}$ is given, constituted of $N_r$ given points $\bfq_\pexp^r$ in $\RRnq$, which are $N_r$ independent realizations of a $\RRnq$-valued random variable $\bfQ_\pexp$ defined on $(\Theta,\curT,\curP)$, independent of $\bfQ$, whose probability measure $P_\bfQ^\pexp$ of $\bfQ_\pexp$ is assumed to be unknown.
Giving the training set $D_d = \{\bfq_d^1,\ldots , \bfq_d^{N_d}\}$ of $\bfQ$ and the target set $D_\pexp= \{\bfq_\pexp^1,\ldots , \bfq_\pexp^{N_r}\}$ of $\bfQ_\pexp$,  the Kullback-Leibler divergence will allow for identifying the probability measure $P_\bfQ^\ppost$ that  is closest to $P_\bfQ(d\bfq)$ while satisfying the constraint defined by $D_\pexp$. The measure $P_\bfQ^\ppost$, which is the measure updated with the constraint, will be called the posterior probability measure of the $\RRnq$-valued random variable $\bfQ_\ppost$ defined on $(\Theta,\curT,\curP)$. The probabilistic learning thus consists in using a MCMC algorithm for generating $N$ realizations $\{\bfq_\ppost^\ell,\ell=1,\ldots , N\}$ of $\bfQ_\ppost$.
Regarding the resampling of a probability measure with MCMC algorithms, it should also be noted that, when the available training set is composed of a small number of points, suitable algorithms should be used like those which have been specifically developed to deal with the case of small data (see \cite{Soize2016,Perrin2018a,Farhat2019,Ghanem2020,Guilleminot2020,Soize2020c,Arnst2021,Soize2021a,Soize2022a} for data-driven problems and \cite{Ghanem2018a,Ghanem2019,Capiez2022} for optimization problems).\\

{\textit{(ii) Second  case referred as the supervised case}}.
The quantity of interest is the above $\RRnq$-valued random variable $\bfQ$ and there is a  control parameter that is a $\RRnw$-valued random variable $\bfW$. The random variables $\bfQ$ and $\bfW$ are defined on the probability space $(\Theta,\curT,\curP)$, whose prior joint probability measure is $P_{\bfQ,\bfW}(d\bfq,d\bfw)$ on $\RRnq\times \RRnw$.
As for the unsupervised case, this prior joint probability measure is unknown but is the underlying probability measure that has been used to generate the training set $D_d=\{\bfx_d^1,\ldots,\bfx_d^{N_d}\}$ constituted of $N_d$ independent realizations, $\{\bfx_d^j = (\bfq_d^j,\bfw_d^j), j = 1,\ldots , N_d \}$ of the $\RRnx$-valued random variable $\bfX=(\bfQ,\bfW)$ with $n_x=n_q+n_w$. The probability measure of $\bfX$ is $P_\bfX(d\bfx) = P_{\bfQ,\bfW}(d\bfq,d\bfw)$. It is assumed that $n_q$ and $n_w$ are big (high dimension problem).
This supervised case can correspond to $\bfQ=\bff(\bfW)$ in which $\bff$ is an unknown  measurable mapping from $\RRnw$ into $\RRnq$
or to $\bfQ=\bff(\bfW, \bfU)$ in which $\bff$ is also an unknown  measurable mapping from $\RRnw\times\RR^{n_u}$ into $\RRnq$ and where $\bfU$ is an uncontrolled $\RRnu$-valued random variable defined on $(\Theta,\curT,\curP)$. In the first case, $\bfq_d^j = \bff(\bfw_d^j)$ and $P_{\bfQ,\bfW}(d\bfq,d\bfw)$ has no density with respect to $d\bfq\otimes d\bfw$, while in the second case, $\bfq_d^j = \bff(\bfw_d^j, \bfu_d^j)$ and $P_{\bfQ,\bfW}(d\bfq,d\bfw)$ can have a density.
As for the unsupervised case, we consider a given target set $D_\pexp= \{\bfq_\pexp^1,\ldots , \bfq_\pexp^{N_r}\}$ for the quantity of interest, constituted of $N_r$ independent realizations of the $\RRnq$-valued random variable $\bfQ_\pexp$ that is independent of $\bfQ$.
Note that no target realization is given for the control variable $\bfW$. If we gave target realizations for $\bfW$, which would amount to giving ourselves a target set of realizations for $\bfX$, then in terms of the methodology presented in this paper, we would be in a situation similar to that of the unsupervised case. In the supervised case that we consider here, the considered system is under-observed with respect to the given target set of realizations.
Similarly to the unsupervised case, giving the training set $D_d = \{\bfx_d^1,\ldots , \bfx_d^{N_d}\}$ of $\bfX$ and the target set $D_\pexp= \{\bfq_\pexp^1,\ldots , \bfq_\pexp^{N_r}\}$ of $\bfQ_\pexp$, the Kullback-Leibler divergence will allow for identifying the probability measure $P_\bfX^\ppost$ that  is closest to $P_\bfX(d\bfx)$ while satisfying the constraint defined by $D_\pexp$. The measure $P_\bfX^\ppost$, which is the measure updated with the constraint on $\bfQ$, will be called the posterior probability measure of the $\RRnq\times\RRnw$-valued random variable $(\bfQ_\ppost,\bfW_\ppost)$ defined on $(\Theta,\curT,\curP)$. The probabilistic learning thus consists in using a MCMC algorithm for generating $N$ realizations $\{\bfx_\ppost^\ell,\ell=1,\ldots , N\}$ of $\bfX_\ppost$, that is to say, $N$ realizations $\{(\bfq_\ppost^\ell,\bfw_\ppost^\ell), \ell=1,\ldots , N\}$ of $(\bfQ_\ppost,\bfW_\ppost)$.
\subsection{Novelty of the paper}
\label{sec:Section1.2}
In this paper, we propose to use the Kullback-Leibler minimum principle to estimate the closest probability measure to a prior measure, which is  indirectly defined by giving a training dataset, under the constraint defined by a set of realizations for which the statistical moments cannot be estimated and therefore are assumed to be unknown. As the considered problem is in high dimension and as the target corresponds  to given realizations from which statistics such as high-order statistical moments cannot be estimated,  this problem is not trivial at all and requires the development of an appropriate approach. We therefore propose a novel functional method, which allows the target set of the realizations to be integrated as a constraint imposed in the form of a mathematical expectation. The functional approach consists in constructing and analyzing a weak formulation of the Fourier transform of the probability measure and to derive from it a finite representation of the functional constraint.
\subsection{Organization of the paper}
\label{sec:Section1.3}
All the developments given in this paper will be presented within the framework of the supervised case. Given the proposed approach, the unsupervised case follows immediately. This paper is organized in three parts.

The first part (Sections~\ref{sec:Section2} to \ref{sec:Section5}) is devoted to the formulation and the construction of a finite representation of the functional constraint.
Section~\ref{sec:Section2} deals with the scaling and the reduced representation of random vector $\bfX$ for which the realizations are the points of the training set, which allows for constructing a normalized random variable $\bfH$ with values in $\RR^\nu$ with $\nu < N_d \leq n_x$.
In Section~\ref{sec:Section3}, Definition~\ref{definition:1} and Lemma~\ref{lemma:1}  defined the functional constraint as an equality of the Fourier transform of the probability measures of $\bfQ$ and $\bfQ_\pexp$ and then of the random variable $\bfH$ and $\bfH_\pexp$, in which $\bfH_\pexp$ is the "projection"  of $\bfQ_\pexp$ on the model.
Section~\ref{sec:Section4} is devoted to the weak formulation of the functional constraint imposed to random variable $\bfH$ giving realizations of $\bfH_\pexp$. Under adapted mathematical hypotheses for covering a large part of applications, Theorem~\ref{theorem:1} (proven with the help of three Lemmas) gives the required mathematical results that are necessary to construct the weak formulation (Definition~\ref{definition:2}) of the functional constraint defined on the space $\curH_1= L^1(\RRnu,\CC)\cap L^2(\RRnu,\CC)$.
In Section~\ref{sec:Section5}, we present the construction and the analysis of a finite representation of the functional constraint derived from the weak formulation, which is restricted to a Hilbert space $\curH_{1,\,\mu}$  that is a subset of $\curH_1$ in which $\mu =p_{\nu}(\bfv)\, d\bfv$ is a Gaussian probability measure on $\RRnu$.
Theorem~\ref{theorem:2} studies the Fourier transform $\hat\varphi$ of a function $\varphi$ in $\curH_{1,\,\mu}$, which is a $\CC$-valued analytic function on $\RRnu$ and which belongs to $\curH_0=\curC_0\cap L^2(\RRnu,\CC)$  (in which $\curC_0$ is the space of all the $\CC$-valued continuous functions on $\RRnu$, which go to zero at infinity).  While the considered weak formulation of the functional constraint is posed on a Hilbert space, the Hilbertian structure leads naturally to introduce a Hilbertian basis used to construct  a finite representation of the weak formulation of the functional constraint. Given the fact that the measure $\mu$ related to $\curH_{1,\,\mu}$ is Gaussian, the multi-dimensional Hermite polynomials could be used. However, the multi-index is in high dimension $\nu$ and consequently, the curse to dimensionality prevents using this type of finite representation. Based on Theorem~\ref{theorem:2}, Lemmas~\ref{lemma:5} and \ref{lemma:6} give sought construction of the functional family of functions in $\curH_{1,\,\mu}$, which allows for constructing the finite representation that is explicitly  described in Definition~\ref{definition:4} and that uses the realizations of $\bfQ_\pexp$ (the points of target set $D_\pexp$).
Lemma~\ref{lemma:7} gives an important property of the constructed finite representation, which will allow for analyzing the existence and the uniqueness of the posterior probability measure.
The first part ends with a numerical illustration of the behavior of the finite representation of functional constraint that is proposed.

The second part of this paper corresponds to Section~\ref{sec:Section6} in which we present the methodology to construct the posterior probability measure based on the use of the Kullback-Leiber minimum principle with the prior model and the target set. This methodology is similar to the one we have used in \cite{Soize2020a,Soize2022bb}, but for which the constraints are now the one presented in Section~\ref{sec:Section5}. Thus the mathematical proofs are adapted and modified because the hypotheses are no longer the same. The finite representation of the weak formulation of the functional constraint is taken into account by introducing a vector-valued Lagrange multiplier $\bflambda$. The posterior probability measure is constructed as the limit of a sequence of random variables $\{\bfH_\bflambda\}_\bflambda$ indexed by $\bflambda$. Theorems~\ref{theorem:3} and \ref{theorem:4} give the explicit construction of the probability measure of $\bfH_\bflambda$ and its MCMC generator based on the nonlinear stochastic dissipative Hamiltonian system studied in \cite{Soize1994}. This second part ends with the iterative algorithm for computing the optimal value of $\bflambda$  and gives elements for its numerical implementation.

The last part, Section~\ref{sec:Section7}, is devoted to a numerical illustration of the supervised case for which the training set  $D_d$ is made up of $N_d$ independent realizations  $\bfx^j\!\! = \!(\bfq_d^j,\bfw_d^j) \in \RRnx \!\! = \!\RRnq\!\times\!\RRnw$ of random variable $\bfX=(\bfQ,\bfW)$ for which $n_x=430\, 098$, $n_q=10\, 098$, $n_w= 420\, 000$, and $N_d\in\{100,200,300,400\}$. The target set $D_\pexp$ is made up of $N_r$ independent realizations $\bfq_\pexp^r \in\RRnq$ of random variable $\bfQ_\pexp$ for which $N_r \in [50\, , N_\pexp]$ with $N_\pexp\in\{100,200,300,400\}$. As we will see, we will also give a lighting  on the associated unsupervised case to this supervised case.

\bigskip

\noindent{\textbf{Notations}

\noindent $x,\eta$: lower-case Latin or Greek letters are deterministic real variables.\\
$\bfx,\bfeta$: boldface lower-case Latin or Greek letters are deterministic vectors.\\
$X$: upper-case Latin letters are real-valued random variables.\\
$\bfX$: boldface upper-case Latin letters are vector-valued random variables.\\
$[x]$: lower-case Latin letters between brackets are deterministic matrices.\\
$[\bfX]$: boldface upper-case letters between brackets are matrix-valued random variables.
\begin{multicols}{2}
\noindent $i\,$: imaginary unit, $i^2=-1$.\\
$\curC_0$: continuous $\CC$ functions on $\RRnu$ going to $0$ at $\infty$.\\
$\curC_{\ad,\bflambda}$: admissible set of $\bflambda\in\RRNr$.\\
$\CC$: set of all the complex numbers.\\
$D_d$: training set of points $\bfx_d^j$ in $\RRnx$.\\
$\curD_d$: training set of points $\bfeta_d^j$ in $\RR^\nu$.\\
$\curD_{\bfH_\pppost}$: constrained learned set for $\bflambda = \bflambda^\psol$.\\
$\curD_{\bfHbflambdai}$: constrained learned set for $\bflambda^{\,i}$.\\
$D_\pexp$: target set of $N_r$ points $\bfq_\pexp^r$ in $\RRnq$.\\
$\MM_{n,m}$: set of the $(n\times m)$ real matrices.\\
$\MM_n$: set of  the square $(n\times n)$ real matrices.\\
$\MM_n^+$: set of  the positive-definite $(n\times n)$ real matrices.\\
$\MM_n^{+0}$: set of the positive $(n\times n)$ real matrices.\\
$N$: number of points in the constrained learned set.\\
$N_d$: number of points in the training set.\\
$\NN$: set of all the integers $\{0,1,2,\ldots\}$.\\
$\NN^*$: $\NN\backslash \{ 0\}$.\\
$\RR$: set of all the real number.\\
$\RR^n$: Euclidean space of dimension $n$.\\
$[I_{n}]$: identity matrix in $\MM_n$.\\
$\bfx = (x_1,\ldots,x_n)$: point in $\RR^n$.\\
$\langle \bfx,\bfy \rangle = x_1 y_1 + \ldots + x_n y_n$: inner product in $\RR^n$.\\
$\Vert\,\bfx\,\Vert$:  norm in $\RR^n$ such that $\Vert\,\bfx\,\Vert = \langle \bfx,\bfx \rangle$.\\
$[x]^T$: transpose of matrix $[x]$.\\
$\tr \{[x]\}$: trace of the square matrix $[x]$.\\
$\Vert\, [x]\, \Vert \, = \sup_{\,\Vert\,\bfy\, \Vert\, = 1} \Vert\, [x]\,\bfy\,\Vert$.\\
$\Vert\, [x]\, \Vert_F$: Frobenius norm of matrix  $[x]$.\\
$\delta_{kk'}$: Kronecker's symbol.\\
$\delta_{\bfx_0}$: Dirac measure at point $\bfx_0$ in $\RR^n$.\\
$\overline z$: conjugate of complex number $z$.\\
$a.s.$: almost surely.\\
BVP: boundary value problem.\\
$E$: mathematical expectation operator.\\
ISDE: It\^o stochastic differential equation.\\
KDE: kernel density estimation.\\
pdf: probability density function.\\
PCA: principal component analysis.\\
PDE: partial differential equation.\\
\end{multicols}
\noindent\textit{Convention used for random variables}.
In this paper, for any finite integer $m\geq 1$, the Euclidean space $\RR^m$ is equipped with the $\sigma$-algebra $\curB_{\RR^m}$. If $\bfY$ is a $\RR^m$-valued random variable defined on the probability space $(\Theta,\curT,\curP)$, $\bfY$ is a  mapping $\theta\mapsto\bfY(\theta)$ from $\Theta$ into $\RR^m$, measurable from $(\Theta,\curT)$ into $(\RR^m,\curB_{\RR^m})$, and $\bfY(\theta)$ is a realization (sample) of $\bfY$ for $\theta\in\Theta$. The probability distribution of $\bfY$ is the probability measure $P_\bfY(d\bfy)$ on the measurable set $(\RR^m,\curB_{\RR^m})$ (we will simply say on $\RR^m$). The Lebesgue measure on $\RR^m$ is noted $d\bfy$ and
when $P_\bfY(d\bfy)$ is written as $p_\bfY(\bfy)\, d\bfy$, $p_\bfY$ is the probability density function (pdf) on $\RR^m$ of $P_\bfY(d\bfy)$ with respect to $d\bfy$.
\section{Scaling and reduced representation}
\label{sec:Section2}
Before performing the construction of the reduced representation that is performed by using a principal component analysis (PCA) of $\bfX$, it is assumed that training set $D_d$ is scaled using the formulation presented in \cite{Soize2016}. The target set $D_\pexp$ is also scaled using the same transformation that the one used for obtaining the scaled training set $D_d$.

Let $\tilde\bfx^j =\bfx_d^j -\underline\bfx$ be the realization of $\bfX$ with $\underline\bfx=(1/N_d)\sum_{j=1}^{N_d} \bfx_d^j\in\RRnx$. Let $[\tilde x] =[\tilde\bfx^1 \ldots \tilde\bfx^{N_d}]$ be the matrix in  $\MM_{n_x,N_d}$ and let
$[\Phi]\,[S]\,[\Phi]^T = [\tilde x]$ be the economy size SVD (thin SVD \cite{Golub1993}) of matrix $[\tilde x]$. The diagonal entries of $[S]$ are the singular values $S_1\geq  \ldots \geq  S_{N_d -1} > S_{N_d} = 0$ that are in increasing order and we have $S_{N_d} = 0$. The matrix $[\Phi]$ is in $\MM_{n_x,\nu}$ with $\nu = N_d-1$ and $[\Phi]^T [\Phi] = [I_{\nu}]$.
Let  $\bfX^{(\nu)}$ be the representation of $\bfX$ defined by
\begin{equation}\label{eq:eq1}
\bfX^{(\nu)}=\underline\bfx + [\Phi]\, [\kappa]^{1/2} \, \bfH \, ,
\end{equation}
in which $[\kappa]$ is the diagonal matrix in $\MM_{\nu}^+$ such that $\kappa_\alpha = [\kappa]_{\alpha\alpha} = S_\alpha^2 /(N_d-1)$, and where $\bfH=(H_1,\ldots,H_{\nu})$ is the $\RRnu$-valued random variable whose $N_d$ independent realizations are
\begin{equation}\label{eq:eq2}
\bfeta_d^j = [\kappa]^{-1/2} \,[\Phi]^T\, (\bfx_d^j -\underline\bfx)\, \quad , \quad j=1,\ldots , N_d\, .
\end{equation}
The positive real numbers $\{\kappa_\alpha\}_\alpha$ are the eigenvalues of the estimated covariance matrix $[\widehat C_\bfX]$ of the covariance matrix $[C_\bfX]$ of $\bfX$, performed using the training set. Therefore,  $[\kappa]$ and $[\Phi]$ depend on $N_d$. As it can be seen, these eigenvalues and the associated eigenvectors are computed without computing $[\widehat C_\bfX]$  because $n_x$ can be very big.
It should be noted that, if $N_d=n_x$ and $\nu < N_d -1$, then the sequence of random variables $\bfX^{(\nu)}$ is mean-square convergent to $\bfX$ when  $\nu$ goes to  $N_d -1$, and if  $\nu=N_d -1 = n_x - 1$, then Eq.~\eqref{eq:eq1} is not an approximation and corresponds to a change of basis.
In general, for the high-dimension problems, $n_x$ is very large and $N_d\ll n_x$. Therefore, Eq.~\eqref{eq:eq1} corresponds to a reduced representation, which is an approximation whose accuracy depends on $\nu$ and $N_d$ and which is classically controlled as follows.
For $N_d$ fixed and for $\nu < N_d-1$, let $\kappa_1\geq \ldots \geq \kappa_{\nu} > 0$ be the $\nu$ largest positive eigenvalues of
$[\widehat C_\bfX]$. Let $\nu$ be chosen such that
\begin{equation}\label{eq:eq2bis0}
\perr_\pPCA(\nu\, ;N_d) = \frac{E\{\Vert\, \bfX - \bfX^{(\nu)}\,\Vert^2\}}{ E\{\Vert\, \bfX\,\Vert^2\}} \simeq
 1- \frac{\sum_{\alpha=1}^\nu\kappa_\alpha}{\tr\{[\widehat C_\bfX]\}} \, \leq \varepsilon_\pPCA \quad , \quad  \nu < N_d-1  \, ,
\end{equation}
in which $\varepsilon_\pPCA$ is a given positive real number sufficiently small. The trace $\tr\{[\widehat C_\bfX]\}$ of  $[\widehat C_\bfX]$ is calculated by estimating the diagonal entries of $[\widehat C_\bfX]$ using the training set.
Note that $\nu\mapsto \perr_\pPCA(\nu\, ;N_d)$ defined by Eq.~\eqref{eq:eq2bis0} gives the relative error as a function of $\nu < N_d-1$ for a fixed value of $N_d$.

Throughout the rest of the paper, in order to simplify the notations, the superscript "$(\nu)$" will be omitted and the random variable
$\bfX^{(\nu)}=(\bfQ^{(\nu)},\bfW^{(\nu)})$ will simply be denoted by $\bfX=(\bfQ,\bfW)$.
From Eq.~\eqref{eq:eq1}, it can be deduced that
\begin{equation}\label{eq:eq2bis}
\bfQ=\underline\bfq + [\Phi_q]\, [\kappa]^{1/2} \, \bfH \quad , \quad
\bfW=\underline\bfw + [\Phi_w]\, [\kappa]^{1/2} \, \bfH \, ,
\end{equation}
in which $\underline\bfx = (\underline\bfq , \underline\bfw) \in \RRnx = \RRnq\times \RRnw$, and where
$[\Phi_q]\in \MM_{n_q,\nu}$ and $[\Phi_w]\in \MM_{n_w,\nu}$ are the corresponding block extraction with respect to $\bfQ$ and $\bfW$.
The training set related to $\bfH$ is
\begin{equation}\label{eq:eq3}
\curD_d = \{\bfeta_d^1 ,\ldots , \bfeta_d^{N_d}\}\, \quad , \quad \bfeta_d^j\in\RRnu \quad , \quad j=1,\ldots , N_d\, ,
\end{equation}
in which $\bfeta_d^j$ is given by Eq.~\eqref{eq:eq2}.
Using $\curD_d$, the estimates $\underline\bfeta\in\RRnu$ and $[\widehat C_\bfH]\in\MM_{\nu}^+$ of the mean value and the covariance matrix of $\bfH$ are such that
\begin{equation}\label{eq:eq4}
\underline\bfeta = \bfzero_{\nu} \quad , \quad [\widehat C_\bfH] = [I_{\nu}] \, .
\end{equation}
The first Eq.~\eqref{eq:eq2bis} allows for defining the $\RRnq$-valued random variable $\widetilde{\bfQ}$ such that
\begin{equation}\label{eq:eq5}
\widetilde{\bfQ} = \bfQ-\underline\bfq \quad , \quad \widetilde\bfQ = [\Phi_q]\, [\kappa]^{1/2} \, \bfH \, .
\end{equation}
 Note that we have $[\Phi_q]^T [\Phi_q] \not = [I_{\nu}]$. We also introduce the $\RRnq$-valued random variable
\begin{equation}\label{eq:eq6}
\widetilde\bfQ_\pexp = \bfQ_\pexp -\underline\bfq  \, ,
\end{equation}
whose $N_r$ realizations are
\begin{equation}\label{eq:eq7}
\tilde\bfq_\pexp^r = \bfq_\pexp^r -\underline\bfq \quad , \quad r\in\{1,\ldots , N_r\} \, .
\end{equation}
Note that $\widetilde\bfQ_\pexp$ is generally not centered because $\underline\bfq$ is not the mean value of $\bfQ_\pexp$, but $\bfH$ is a centered one (see Eq.~\eqref{eq:eq4}).
\section{Definition of the functional constraint for estimating the posterior probability measure}
\label{sec:Section3}
The objective is to construct the posterior probability measure $P_{\bfQ}^\ppost(d\bfq) = P_{\bfQ,\bfW}^\ppost(d\bfq,\RRnw)$ that is closest to $P_\bfQ^\pexp(d\bfq)$, which is equivalent (see Eqs.~\eqref{eq:eq5} and \eqref{eq:eq6}) to construct the posterior probability measure $P_{\widetilde\bfQ}^\ppost(d\bfq)$ that is closest to the probability measure
$P_{\widetilde\bfQ}^\pexp(d\bfq)$ of $\widetilde\bfQ_\pexp$ for which $N_r$ independent realizations $\{\tilde\bfq_\pexp^r,r=1,\ldots, N_r\}$ are given.
For using the Kullback-Leibler minimum principle, we need to express the constraint as a mathematical expectation of a random variable.
We propose to use the equality of the Fourier transforms of the probability measures (characteristic functions)  instead of the probability measures for the reason given in Remark~\ref{remark:1}-(ii).
%
%---DEFINITION 1 ----------------
\begin{definition}[Constraint defined by the equality of the Fourier transform of the probability measures] \label{definition:1}
{\textit{ The constraint is defined as follows,
\begin{equation}\label{eq:eq8}
\forall \bfy\in\RRnq \quad , \quad \ff(\bfy) = \ff_\pexp(\bfy) \, ,
\end{equation}
in which the complex-valued functions $\ff$ and $\ff_\pexp$ are the characteristic functions defined on $\RRnq$ of the $\RRnq$-valued random variables $\widetilde\bfQ$ and $\widetilde\bfQ_\pexp$,
\begin{equation}\label{eq:eq9}
\ff(\bfy) = E \{ e^{i\,\langle\, \bfy ,\, \widetilde\bfQ\, \rangle } \}  \quad , \quad
\ff_\pexp(\bfy) = E \{ e^{i\,\langle \,\bfy ,\, \widetilde\bfQ_\pexp \,\rangle } \} \, .
\end{equation}
}}
\end{definition}
The constraint defined by Eq.~\eqref{eq:eq8} is in high dimension because $\bfy\in\RRnq$. We then propose to reduce the dimension using the representation of $\widetilde\bfQ$ given by Eq.~\eqref{eq:eq5}.
%
%---LEMMA 1 ----------------------
\begin{lemma}[Functional constraint using the representation of $\widetilde\bfQ$] \label{lemma:1}
It is chosen to project the target on the prior model. Therefore, let $\bfH_\pexp$ be the $\RRnu$-valued random variable defined by
\begin{equation}\label{eq:eq16bis}
\bfH_\pexp = [V]^T\,\widetilde\bfQ_\pexp \, ,
\end{equation}
in which the matrix $[V]\in\MM_{n_q,\nu}$ is written as
$[V] = [\Phi_q]\,([\Phi_q]^T \,[\Phi_q])^{-1} \,[\kappa]^{-1/2}$.
Using the representation of $\widetilde\bfQ$ defined by Eq.~\eqref{eq:eq5}, the functional constraint associated with Eq.~\eqref{eq:eq8} is written as
\begin{equation}\label{eq:eq17}
\forall \bfv\in\RRnu \quad , \quad f(\bfv) = f_\pexp(\bfv) \, ,
\end{equation}
in which the complex-valued functions $f$ and $f_\pexp$ are the characteristic functions, defined on $\RRnu$, of the $\RRnu$-valued random variables $\widetilde\bfH$ and $\widetilde\bfH_\pexp$,
\begin{equation}\label{eq:eq18}
f(\bfv) = E \{ e^{i\,\langle \,\pbfv ,\, \bfH \,\rangle } \}  \quad , \quad
f_\pexp(\bfv) = E \{ e^{i\,\langle \,\pbfv ,\, \bfH_\pexp \,\rangle } \} \, .
\end{equation}
\end{lemma}
%
%------ PROOF LEMMA 1 ------------------
\begin{proof} (Lemma~\ref{lemma:1}).
Using Eq.~\eqref{eq:eq5}, for all $\bfy$ in $\RRnq$, we have
$\langle\,\bfy\, , \widetilde\bfQ\,\rangle = \langle\,\bfy\, , [\Phi_q]\, [\kappa]^{1/2} \bfH\,\rangle
= \langle \,[\kappa]^{1/2} \,[\Phi_q]^T \bfy\, ,   \bfH\,\rangle = \langle \,\bfv\, ,   \bfH\,\rangle$ in which
$\bfv = [\kappa]^{1/2} [\Phi_q]^T\, \bfy \in\RRnu$.
In the other hand, we perform the projection of $\widetilde\bfQ_\pexp$ on the prior model. Since the matrix
$[\Phi_q]^T[\Phi_q]\in\MM_{\nu}^+$ is invertible and since the diagonal matrix $[\kappa]\in\MM_{\nu}^+$ is also invertible, we introduce the
pseudo-inverse $[V]$ of $[\kappa]^{1/2} [\Phi_q]^T)$ (projection) such that  $([\kappa]^{1/2} [\Phi_q]^T)\, [V] = [I_\nu]$.
Therefore, taking $\bfy = [V]\, \bfv$  for all $\bfv$ in $\RRnu$, we have
$\langle\,\bfy\, , \widetilde\bfQ_\pexp\,\rangle = \langle \,[V]\, \bfv\, , \widetilde\bfQ_\pexp\,\rangle
= \langle \,\bfv\, , [V]^T\, \widetilde\bfQ_\pexp\,\rangle  = \langle \,\bfv\, ,   \bfH_\pexp\,\rangle$.
Eqs.~\eqref{eq:eq17} and \eqref{eq:eq18} are then deduced from Eqs.~\eqref{eq:eq8} and \eqref{eq:eq9}.
\end{proof}
%
%
%----- REMARK 1 ------------------------------
\begin{remark} \label{remark:1}
There are two difficulties.

(i) The constraint does not concern all the variables, that is to say $\bfX=(\bfQ,\bfW)$, but only the quantity of interest $\bfQ$. We are therefore in an under-observed case with respect to the applied constraint. This choice of the developments framework imposes to project the target $\bfQ_\pexp$ on the prior model in order to obtain a representation of the target $\bfH_\pexp$ that only depends on $\bfQ_\pexp$ and not on $\bfQ_ \pexp$ and $\bfW_\pexp$ because $\bfW_\pexp$ is not given as a constraint.

(ii) The explicitness of the constraint defined by Eq.~\eqref{eq:eq17} requires to sample $\bfv$ in $\RRnu$, what is not easy, not efficient, and not accurate in high dimension ($\nu\gg 1$). If such a sampling method was used, then the number of constraints that should be considered in the Kullback-Liebler minimum principle would be huge or even unrealistic. We thus propose to construct a weak formulation of the functional equation defined by  Eq.~\eqref{eq:eq17} using the fundamental properties of the Fourier transform of the probability measures (see for instance \cite{Guelfand1964}).
\end{remark}
\section{Weak formulation of the functional constraint}
\label{sec:Section4}
%
%-----NOTATION 1
%
\noindent \textbf{Notation 1} (\textbf{Defining $\curC_b,\curC_0,\curH_0$ and $\curH_1$ complex vector spaces}).
(i) Let $C^0(\RRnu,\CC)$ (resp. $L^\infty(\RRnu,\CC)$) be the complex vector space of continuous (resp. bounded) functions on $\RRnu$ with values in $\CC$. The norm $\Vert \,g\, \Vert_{L^\infty}$ in $L^\infty(\RRnu,\CC)$ of a function
$g\in C^0(\RRnu,\CC) \cap L^\infty(\RRnu,\CC)$ is
\begin{equation}\label{eq:eq21}
\Vert \, g\, \Vert_{L^\infty}\, = {\rm{ess}}. \, {\rm{sup}}_{\bfv}\,  \vert \,g(\bfv)\,\vert \,=  {\rm{sup}}_{\bfv} \, \vert \,g(\bfv)\,\vert \, .
\end{equation}
The norm $\Vert\,\varphi\,\Vert_{L^q}$ in $L^q(\RRnu,\CC)$ of the complex-valued functions on $\RRnu$ is
\begin{equation}\label{eq:eq22}
\Vert \, \varphi \,\Vert_{L^q} = \left ( \int_{\RRnu} \vert\, \varphi(\bfv) \,\vert^q \, d\bfv \right )^{1/q}
\quad , \quad 1\leq q < +\infty \, .
\end{equation}
(ii) We define the vector spaces of complex-valued functions, $\curC_b,\curC_0,\curH_0$ and $\curH_1$, such that
\begin{equation}\label{eq:eq22bis}
\curC_b = C^0(\RRnu,\CC)\cap L^\infty(\RRnu,\CC) \, ,
\end{equation}
\begin{equation}\label{eq:eq22ter}
\curC_0 = \left\{ g \in C^0(\RRnu,\CC)\,\, , \,\, \vert\, g(\bfv)\,\vert \rightarrow 0 \,\, {\rm{as}}\,\, \Vert\,\bfv\,\Vert \rightarrow +\infty \right \} \subset\curC_b \subset L^\infty(\RRnu,\CC) \, ,
\end{equation}
\begin{equation}\label{eq:eq23}
\curH_0 = \curC_0 \cap L^2(\RRnu,\CC) \, ,
\end{equation}
\begin{equation}\label{eq:eq24}
\curH_1 = L^1(\RRnu,\CC) \cap L^2(\RRnu,\CC) \, .
\end{equation}
%
%--- HYPOTHESIS 1 --------------------
\begin{hypothesis}[Existence, regularity, and integrability of the density of $P_\bfH$ (resp. $P_{\bfH_\pexp}$)] \label{hypothesis:1}
\textit{It is assumed that the probability measure $P_\bfH$ (resp. $P_{\bfH_\pexp}$) on $\RRnu$ of the $\RRnu$-valued random variable $\bfH$ (resp. $\bfH_\pexp$) is defined by a density $\bfeta\mapsto p_\bfH(\bfeta)$ (resp. $\bfeta\mapsto p_{\bfH_\pexp}(\bfeta)$) with
respect to the Lebesgue measure $d\bfeta$, such that
\begin{equation}\label{eq:eq20}
p_\bfH \,\,(\hbox{resp.} \, p_{\bfH_\pexp}) \,\, \in \,C^0(\RRnu,\RR)\cap L^1(\RRnu,\RR)\cap L^2(\RRnu,\RR) \, .
\end{equation}
}
\end{hypothesis}
%
%----- REMARK 2 ------------------------------
\begin{remark} \label{remark:2} (i)  A probability density function on $\RRnu$ is always in $L^1(\RRnu,\RR)$. The unusual hypothesis is the belonging to $C^0(\RRnu,\RR)\cap L^2(\RRnu,\RR)$.  As part of the method we propose, the prior probability density $p_\bfH$ of $\bfH$ will be estimated using the KDE method of the nonparametric statistics with the training set (see Section~\ref{sec:Section6.1}). In this situation, the hypothesis will be verified.
This hypothesis  will allow us an efficient finite representation of the weak formulation to be constructed.

\noindent (ii)  The prior pdf $p_\bfH$ will effectively be used to construct the posterior pdf $p_\bfH^\ppost$ by using the Kullback-Leibler minimum principle under the constraint defined by the target set $D_\pexp$ as we have previously explained. The pdf $p_{\bfH_\pexp}$ is not used in the methodology proposed and moreover, if it were to be used, there would be a difficulty because it is assumed that $\nu$ is large enough and that $N_r$ is not  sufficiently large for obtaining a converged estimate of $p_{\bfH_\pexp}$ using the KDE method with the target set $D_\pexp$. The hypothesis defined by Eq.~\eqref{eq:eq20} for $p_{\bfH_\pexp}$ will strongly be used and is coherent with the one introduced for $p_\bfH$.
\end{remark}
%
%
%---LEMMA 2 ----------------------
\begin{lemma}[Fourier transform of the probability measures $P_\bfH$ and $P_{\bfH_\pexp}$] \label{lemma:2}
The Fourier transform of the probability measures $P_\bfH$ and $P_{\bfH_\pexp}$ (characteristic functions) $\bfv \mapsto f(\bfv) = E\{ e^{i\,\langle \,\pbfv , \,\bfH\,\rangle}\}$
and $\bfv \mapsto f_\pexp(\bfv) = E\{ e^{i\,\langle \,\pbfv , \,\bfH_\pexp\,\rangle}\}$ from $\RRnu$ into $\CC$ are such that
\begin{equation}\label{eq:eq25}
f \, \in \, \curH_0 \quad , \quad f_\pexp \, \in \, \curH_0\, .
\end{equation}
\end{lemma}
%
%------ PROOF LEMMA 2 ------------------
\begin{proof} (Lemma~\ref{lemma:2}).
Since $\forall\, \bfv\in\RRnu$, $f(\bfv) = \int_{\RRnu} e^{\,i\,\langle \,\pbfv , \,\bfeta\,\rangle}\, p_\bfH(\bfeta)\, d\bfeta$ and taking into account Eq.~\eqref{eq:eq20}, it is deduced that $f\in\curC_0$ because $p_\bfH\in L^1(\RRnu,\CC)$, and on the other hand, since
$p_\bfH\in L^2(\RRnu,\CC)$ then $f\in L^2(\RRnu,\CC)$. Consequently, $f\in\curH_0$ and the proof is similar for $f_\pexp$.
\end{proof}
%
%
%
%----- REMARK 3 ------------------------------
\begin{remark} \label{remark:3}
Since $f$ is the Fourier transform of probability measure $P_\bfH(d\bfeta) = p_\bfH(\bfeta)\, d\bfeta$ (positive bounded measure), then it is known (Bochner's theorem) that $f$ (same properties for $f_\pexp$) is a positive-type function, that is to say, for all integer
$m\geq 1$, for all complex numbers $z_1,\ldots, z_m$, and for all vectors $\bfv^1,\ldots , \bfv^m$in $\RRnu$, we have
\begin{equation}\label{eq:eq26}
\sum_{k=1}^m \sum_{k'=1}^m f(\bfv^k -\bfv^{k'})\, z_{k'}\, \overline z_{k} \, \geq \, 0 \, .
\end{equation}
Note that Eq.~\eqref{eq:eq26} can simply be deduced without evoking the Bochner theorem because
$\sum_{k,k'} f(\bfv^k -\bfv^{k'})\, z_{k'}\, \overline z_{k}  =
\sum_{k,k'}  E\{ e^{\,i\,\langle \,\pbfv^k -\pbfv^{k'} \!\! , \,\bfH\,\rangle}\} \, z_{k'}\, \overline z_{k} = $
$ E\left\{ \vert\,  \sum_{k'=1}^m e^{-i\,\langle\, \pbfv^{k'} \!\! , \,\bfH\,\rangle}\} \, z_{k'}\,\vert \right\} \, \geq \, 0$.
We have a similar property to Eq.~\eqref{eq:eq26} for $f_\pexp$, that is to say,
\begin{equation}\label{eq:eq26bis}
\sum_{k=1}^m \sum_{k'=1}^m f_\pexp(\bfv^k -\bfv^{k'})\, z_{k'}\, \overline z_{k} \, \geq \, 0 \, .
\end{equation}
\end{remark}
%
%
%---LEMMA 3 ----------------------
\begin{lemma}[Convolution operator $A$] \label{lemma:3}
Let $f$  be defined by Eq.~\eqref{eq:eq18}, which belongs to $\curH_0$ (see Lemma~\ref{lemma:2}). Let $A$ be the convolution operator defined, for all $\bfv\in\RRnu$, by
\begin{equation}\label{eq:eq27}
(A\varphi)(\bfv) = \int_{\RRnu} f(\bfv - \bfv')\, \varphi(\bfv')\, d\bfv' \, ,
\end{equation}
in which $\varphi$ is a function from $\RRnu$ in $\CC$.
Since $f\in\curH_0$, then $f\in L^\infty(\RRnu,\CC)$ and $A$ is a continuous linear operator from
$L^1(\RRnu,\CC)$ into $\curC_b$ and we have
$\Vert\, A\varphi \, \Vert_{L^\infty} \,\, \leq \, c_1\, \Vert \, \varphi\, \Vert_{L^1}$
in which $c_1$ is such that
$\Vert\, f \, \Vert_{L^\infty} = c_1 \, < +\infty$.
This Lemma holds by replacing $f$ by $f_\pexp$.
\end{lemma}
%
%------ PROOF LEMMA 3 ------------------
\begin{proof} (Lemma~\ref{lemma:3}).
For the proof of this usual result, we refer the reader, for instance, to \cite{Dieudonne1978} or to Propositions 1 and 2, Pages 164-165 of \cite{Soize1993b}.
\end{proof}
%
%----- REMARK 4 ------------------------------
\begin{remark} \label{remark:4}
Since $f\in\curH_0$, then $f$ also belongs to $L^2(\RRnu,\CC)$ and consequently, $A$ is also a continuous linear operator from
$L^1(\RRnu,\CC)$ into $L^2(\RRnu,\CC)$, and also from $L^2(\RRnu,\CC)$ into $\curC_0$, but we do not need to use these properties. This remark holds by replacing $f$ by $f_\pexp$.
\end{remark}
%
%
%---LEMMA 4 ----------------------
\begin{lemma}[Hermitian form $F$ on $\curH_1$ associated with $f$] \label{lemma:4}
Let $\varphi\mapsto F(\varphi)$ be the functional defined on $\curH_1$ with values in $\CC$, such that
\begin{equation}\label{eq:eq30}
F(\varphi) = \int_{\RRnu} \int_{\RRnu}  f(\bfv -\bfv')\, \varphi(\bfv')\,\overline{\varphi(\bfv)}\, d\bfv'\, d\bfv \, .
\end{equation}
Then $F$ is a positive Hermitian form on $\curH_1$. There is a finite positive constant $0 < c_1 < +\infty$, such that for all $\varphi$ in $\curH_1$,
\begin{equation}\label{eq:eq31}
\vert\, F(\varphi)\, \vert \, \leq \, c_1\, \Vert \, \varphi\, \Vert_{L^1}^2 \, < \, +\infty \quad , \quad F(\varphi) \, \geq \, 0\, .
\end{equation}
This Lemma holds by replacing $f$ by $f_\pexp$.
\end{lemma}
%
%------ PROOF LEMMA 4 ------------------
\begin{proof} (Lemma~\ref{lemma:4}).
(i) Using Eq.~\eqref{eq:eq27}, we have
$\vert\, F(\varphi)\, \vert \,  = \vert \int_{\RRnu} (A\varphi)(\bfv)\, \overline{\varphi(\bfv)}\, d\bfv \, \vert
\, \leq \, \int_{\RRnu} \vert\, (A\varphi)(\bfv)\,\vert \,\, \vert\, \varphi(\bfv)\,\vert \, d\bfv$.
As $\varphi\in\curH_1$,  we thus have $\varphi\in L^1(\RRnu,\CC)$. Using Lemma~\ref{lemma:3}, since $f\in\curH_0$,
$\Vert\, A\varphi \, \Vert_{L^\infty} \,\, \leq \, c_1\, \Vert \, \varphi\, \Vert_{L^1}$.
Consequently,
$\vert\, F(\varphi)\, \vert \, \leq \, \Vert\, A\varphi \, \Vert_{L^\infty} \int_{\RRnu} \vert \,\varphi(\bfv)\,\vert \, d\bfv
= \Vert\, A\varphi \, \Vert_{L^\infty} \,\Vert \, \varphi\, \Vert_{L^1}$, and
$\vert\, F(\varphi)\, \vert \, \leq \, c_1\, \Vert \, \varphi\, \Vert_{L^1}^2 \, < \, +\infty$,
which proves the first part of Eq.~\eqref{eq:eq31}.

\noindent (ii) Let us proof that $F(\varphi) \in\RR$. We have
$\overline{F(\varphi)} = \int_{\RRnu} \int_{\RRnu}  \overline{f(\bfv -\bfv')}\, \overline{\varphi(\bfv')}\,\varphi(\bfv)\, d\bfv'\, d\bfv =  \int_{\RRnu} \int_{\RRnu}  \overline{f(\bfv' -\bfv)}\, \overline{\varphi(\bfv)}\,\varphi(\bfv')\, d\bfv\, d\bfv'$ and
$\overline{f(\bfv' -\bfv)} = f(\bfv -\bfv')$. We then have
$\overline{F(\varphi)} = F(\varphi)$.

\noindent (iii) Since $\varphi\in\curH_1$, we have $\varphi\in L^1(\RRnu,\CC)$.
The property $F(\varphi)\, \geq \, 0$ is similar to the one defined by Eq.~\eqref{eq:eq26} or \eqref{eq:eq26bis}.
Since $\vert\, F(\varphi)\, \vert \, < \, +\infty$, $F(\varphi)$ exists and
$F(\varphi) = \!\int_{\RRnu} \int_{\RRnu}  f(\bfv -\bfv')\, \varphi(\bfv')\,\overline{\varphi(\bfv)}\, d\bfv' d\bfv =
E\left\{ \vert \int_{\RRnu} e^{-i\,\langle\,\pbfv'\! , \, \bfH\,\rangle} \varphi(\bfv')\, d\bfv' \vert^2 \right\}  \geq  0$.
\end{proof}
%
%---- THEOREM 1 ----------------------
\begin{theorem}[Representation of Hermitian forms $F$ and $F_\pexp$] \label{theorem:1}
(i) The positive Hermitian form $\varphi\mapsto F(\varphi):\curH_1\rightarrow \RR^+$, defined by Eq.~\eqref{eq:eq30}, can be rewritten as
\begin{equation}\label{eq:eq32}
 F(\varphi) =\int_{\RRnu} f(\bfv)\, \overline{\psi(\bfv)} \, d\bfv\, ,
\end{equation}
in which $\bfv\mapsto \psi(\bfv): \RRnu\rightarrow \CC$ is such that
$\psi = \overline{\varphi^\vee}  \ast  \varphi  \in  L^1(\RRnu,\CC)$,
with $\varphi^\vee(\bfv) = \varphi(-\bfv)$ and where
$(\overline{\varphi^\vee}  \ast  \varphi)(\bfv) = \int_{\,\RRnu} \overline{\varphi^\vee}(\bfv-\bfv') \, \varphi(\bfv')\, d\bfv'$ is the convolution product of $\overline{\varphi^\vee}$ with $\varphi$.

\noindent (ii) The Fourier transform $\bfeta\mapsto\hat\varphi(\bfeta) =\int_{\,\RRnu} e^{-i\,\langle\,\bfeta\, ,\, \pbfv\,\rangle}\,\varphi(\bfv)\, d\bfv$ of $\varphi\in\curH_1$ is such that $\hat\varphi\in\curH_0$. For all $\bfeta$ in $\RRnu$,
the Fourier transform $\hat\psi(\bfeta) = \int_{\,\RRnu} e^{-i\,\langle\,\bfeta\, ,\, \pbfv\,\rangle}\,\psi(\bfv)\, d\bfv$ of $\psi$ is written as
$\hat\psi(\bfeta) = \vert\, \hat\varphi(\bfeta) \,\vert^2 \,\in \RR^+$ and  $\hat\psi$ is a positive-valued function that belongs to $\curC_0\cap L^1(\RRnu,\CC)$.

\noindent (iii) We have the following representation of $F(\varphi)$,
\begin{equation}\label{eq:eq37}
 \forall\,\varphi\in\curH_1\quad , \quad F(\varphi) =E\{\,\vert\,\hat\varphi(\bfH) \,\vert^2\,\} = E\{\,\hat\psi(\bfH)\,\} \, < \, +\infty \, .
\end{equation}
Results (i) and (ii) hold by replacing $f$ by $f_\pexp$ and $F$ by $F_\pexp$ and we thus have the following representation of $F_\pexp(\varphi)$,
\begin{equation}\label{eq:eq38}
 \forall\,\varphi\in\curH_1\quad , \quad F_\pexp(\varphi) =E\{\,\vert\,\hat\varphi([V]^T\,\widetilde\bfQ_\pexp) \,\vert^2\,\} = E\{\,\hat\psi([V]^T\widetilde\bfQ_\pexp)\} \, < \, +\infty \, .
\end{equation}
\end{theorem}
%
%------ PROOF OF THEOREM 1 ------------------
\begin{proof} (Theorem~\ref{theorem:1}).
(i) The change of variable $\bfv-\bfv'=\bfu'$ in the right-hand side of Eq.~\eqref{eq:eq30} yields
$F(\varphi) = \int_{\,\RRnu} \int_{\,\RRnu}  f(\bfu')\, \varphi(\bfv-\bfu')\,\overline{\varphi(\bfv)}\, d\bfv\, d\bfu'$.
Using the notation $\varphi^\vee$, $F(\varphi)$ can be rewritten as
$F(\varphi) = \int_{\,\RRnu} f(\bfu') \int_{\,\RRnu}  \varphi^\vee(\bfu'-\bfv)\,\overline{\varphi(\bfv)}\, d\bfv\, d\bfu'
= \int_{\,\RRnu} f(\bfv) \, \overline{\psi(\bfv)}\, d\bfv$ in which
$\overline{\psi(\bfv)} = \int_{\,\RRnu}  \varphi^\vee(\bfv-\bfv')\,\overline{\varphi(\bfv')}\, d\bfv'$.
As $\varphi\in\curH_1$, we thus have $\varphi$ and thus $\varphi^\vee$ in $L^1(\RRnu,\CC)$. The convolution product of two functions in $L^1(\RRnu,\CC)$ is a function in $L^1(\RRnu,\CC)$.

\noindent (ii) Function $\varphi$ belongs to $\curH_1$. Therefore, $\hat\varphi$ belongs to $\curH_0$.
Since $\psi\in L^1(\RRnu,\CC)$, its Fourier transform $\bfeta\mapsto\hat\psi(\bfeta)$ on $\RRnu$ belongs to $\curC_0$ and is written as
$\hat\psi(\bfeta) =  \widehat{\overline{\varphi}^\vee}(\bfeta) \times  \varphi(\bfeta) = \vert\, \hat\varphi(\bfeta) \,\vert^2$, which shows that
$\hat\psi$ is a positive-valued function. Since $\hat\varphi\in\curH_0$, this means that $\hat\varphi\in\curC_0\cap L^2(\RRnu,\CC)$ and then
$\hat\psi$ is a positive-valued function that belongs to $\curC_0\cap L^1(\RRnu,\CC)$.

\noindent (iii) Using Eqs.~\eqref{eq:eq18}, \eqref{eq:eq32}, and Hypothesis~\ref{hypothesis:1} yield
$F(\varphi) = \int_{\,\RRnu}\int_{\,\RRnu} e^{\,i\,\langle\,\bfeta\, ,\, \pbfv\,\rangle}\,\overline{\psi(\bfv)}\, d\bfv\, p_\bfH(\bfeta)\, d\bfeta = \int_{\,\RRnu}\overline{\hat\psi(\bfeta)} p_\bfH(\bfeta)\, d\bfeta
= \int_{\,\RRnu} \vert\, \hat\varphi(\bfeta) \,\vert^2 \, p_\bfH(\bfeta)\, d\bfeta$.
Since $\hat\varphi$ is a continuous function on $\RRnu$, $\hat\varphi(\bfH)$ is a $\CC$-valued random variable such that
$F(\varphi) =E\{\,\vert\,\hat\varphi(\bfH) \,\vert^2\,\}$, and due to Eq.~\eqref{eq:eq31}, we have $F(\varphi) = \vert\,F(\varphi)\,\vert \, < \, +\infty$. The proof of Eq.~\eqref{eq:eq38} is similar to the proof of Eq.~\eqref{eq:eq37} by using also Eq.~\eqref{eq:eq16bis} with the second equation in Eq.~\eqref{eq:eq18}
\end{proof}
%
%
%----- REMARK 5 ------------------------------
\begin{remark} [Functional $\FF(\psi)$ on $L^1(\RRnu,\CC)$] \label{remark:5}
Due to Theorem~\ref{theorem:1}-(i), functional $\varphi\mapsto F(\varphi)$ on $\curH_1$ can also be viewed as a functional $\psi\mapsto \FF(\psi)$ on $L^1(\RRnu,\CC)$, such that  $\FF(\psi)= F(\varphi)$ and
\begin{equation}\label{eq:eq38bis}
 \FF(\psi) =\int_{\RRnu} f(\bfv)\, \overline{\psi(\bfv)} \, d\bfv\, ,
\end{equation}
with $\vert\, \FF(\psi) \, \vert \, < +\infty$ (due to Eq.~\eqref{eq:eq37}).
Since $\curC_0\subset L^\infty(\RRnu,\CC)$, the right-hand side of Eq.~\eqref{eq:eq38bis} can be seen as the duality bracket of
$L^\infty(\RRnu,\CC)$ and $L^1(\RRnu,\CC)$. Similarly, $\varphi\mapsto F_\pexp(\varphi)$ can be viewed as a functional $\psi\mapsto \FF_\pexp(\psi) = F_\pexp(\varphi)$ on $L^1(\RRnu,\CC)$.
\end{remark}
%
%---DEFINITION 2 ----------------
\begin{definition}[Weak formulation of the constraint] \label{definition:2}
{\textit{Lemma~\ref{lemma:2} and Theorem~\ref{theorem:1} can be applied to function
$\bfv\mapsto f(\bfv) = E\{\, \exp( i\, \langle\, \bfv  ,  \bfH\,\rangle ) \,\}$ and to function
$\bfv\mapsto f_\pexp(\bfv) = E\{\, \exp( i\, \langle \bfv  ,  [V]^T\widetilde\bfQ_\pexp\,\rangle) \,\}$ on $\RRnu$.
For all $\varphi$ in $\curH_1$, we have
\begin{equation}\label{eq:eq39}
 F(\varphi) =\int_{\RRnu} f(\bfv)\, \overline{\psi(\bfv)} \, d\bfv\,  \geq  \, 0 \quad , \quad
 F_\pexp(\varphi) =\int_{\RRnu} f_\pexp(\bfv)\, \overline{\psi(\bfv)} \, d\bfv\,  \geq  \, 0 \, ,
\end{equation}
in which $\psi=\overline{\varphi^\vee} \ast \varphi  \in  L^1(\RRnu,\CC)$ (see Theorem~\ref{theorem:1}-(i)). We thus define a weak formulation of the constraint $\{\, f(\bfv) = f_\pexp(\bfv)$, $\forall\, \bfv\in\RRnu\,\}$, as follows
\begin{equation}\label{eq:eq40}
\forall\,\varphi \in \curH_1 \quad , \quad F(\varphi)= F_\pexp(\varphi) \, .
\end{equation}
Using Eqs.~\eqref{eq:eq37} and \eqref{eq:eq38} for $F$ and $F_\pexp$ yields,
for all $\varphi$ in $\curH_1$ and $\psi=\overline{\varphi^\vee} \ast \varphi$ in $L^1(\RRnu,\CC)$,
\begin{equation}\label{eq:eq41}
E\{\,\hat\psi(\bfH) \,\} = \curB^c(\hat\psi) \, ,
\end{equation}
in which $\hat\psi(\bfeta) = \vert\,\hat\varphi(\bfeta)\,\vert^2 \in\curC_0\cap L^1(\RRnu,\CC)$ ,
$\{\bfeta\mapsto\hat\varphi(\bfeta) = \int_{\,\RRnu} e^{-i\,\langle\,\bfeta\, ,\, \pbfv\,\rangle}\,\varphi(\bfv)\, d\bfv \}\in\curH_0$
(see Theorem~\ref{theorem:1}-(ii)), and
\begin{equation}\label{eq:eq42}
 \curB^c(\hat\psi) =E\{\,\hat\psi([V]^T\widetilde\bfQ_\pexp) \,\} \, < +\infty \, .
\end{equation}
For $N_r$ sufficiently large, the right-hand side of Eq.~\eqref{eq:eq42} can be estimated with the realizations of $\widetilde\bfQ_\pexp$, yielding
\begin{equation}\label{eq:eq43}
 \curB^c(\hat\psi) = \frac{1}{N_r} \sum_{r'=1}^{N_{r}} \hat\psi([V]^T\tilde\bfq_\pexp^{r'}) \, < +\infty \, .
\end{equation}
}}
\end{definition}
%
%
%----- REMARK 6 ------------------------------
\begin{remark} [Rewriting the weak formulation] \label{remark:6}
It can easily be seen that the weak formulation defined by Eq.~\eqref{eq:eq40} can be written as follows
\begin{equation}\label{eq:eq43bis}
\varphi\in\curH_1 \quad , \quad \psi=\overline{\varphi^\vee} \ast \varphi\in L^1(\RRnu,\CC)
\quad , \quad \FF(\psi) = \FF_\pexp(\psi)\, ,
\end{equation}
in which using Remark~\ref{remark:5} yields
\begin{equation}\label{eq:eq43ter}
\FF(\psi) = \int_{\,\RRnu} f(\bfv)\, \overline{\psi(\bfv)}\, d\bfv  = E\{ \hat\psi(\bfH) \} \, ,
\end{equation}
\begin{equation}\label{eq:eq43quart}
\FF_\pexp(\psi) = \int_{\,\RRnu} f_\pexp(\bfv)\, \overline{\psi(\bfv)}\, d\bfv  = \curB^c( \hat\psi)  \, .
\end{equation}
\end{remark}
\section{Construction and analysis of a finite representation of the constraint derived from the weak formulation}
\label{sec:Section5}
We have seen that the weak formulation of constraint was defined for $\varphi\in\curH_1$ (see Definition~\ref{definition:2}).
If the Gaussian KDE of $p_{\bfH_\pexp}$ was used with a finite number of realizations $\tilde\bfeta_\pexp^r = [V]^T \tilde\bfq_\pexp^r$ of $\bfH_\pexp$, then $p_{\bfH_\pexp}$ would be decreasing as $\bfeta\mapsto\exp(- \Vert\,\bfeta\, \Vert^2 /(2\tilde s^2))$ in which $\tilde s$ would be, for instance, the Silverman bandwidth \cite{Bowman1997}. Then the Fourier transform would be decreasing as
$\bfv\mapsto\exp(-(\tilde s^2/2)\, \Vert\,\bfv\, \Vert^2)$. This remark leads us to restrict the weak formulation defined by Eq.~\eqref{eq:eq40} to a subspace
$\curH_{1 , \,\mu}\subset \curH_1$ defined as follows.
%
%---DEFINITION 3 ----------------
\begin{definition}[Definition of vector space $\curH_{1,\,\mu}$] \label{definition:3}
{\textit{
Let $\mu(d\bfv)$ be the probability measure on $\RRnu$ defined by
\begin{equation}\label{eq:eq44}
\mu(d\bfv) = p_\nu(\bfv)\, d\bfv \quad , \quad p_\nu(\bfv) = s^{\nu}\left( \frac{\nu}{2\pi}\right)^{\nu/2}  \, \exp\left(-\frac{\nu s^2}{2} \Vert\,\bfv\, \Vert^2\right )\, ,
\end{equation}
in which $s$ is written as
\begin{equation}\label{eq:eq45}
s= \left (\frac{4}{N_r(2+\nu)} \right)^{1/(\nu + 4)} \, .
\end{equation}
The subspace $\curH_{1,\,\mu}$ of $\curH_1$ is then defined by
\begin{equation}\label{eq:eq46}
\curH_{1,\,\mu} = \{\,\varphi : \RRnu\rightarrow \CC\,\, , \,\, \varphi(\bfv) = \varphitilde(\bfv)\, p_\nu(\bfv) \,\, , \,\,
\varphitilde \in L^2_\mu(\RRnu,\CC)\, \}\, ,
\end{equation}
in which the Hilbert space $L^2_\mu(\RRnu,\CC)$ is equipped with the inner product and the associated norm,
\begin{equation}\label{eq:eq47}
(\,\varphitilde \, , \varphitilde' )_{L_\mu^2} = \int_{\,\RRnu}  \varphitilde(\bfv)\,
\overline{\varphitilde'(\bfv)}\,  p_\nu(\bfv)
\quad , \quad
\Vert\,\varphitilde\,\Vert_{L_\mu^2} = \bigg( \int_{\,\RRnu}  \vert\, \varphitilde(\bfv)\,
\vert^2\,  p_\nu(\bfv) \bigg)^{1/2}\, .
\end{equation}
}}
\end{definition}
%
%
%----- REMARK 7 ------------------------------
\begin{remark} \label{remark:7}

(i) The choice of the probability measure defined by Eq.~\eqref{eq:eq44} will appear later. But already now, it can be seen that, for all $\bfeta$ in $\RRnu$,
\begin{equation}\label{eq:eq49}
\hat p_\nu(\bfeta) = \int_{\,\RRnu} e^{-i\,\langle\,\bfeta\, , \,\pbfv\,\rangle} \, p_\nu(\bfv)\, d\bfv  =
\exp\bigg(-\frac{1}{2 \nu s^2} \Vert\,\bfeta\, \Vert^2\bigg ) \, .
\end{equation}
Since $s\rightarrow 0$ as $N_r\rightarrow +\infty$, the sequence of measures
$(\sqrt{2\pi} \, s \, \sqrt{\nu})^{-\nu} \, \hat p_\nu(\bfeta) \, d\bfeta$ converges to the Dirac measure $\delta_0(\bfeta)$ on $\RRnu$ in the space of the bounded measures on $\RRnu$.

\noindent(ii) Parameter $\nu$ has been introduced in the exponential of $p_\nu$ for numerical conditioning. It can be seen that, if  all the components of $\bfeta$ are of order $1$, then $\Vert\,\bfeta\,\Vert^2 \sim \nu$ and consequently, $\Vert\,\bfeta\,\Vert^2 / \nu \sim 1$.

\noindent (iii) Note also that $s$ defined by Eq.~\eqref{eq:eq45} is the Silveman bandwidth corresponding to $N_r$ realizations of $\widetilde\bfQ_\pexp$ and not to the $N_d$ realizations of $\widetilde\bfQ$. As we explained, we have to  construct a finite representation of the constraint, which is consistent with a "projection on the model" of the target set of the realizations, that is to say, of the realizations of the random variable $\bfH_\pexp = [V]^T \widetilde\bfQ_\pexp$. We recall that the pdf of $\bfH_\pexp$ is assumed to be unknown and will not be estimated with the Gaussian KDE from the training set $D_\pexp =\{\bfq_\pexp^1,\ldots , \bfq_\pexp^{N_r}\}$. Indeed, we have only assumed that the unknown probability measure $P_{\bfH_\pexp}$ of $\bfH_\pexp$, which is unknown, admits a density $p_{\bfH_\pexp}$ with respect to $d\bfeta$, which belongs to
$C^0(\RRnu,\CC)\cap\curH_1$ (see Eq.~\eqref{eq:eq20}).
\end{remark}
%
%---- THEOREM 2 ----------------------
\begin{theorem}[Properties of the Fourier transform of $\curH_{1,\,\mu}$] \label{theorem:2}
(i) For all $\varphitilde\in L^2_\mu(\RRnu,\CC)$, the complex-valued function
$\bfv\mapsto\varphi(\bfv) = \varphitilde(\bfv)\, p_\nu(\bfv)$ on $\RRnu$ belongs to $\curH_{1 ,\,\mu} \subset \curH_1$,
in which $p_\nu$ is defined by Eq.~\eqref{eq:eq44}. Let $\bfeta \mapsto \hat\varphi(\bfeta) = \int_{\,\RRnu} e^{-i\,\langle\, \bfeta\, , \, \pbfv\, \rangle} \varphi(\bfv)\, d\bfv$  be the Fourier transform of $\varphi$ on $\RRnu$. Then $\hat\varphi$ belongs to $\curH_0$.
(ii) The complex-valued function $\hat\varphi$ is analytic on $\RRnu$.  (iii) Let $\bfH_\pexp = [V]^T\widetilde\bfQ_\pexp$ be the $\RRnu$-valued random variable defined in Lemma~\ref{lemma:1}, whose $N_r$ realizations are $\{ \bfeta_\pexp^r = [V]^T \tilde\bfq_\pexp^r , r=1,\ldots , N_r\}$. If $\hat\varphi(\bfeta_\pexp^r ) = 0$ for all $r$ in $\NN^*=\NN\backslash {0}$, then $\varphitilde = 0$, $d\bfv$-almost everywhere.
\end{theorem}
%
%------ PROOF OF THEOREM 2 ------------------
\begin{proof} (Theorem~\ref{theorem:2}).
(i)
$\Vert\,\varphi\,\Vert_{L^1} = \!\int_{\,\RRnu} \vert\, \varphitilde(\bfv)\,\vert \, p_\nu(\bfv) \, d\bfv$
$ = \!\int_{\,\RRnu} \vert\,\varphitilde(\bfv)\,\vert \, p_\nu(\bfv)^{1/2}  \, p_\nu(\bfv)^{1/2} \, d\bfv$
$ \leq \left (\int_{\,\RRnu} \vert\, \varphitilde(\bfv)\,\vert^2 \, p_\nu(\bfv) \, d\bfv\right )^{1/2} \times$
   $ \left (\int_{\,\RRnu}  p_\nu(\bfv) \, d\bfv \right )^{1/2} $
$ \! = \Vert\,\varphitilde\,\Vert_{L^2_\mu} \, < +\infty$
because $\varphitilde \in L^2_\mu(\RRnu,\CC)$. In addition, it can be seen that
$\Vert\,\varphi\,\Vert_{L^2} = \int_{\,\RRnu} \vert\,\varphitilde(\bfv)\,\vert^2 \, p_\nu(\bfv)^2\, d\bfv$
$\leq (\sup_\pbfv  p_\nu(\bfv)) \int_{\,\RRnu} \vert\,\varphitilde(\bfv)\,\vert^2 \, p_\nu(\bfv)\, d\bfv$
$= s^{\nu} \,( (\nu/(2\pi) )^{\nu/2} \, \Vert\,\varphitilde\,\Vert_{L^2_\mu}^2 < +\infty$.
Consequently, $\varphi\in L^1(\RRnu,\CC)\cap L^2(\RRnu,\CC) =\curH_1$.
Since $\varphi\in\curH_1$, its Fourier transform $\hat\varphi$ belongs to $\curH_0$.

\noindent (ii) We have now to prove that the complex-valued function $\hat\varphi$ in analytic on $\RRnu$ using a proof similar to the one of Proposition~II.2.36 of \cite{Kree1986}).
For $j=1,\ldots,\nu$, let $z_j= u_j +i\, \eta_j\in\CC$ with $u_j$ and $\eta_j$ in $\RR$.
Let $\varepsilon > 0$ be a given real number. Let $\bfu=(u_1,\ldots, u_{\nu})$ be in $\RRnu$ such that
$\max_j \,\vert\, u_j\,\vert \leq \varepsilon$. Consequently, for all $\bfv=(\vc_1,\ldots, \vc_{\nu})$ in $\RRnu$,
$\vert \,e^{\,\sum_j z_j \pvc_j} \vert = e^{\,\sum_j u_j \pvc_j} \leq e^{\,\varepsilon \, \vert \, \pbfv\,\vert}$,
with $\vert\, \bfv\, \vert = \vert \vc_1\vert +\ldots + \vert \vc_{\nu}\vert$. Let us consider the Laplace transform
$\widetilde \varphi(\bfz) = \int_{\,\RRnu} e^{\,\sum_j z_j \pvc_j} \varphitilde(\bfv)\, p_\nu(\bfv)\, d\bfv$ for $\bfz\in\CC^{\nu}$ of
function $\bfv\mapsto \varphi(\bfv) = \varphitilde(\bfv)\, p_\nu(\bfv)$. We have $\vert\, \widetilde \varphi(\bfz)\, \vert\, \leq
\int_{\,\RRnu} e^{\,\varepsilon \, \vert \, \pbfv\,\vert} \, \vert \, \varphitilde(\bfv)\, \vert \, p_\nu(\bfv)\, d\bfv$
$ = \int_{\,\RRnu}\vert\, \varphitilde(\bfv)\, \vert \,p_\nu(\bfv)^{1/2}\,
e^{\,\varepsilon \, \vert \, \pbfv\,\vert} \, p_\nu(\bfv)^{1/2}\, d\bfv$
$\leq \left ( \int_{\,\RRnu}\vert\, \varphitilde(\bfv)\, \vert^2 \,p_\nu(\bfv)\, d\bfv \right )^{1/2}\,
\left ( \int_{\,\RRnu} e^{\,2\varepsilon \, \vert \, \pbfv\,\vert} \, p_\nu(\bfv)\, d\bfv\right )^{1/2} < +\infty$
because we have  $\int_{\,\RRnu}\vert\, \varphitilde(\bfv)\, \vert^2 \,p_\nu(\bfv)\, d\bfv = \Vert\,\varphitilde\, \Vert^2_{L_\mu^2}\, < +\infty$ and
$\int_{\,\RRnu} e^{\,2\varepsilon \, \vert \, \pbfv\,\vert} \, p_\nu(\bfv)\, d\bfv
= s^{\nu} \, (\nu/(2\pi))^{\nu/2} \int_{\,\RRnu} \exp( -(\nu s^2/2) \,\Vert \, \bfv\, \Vert^2\,  + \, 2\varepsilon\, \vert\,\bfv\, \vert) \, d\bfv \, <  +\infty$.
Consequently, $\bfz\mapsto\widetilde\varphi(\bfz)$ exits in the domain
$D_\varphi = \{\bfz\in\CC^{\nu}\, ,\, u_j \in ] -\varepsilon\, , \varepsilon\,[ \, , \,  j=1,\ldots , \nu \}$ and is a holomorphic function in $D_\varphi\subset \CC^{\nu}$. Therefore, the conjugate $\overline{\hat\varphi(\bfeta)}$ of the Fourier transform
$\hat\varphi(\bfeta) = \int_{\,\RRnu} e^{-i\,\langle\, \bfeta\, , \, \pbfv\, \rangle} \varphi(\bfv)\, d\bfv$
can be written as $\overline{\hat\varphi(\bfeta)} = \widetilde\varphi(\bfzero + i\, \bfeta)$, which shows that
$\bfeta\mapsto \hat\varphi(\bfeta)$ is a $\CC$-valued analytic function on $\RRnu$.

\noindent (iii) Finally, we have to prove the last assertion of Theorem~\ref{theorem:2}. It should be noted that, although $\hat\varphi$ is an analytic function on $\RRnu$, the conditions $\hat\varphi(\bfeta_\pexp^r ) = 0$ for all $r$ in $\NN^*$ do not imply, \textit{a priori},  that $\varphitilde = 0$ because the independent realizations $\{\bfeta_\pexp^r , r\in\NN^*\}$ constitute a countable number of zeros of $\hat\varphi$.
However, $\hat\varphi(\bfeta_\pexp^r ) = 0$ for $r\in\{1,\ldots , N_r\}$ implies  that
$(1/N_r) \sum_{r=1}^{N_r} \vert \, \hat\varphi(\bfeta_\pexp^r ) \, \vert^2 = 0$ and therefore, for $N_r\rightarrow +\infty$, implies that $E\{\vert\, \hat\varphi(\bfH_\pexp) \, \vert^2 \} = 0$. Since $\hat\varphi\in\curH_0$ (and is also analytic), $\hat\varphi$ is a continuous function on $\RRnu$ and $\bfH_\pexp$ has a probability measure $p_{\bfH_\pexp}(\bfeta)\, d\bfeta$ on $\RRnu$ for which the pdf
belongs to $C^0(\RRnu,\CC)\cap \curH_1$ (see Eq.~\eqref{eq:eq20}). We can then conclude that
$E\{\vert\, \hat\varphi(\bfH_\pexp) \, \vert^2 \} = 0$ implies $\hat\varphi=0$. Since
$\bfv\mapsto \varphi(\bfv) = \varphitilde(\bfv)\, p_\nu(\bfv) \in \curH_{1, \,\mu} \subset\curH_1$,
then the Plancherel equality, $\Vert\, \varphi\,\Vert_{L^2}  =(2\pi)^{-\nu/2} \, \Vert\,\hat\varphi\,\Vert_{L^2}$ shows that $\hat\varphi=0$ implies $\varphi=0$, $d\bfv$- almost everywhere, and since $p_\nu(\bfv) > 0$ for all $\bfv$ in $\RRnu$, then this implies that
$\varphitilde = 0$, $d\bfv$- almost everywhere.
\end{proof}
\subsection{About a possible use of a polynomial representation}
\label{sec:Section5.1}
The weak formulation defined by Eq.~\eqref{eq:eq40}, restricted to subspace $\curH_{1,\,\mu}$ of $\curH_1$, is written as
\begin{equation}\label{eq:eq51}
\forall\,\varphi \in \curH_{1,\,\mu} \quad , \quad F(\varphi)= F_\pexp(\varphi) \, .
\end{equation}
For constructing a finite representation of Eq.~\eqref{eq:eq51}, a classical method consists in performing the expansion of $\varphi\in\curH_{1,\,\mu}$ with respect to the orthogonal polynomials in $L^2_\mu(\RRnu,\CC)$.
Let $\bfalpha=(\alpha_1,\ldots, \alpha_{\nu})$ be the multi-index in $\NN^{\nu}$. We introduce the classical notations:
$\vert\bfalpha\vert = \alpha_1 + \ldots + \alpha_{\nu}$,
$\,\bfalpha ! = \alpha_1 ! \times \ldots \times \alpha_{\nu} !$,
$\,i^{\,\bfalpha} = i^{\,\vert\bfalpha\vert}$,
and for $\bfeta=(\eta_1,\ldots , \eta_{\nu})\in\RRnu$, $\bfeta^\bfalpha = \eta_1^{\alpha_1}\times\ldots\times\eta_{\nu}^{\alpha_{\nu}}$.
For $\bfv=(\vc_1,\ldots, \vc_{\nu})\in\RRnu$, let
$\HH_\bfalpha(\bfv) = \HH_{\alpha_1}(\vc_1)\times\ldots\times\HH_{\alpha_{\nu}}(\vc_{\nu})$
be the multi-index Hermite polynomial on $\RRnu$ of degree $\vert\bfalpha\vert$ such that the real Hermite polynomials $\HH_k(y)$ on $\RR$ are
$\HH_0(y) =1$, $\HH_1(y) = y$, $\HH_2(y)=y^2-1$, $\HH_3(y) = y^3-3y$, etc. It is known that the countable family
$\{ \phitilde_{\bfalpha},\bfalpha \in\NN^\nu\}$ such that
$\phitilde_{\bfalpha}(\bfv) = (\bfalpha !)^{-1/2}\HH_\bfalpha (s\sqrt{\nu}\, \bfv)$
is a Hilbertian basis of $L^2_\mu(\RRnu,\RR)$ and is also a Hilbertian basis of $L^2_\mu(\RRnu,\CC)$ considered as the complexified space of
$L^2_\mu(\RRnu,\RR)$. We then have
$( \phitilde_{\bfalpha}  , \,  \phitilde_{\bfbeta})_{L_\mu^2} = \delta_{\bfalpha\bfbeta}$ for $\bfalpha$ and $\bfbeta$ in $\NN^{\nu}$.
Therefore, any function $\varphitilde$ in $L^2_\mu(\RRnu,\CC)$ can be written as
$\varphitilde(\bfv) = \sum_{\bfalpha,\,\vert\bfalpha\vert\, =\, 0}^{+\infty} \,\xi_\bfalpha \,  {\phitilde}_{\bfalpha}(\bfv)$
in which
$\xi_\bfalpha = ( \varphitilde  , \,  \phitilde_{\bfalpha})_{L_\mu^2}
=\int_{\,\RRnu} \varphitilde(\bfv)\, \phitilde_{\bfalpha}\, p_\nu(\bfv) \, d\bfv$.
The series in the right-hand side member of the expansion of $\varphitilde$  is convergent in $L_\mu^2(\RRnu,\CC)$
and we have $\Vert\,\varphitilde\,\Vert_{L_\mu^2}^2 = \sum_{\bfalpha,\,\vert\bfalpha\vert\, =\, 0}^{+\infty} \,\vert\, \xi_\bfalpha\,\vert^2 \, < +\infty$.
It can be deduced that all function $\varphi$ in $\curH_{1,\,\mu}\subset\curH_1$ can be written as
$\varphi(\bfv) = \sum_{\bfalpha,\,\vert\bfalpha\vert\, =\, 0}^{+\infty} \,\xi_\bfalpha \,  {\phitilde}_{\bfalpha}(\bfv)\, p_\nu(\bfv)$.
For all $\bfv\in\RRnu$, from a classical formula, we can deduce the following one,
\begin{equation}\nonumber
\int_{\,\RRnu} e^{-i\,\langle\, \bfeta\, , \pbfv\,\rangle} \,\HH_\bfalpha (s\sqrt{\nu}\, \bfv)\, p_\nu(\bfv)\, d\bfv =
\left (\frac{-i}{s\sqrt{\nu}}\right )^{\,\vert\,\bfalpha\,\vert } \, \bfeta^\bfalpha\, \exp\left(-\frac{1}{2 \nu s^2} \, \Vert\,\bfeta\, \Vert^2\right )\, .
\end{equation}
Thus, the Fourier transform $\hat\varphi$ of $\varphi\in\curH_{1,\,\mu}$  belongs to  $\curH_0$ and can be written, for all $\bfeta$ in $\RRnu$, as
$\hat\varphi(\bfeta) = \sum_{\bfalpha,\,\vert\bfalpha\vert\, =\, 0}^{+\infty} \,\xi_\bfalpha \,  {\widehat\phi}_{\bfalpha}(\bfeta)$
in which
\begin{equation}\label{eq:eq60}
{\widehat\phi}_{\bfalpha}(\bfeta) = \frac{1}{\sqrt{\bfalpha !}} \left (\frac{-i}{s\sqrt{\nu}}\right )^{\,\vert\,\bfalpha\,\vert } \, \bfeta^\bfalpha\, \exp\left(-\frac{1}{2 \nu s^2} \, \Vert\,\bfeta\, \Vert^2\right ) \quad , \quad \bfalpha\in\NN^{\nu}\, .
\end{equation}
Using the family $\{ \widehat\phi_{\bfalpha} \}_\bfalpha$ defined by Eq.~\eqref{eq:eq60} and Definition~\ref{definition:2}, the finite representation of the weak formulation defined by Eq.~\eqref{eq:eq51}, can be written as,
\begin{equation}\label{eq:eq61}
E\{ h_\bfalpha^c(\bfH)\} = b_\bfalpha^c \quad , \quad \bfalpha\in\{\bfalpha^{(1)},\ldots ,\bfalpha^{(N_r)} \} \, ,
\end{equation}
in which
\begin{equation}\label{eq:eq61bis}
h_\bfalpha^c(\bfeta)\} = \vert\, {\widehat\phi}_{\bfalpha}(\bfeta)\, \vert^2 = \widehat\Psi_\bfalpha(\bfeta) \quad , \quad \bfalpha\in \NN^{\nu} \, ,
\end{equation}
\begin{equation}\label{eq:eq62}
\widehat\Psi_\bfalpha(\bfeta) = \frac{(\nu s^2)^{-\vert\,\bfalpha\,\vert}}{\bfalpha !}
\, \bfeta^{2\bfalpha}\, \exp\left(-\frac{1}{\nu s^2} \, \Vert\,\bfeta\, \Vert^2\right ) \, ,
\end{equation}
and where, for $N_r$ sufficiently large,
\begin{equation}\label{eq:eq63}
b^c_\bfalpha\simeq  \frac{1}{N_r} \sum_{r'=1}^{N_r} \widehat\Psi_\bfalpha([V]^T \tilde\bfq_\pexp^r) \quad , \quad \bfalpha\in \NN^{\nu}\, .
\end{equation}
The finite representation defined by Eq.~\eqref{eq:eq61} with Eqs.~\eqref{eq:eq61bis} and \eqref{eq:eq62} will not be efficient as soon as $\nu$ will be large. In addition, $\widehat\Psi_\bfalpha$ defined by Eq.~\eqref{eq:eq62} does not depend on the sampling defined by the points of the target set $D_\pexp$.
\subsection{Construction of an adapted finite representation of the functional constraint}
\label{sec:Section5.2}
The following Lemma gives the construction of a family in $\curH_{1,\,\mu}$, which is based on the sampling points of the target set $D_\pexp$.
%
%---LEMMA 5 ----------------------
\begin{lemma}[Construction of a family $\{\widehat\psi_r\}_r$ in $\curC_0\cap L^1(\RRnu,\CC)$] \label{lemma:5}
Let $r$ be fixed in $\{1,\ldots , N_r\}$.

\noindent (i) Let $\bfv\mapsto\varphi_r(\bfv)\in\curH_{1 , \,\mu}$ defined, for all $\bfv$ in $\RRnu$, by
\begin{equation}\label{eq:eq64}
\varphi_r(\bfv) = \varphitilde_r(\bfv) \, p_\nu(\bfv) \quad , \quad \varphitilde_r(\bfv) =
\exp(\,i\,\langle\bfv ,  [V]^T \tilde\bfq_\pexp^r\rangle) \, ,
\end{equation}
in which $p_\nu$ is defined by Eq.~\eqref{eq:eq44} with Eq.~\eqref{eq:eq45} and where $\varphitilde_r\in L_\mu^2(\RRnu,\CC)$ such that
$\Vert\, \varphitilde_r \Vert_{L_\mu^2} = 1$. Then the Fourier transform
$\bfeta\mapsto \hat\varphi_r(\bfeta) = \int_{\,\RRnu} e^{\, -i\,\langle \, \bfeta\, , \, \pbfv\,\rangle}\, \varphi_r(\bfv) \, d\bfv$
belongs to $\curH_0$ and is written as
\begin{equation}\label{eq:eq65}
\hat\varphi_r(\bfeta) = \exp\bigg ( -\frac{1}{2 \nu s^2} \,\Vert \,\bfeta - [V]^T \tilde\bfq_\pexp^r\, \Vert^2\bigg ) \quad , \quad \forall\, \bfeta\in \RRnu\, .
\end{equation}
Eq.~\eqref{eq:eq65} shows that $\hat\varphi_r$, which is analytic (see Theorem~\ref{theorem:2}), also belongs to $\curH_0\cap L^q(\RRnu,\CC)$ for all $3\leq q < +\infty$.

\noindent (ii) Let $\psi_r\in L^1(\RRnu,\CC)$ be the function defined by
$\psi_r = \overline{\varphi_r^\vee} \ast \varphi_r$. Its Fourier transform is such that
(see Theorem~\ref{theorem:1}) $\hat\psi_r \in \curC_0\cap L^1(\RRnu,\CC)$ and is written as
\begin{equation}\label{eq:eq66}
\hat\psi_r(\bfeta) = \vert\, \hat\varphi_r(\bfeta)\, \vert^2 \, = \exp\bigg ( -\frac{1}{\nu s^2} \,\Vert \,\bfeta - [V]^T \tilde\bfq_\pexp^r\, \Vert^2\bigg ) \quad , \quad \forall\, \bfeta\in \RRnu\, .
\end{equation}
Note that $\hat\psi_r$ is also in $\curH_1\cap L^q(\RRnu,\CC)$ for all $3\leq q < +\infty$ and
function $\psi_r$ is written as
\begin{equation}\label{eq:eq67}
\psi_r(\bfv) = \bigg ( \frac{s}{2} \sqrt{\frac{\nu}{2\pi}}  \bigg )^{\nu/2} \! p_\nu(\bfv)^{1/2} \,
 \exp \bigg(\,i\,\langle\,\bfv ,  [V]^T \tilde\bfq_\pexp^r\,\rangle \bigg) \,   \quad , \quad \forall \bfv\in\RRnu\, .
\end{equation}
\end{lemma}
%
%------ PROOF LEMMA 5 ------------------
\begin{proof} (Lemma~\ref{lemma:5}).
(i) We have $\Vert\,\varphitilde_r\,\Vert_{L_\mu^2} =\int_{\,\RRnu} p_\nu(\bfv)\, d\bfv = 1$. From Eqs.~\eqref{eq:eq44} and \eqref{eq:eq64}, and introducing $\sigma=1/(s\sqrt{\nu})\,$ yield
$\,\hat\varphi_r(\bfeta) = \int_{\,\RRnu} e^{\, -i\,\langle \, \bfeta\, , \, \pbfv\,\rangle}\,
 \exp (\,i\,\langle\,\bfv ,  [V]^T \tilde\bfq_\pexp^r\,\rangle ) \,
\bigg ( s \!\sqrt{\frac{\nu}{2\pi}}  \bigg )^{\nu} \! \, \exp\bigg ( -\frac{\nu s^2}{2} \,\Vert \,\bfv \, \Vert^2\bigg )\, d\bfv$
$ =\! \int_{\,\RRnu} \exp ( \,i\,\langle\,\bfv ,  [V]^T \tilde\bfq_\pexp^r - \bfeta\,\rangle ) \,
(\sqrt{2\pi}\,\sigma)^{-\nu} \! \, \exp\bigg ( -(2\sigma^2)^{-1} \,\Vert \,\bfv \, \Vert^2\bigg )\, d\bfv$,  which gives Eq.~\eqref{eq:eq65}.
\noindent (ii) Eq.~\eqref{eq:eq66} is obtained by substituting Eq.~\eqref{eq:eq65} in $\hat\psi_r(\bfeta) = \vert\,\hat\varphi_r(\bfeta) \, \vert^2$.  For all $\bfv\in\RRnu$, we have
$\psi_r(\bfv) =(2\pi)^{-\nu} \int_{\,\RRnu}e^{\, i\,\langle \, \pbfv , \, \bfeta\,\rangle}\, \hat\psi_r(\bfeta)\, d\bfeta$
$ = (2\pi)^{-\nu} \int_{\,\RRnu}e^{\, i\,\langle \, \pbfv , \, \bfeta\,\rangle}\, \exp\bigg ( -\frac{1}{\nu s^2} \,\Vert \,\bfeta - [V]^T \tilde\bfq_\pexp^r\, \Vert^2 \bigg )\, d\bfv$
$ = (s/2)^\nu \,(\nu/\pi)^{\nu/2} \,
 \exp \bigg(\,i\,\langle\,\bfv ,  [V]^T \tilde\bfq_\pexp^r\,\rangle \bigg)  -\frac{\nu s^2}{4} \,\Vert \,\bfv  \Vert^2\bigg )$,
 which can be rewritten as Eq.~\eqref{eq:eq67}  by using Eq.~\eqref{eq:eq44}.
\end{proof}
%
%
%---LEMMA 6 ----------------------
\begin{lemma}[Orthonormal family $\{\varphitilde_r, r=1,\ldots , N_r\}$ in $L_\mu^2(\RRnu,\CC)$ for $N_r\rightarrow +\infty$ ] \label{lemma:6}
For $r\in\{1,\ldots,N_r\}$, let $\varphitilde_r$ be the function in $L_\mu^2(\RRnu,\CC)$ defined by Eq.~\eqref{eq:eq64}.
For $N_r \rightarrow +\infty$, the family $\{\varphitilde_r, r=1,\ldots , N_r\}$ in $L_\mu^2(\RRnu,\CC)$  goes to an orthonormal family  in
$L_\mu^2(\RRnu,\CC)$:
$\lim_{N_r\rightarrow+\infty}\,\, (\varphitilde_r , \varphitilde_{r\,'} )_{L_\mu^2} = \delta_{r r\,'}$.
Let $L_\mu^{2,(N_r)}= \hbox{span}\, \{\,\varphitilde_1,\ldots , \varphitilde_{N_r}\}$ be the subspace of $L_\mu^2(\RRnu,\CC)$ spanned by $\{\varphitilde_r, r=1,\ldots , N_r\}$.
For $N_r\rightarrow +\infty$, the sequence of subspaces $L_\mu^{2,(N_r)}$ goes to a subspace that is dense in $L_\mu^2(\RRnu,\CC)$.
\end{lemma}
%
%------ PROOF LEMMA 6 ------------------
\begin{proof} (Lemma~\ref{lemma:6}).
We have
$(\varphitilde_r , \varphitilde_{r\,'} )_{L_\mu^2} = \int_{\,\RRnu} \varphitilde_r(\bfv) \, \overline{\varphitilde_{r\,'}(\bfv)}\, p_\nu(\bfv)\, d\bfv = \int_{\,\RRnu} \exp ( \, i\,\langle\,\bfv , [V]^T(\tilde\bfq_\pexp^r - \tilde\bfq_\pexp^{r\,'} \, \rangle) \, p_\nu(\bfv)\, d\bfv$.
Using Eq.~\eqref{eq:eq49} allows for writing
$(\varphitilde_r , \varphitilde_{r\,'} )_{L_\mu^2} = \exp ( -(2\nu s^2)^{-1} \, \Vert \, [V]^T(\tilde\bfq_\pexp^r - \tilde\bfq_\pexp^{r\,'}\, \Vert^2 )$.
For $r=r\,'$, we have $\Vert\,\varphitilde_r\Vert_{L_\mu^2} = 1$. Let us now consider the case  $r\not = r\,'$. For $N_r\rightarrow +\infty$, we have $s\rightarrow 0$, and consequently, $(\varphitilde_r , \varphitilde_{r\,'} )_{L_\mu^2} \rightarrow 0$. We then have proven the first part of the Lemma. We have now  to prove that, for any $\varphitilde$ in $L_\mu^2(\RRnu,\CC)$ and for $N_r\rightarrow +\infty$ if
$(\varphitilde , \varphitilde_{r} )_{L_\mu^2} = 0\, ,\forall\, r$, then $\varphitilde = 0$.
We have
$(\varphitilde , \varphitilde_{r} )_{L_\mu^2} = \int_{\,\RRnu} \varphitilde(\bfv) \, \overline{\varphitilde_{r}(\bfv)}\, p_\nu(\bfv)\, d\bfv$
$= \int_{\,\RRnu} \exp ( - i\,\langle\,\bfv , [V]^T \tilde\bfq_\pexp^r \, \rangle) \, \varphi(\bfv)\, d\bfv$,
in which $\varphi(\bfv) = \varphitilde(\bfv)\, p_\nu(\bfv)$. We then obtain
$(\varphitilde , \varphitilde_{r} )_{L_\mu^2} = \hat\varphi([V]^T \tilde\bfq_\pexp^r)$.
Using Theorem~\ref{theorem:2}, the condition
$\hat\varphi([V]^T \tilde\bfq_\pexp^r)= 0$, $\forall\, r\in\NN^*$ implies that $\varphitilde = 0$ $d\bfv$-almost everywhere. We then have proven the Lemma.
\end{proof}
%
%
%---DEFINITION 4 ----------------
\begin{definition}[Finite representation of the functional constraint] \label{definition:4}
{\textit{
Using the family $\{\hat\psi_r, r=1,\ldots , N_r\}$ in $\curC_0\cap L^1(\RRnu,\CC)$ defined in Lemma~\ref{lemma:5} (see Eq.~\eqref{eq:eq66}, taking into account Lemma~\ref{lemma:6} and using Definition~\ref{definition:2} of the weak formulation of the constraint (see Eqs.~\eqref{eq:eq41} to \eqref{eq:eq43}), restricted to $\varphi\in\curH_{1,\,\mu} \subset\curH_1$, the finite representation of the constraint is written as
\begin{equation}\label{eq:eq72}
E\{ \bfh^c(\bfH)\} = \bfb^c \,\, \hbox{on} \,\, \RR^{N_r} \, ,
\end{equation}
in which $\bfh^c(\bfeta)  = (h_1^c(\bfeta),\ldots , h_{N_r}^c(\bfeta))$ and
$\bfb^c  = (b_1^c,\ldots , b_{N_r}^c)$  are the vectors in $\RR^{N_r}$, which are written, for $r\in\{1,\ldots , N_r\}$, as
\begin{equation}\label{eq:eq73}
h_r^c(\bfeta) = \exp\bigg ( -\frac{1}{\nu s^2} \,\Vert \,\bfeta - [V]^T \tilde\bfq_\pexp^r\, \Vert^2\bigg ) \quad , \quad \forall\, \bfeta\in \RRnu\, .
\end{equation}
\begin{equation}\label{eq:eq74}
b_r^c = E\bigg\{ \exp\bigg ( -\frac{1}{\nu s^2} \,\Vert \, [V]^T(\widetilde\bfQ_\pexp - \tilde\bfq_\pexp^r\, \Vert^2\bigg ) \bigg \} \, ,
\end{equation}
which can be estimated, for $N_r$ sufficiently large, by
\begin{equation}\label{eq:eq75}
b_r^c = \frac{1}{N_r} \sum_{r\,'=1}^{N_r} \exp\bigg ( -\frac{1}{\nu s^2} \,\Vert \, [V]^T(\tilde\bfq_\pexp^{r\,'} - \tilde\bfq_\pexp^r\, \Vert^2\bigg ) \, .
\end{equation}
}}
\end{definition}
%
%
%----- REMARK 8 ------------------------------
\begin{remark} \label{remark:8}

(i) Definition~\ref{definition:4} shows that if the random variables $\bfH$ and $\bfH_\pexp = [V]^T\widetilde\bfQ_\pexp$ are isonomic, then the constraint defined by Eq.~\eqref{eq:eq72} is exactly satisfied.
Consequently, the use of the Kullback-Leibler minimum principle for estimating the posterior probability measure under this constraint will be well posed.

\noindent (ii) Let $\hat\psi^{(N_r)}$ be the function in $\curC_0\cap L^1(\RRnu,\CC)$, such that for all $\bfeta\in\RRnu$,
$\hat\psi^{(N_r)}(\bfeta) =\frac{1}{N_r} \sum_{r=1}^{N_r} (s\sqrt{\nu\pi})^{-\nu}\hat\psi_r(\bfeta)$.
Using Eq.~\eqref{eq:eq66}, we have the following equality in the space of the bounded measures,
$\hat\psi^{(N_r)}(\bfeta)\, d\bfeta =\frac{1}{N_r} \sum_{r=1}^{N_r} (\sqrt{2\pi}\,\sigmatilde)^{-\nu}$
$\exp\bigg ( - (2\sigmatilde^2)^{-1} \,\Vert \,\bfeta - [V]^T \tilde\bfq_\pexp^r\, \Vert^2\bigg )\, d\bfeta$,
in which $\sigmatilde = s\sqrt{\nu/2}$. For $N_r\rightarrow +\infty$, since $\sigmatilde \rightarrow 0$  because $s\rightarrow 0$,
it can be seen that the right-hand side of this last equality goes to  the probability measure $p_{\bfH_\pexp}(\bfeta)\, d\bfeta$
in which $\{\,\bfeta_\pexp^r = [V]^T \tilde\bfq_\pexp^r\, , r=1,\ldots, N_r\,\}$ are $N_r$ independent realizations of
$\bfH_\pexp = [V]^T \widetilde\bfQ_\pexp$. We then have the following convergence property in the space of the bounded measures,
$\lim_{N_r\rightarrow +\infty} \hat\psi^{(N_r)}(\bfeta) \, d\bfeta = p_{\bfH_\pexp}(\bfeta)\, d\bfeta$. This result contributes to justify the construction presented in Definition~\ref{definition:4}.
\end{remark}
The following Lemma will be used in the next section for analyzing the existence and uniqueness of the posterior probability measure constructed by using the Kullback-Leibler minimum principle.
%
%---LEMMA 7 ----------------------
\begin{lemma}[Positive definiteness of matrix $E\{\bfh^c(\bfH)\otimes \bfh^c(\bfH)\}$] \label{lemma:7}
Let us consider any finite positive fixed value of integer $N_r$. Let $\bfH$ be the $\RRnu$-valued random variable whose probability measure $P_{\bfH}(d\bfeta)$ is such that the following $\MM^{+0}_{N_r}$-valued matrix exits,
\begin{equation}\label{eq:eq75bis}
E\{\bfh^c(\bfH)\otimes \bfh^c(\bfH)\} = \int_{\RRnu} \bfh^c(\bfeta)\otimes \bfh^c(\bfeta)\, P_{\bfH}(d\bfeta) \, ,
\end{equation}
in which, $h_r^c(\bfeta) = \exp\bigg (- (\nu s^2)^{-1} \,\Vert\,\bfeta - \bfeta_\pexp^r\Vert^2\bigg)$ with
$\bfeta_\pexp^r  = [V]^T \tilde\bfq_\pexp^r$ for $r=1,\ldots, N_r$. Therefore, this matrix is positive definite,
\begin{equation}\label{eq:eq75ter}
E\{\bfh^c(\bfH)\otimes \bfh^c(\bfH)\} \in \MM_{N_r}^+ \, .
\end{equation}
\end{lemma}
%
%------ PROOF LEMMA 7 ------------------
\begin{proof} (Lemma~\ref{lemma:7}).
For all $\bfw=(w_1,\ldots , w_{N_r}) \in \RR^{N_r}$, we have
$\langle\, E\{\bfh^c(\bfH)\otimes \bfh^c(\bfH)\}\, \bfw \, , \bfw\,\rangle
= E\{ \langle \,\bfh^c(\bfH) \, , \bfw\,\rangle^2\} \geq 0$.
We then have to prove that, for all $\bfw$ in $\RR^{N_r}$ with $\Vert\,\bfw\,\Vert \,\,\not = \! 0$, we have
$E\{ \langle \,\bfh^c(\bfH) \, , \bfw\,\rangle^2\} >  0$ or equivalently, that
$E\{ \langle \,\bfh^c(\bfH) \, , \bfw\,\rangle^2\} =  0 \Leftrightarrow \bfw = 0$.
We have
$E\{ \langle \,\bfh^c(\bfH) \, , \bfw\,\rangle^2\} = \int_{\RRnu} \langle\, \bfh^c(\bfeta) \, , \bfw\,\rangle^2 P_{\bfH}(d\bfeta) = 0$
$\Longleftrightarrow \int_{\RRnu} \langle \,\bfh^c(\bfeta) \, , \bfw\,\rangle^2 \, d\bfeta = 0$
$\Longleftrightarrow\sum_{r=1}^{N_r} \sum_{r'=1}^{N_r} w_r w_{r'}\int_{\RRnu}  \exp\big(- (\nu s^2)^{-1}\gamma_{rr'}(\bfeta)\big) \, d\bfeta = 0$
in which for all $\bfeta$ in $\RRnu$,
$\gamma_{rr'}(\bfeta) = \Vert\,\bfeta - \bfeta_\pexp^r\Vert^2 + \Vert\,\bfeta - \bfeta_\pexp^{r'}\Vert^2 \, =$
$2 \, \Vert\,\bfeta - \frac{1}{2}(\bfeta_\pexp^r + \bfeta_\pexp^{r'})\Vert^2
+ \frac{1}{2} \, \Vert \,\bfeta_\pexp^r - \bfeta_\pexp^{r'}\Vert^2 $.
Consequently,
$E\{ \langle\, \bfh^c(\bfH) \, , \bfw\,\rangle^2\} =  0$
$\Longleftrightarrow\sum_{r=1}^{N_r} \sum_{r'=1}^{N_r} w_r w_{r'}
\bigg ( \exp \big (- (2\nu s^2)^{-1}\Vert\, \bfeta_\pexp^r - \bfeta_\pexp^{r'}\Vert^2 \big ) \bigg )\,$
$\int_{\RRnu}  \exp\big (- 2 (\nu s^2)^{-1} \Vert\,\bfeta - \frac{1}{2}(\bfeta_\pexp^r + \bfeta_\pexp^{r'})\Vert^2\big ) \, d\bfeta = 0$
$\Longleftrightarrow  ( (s/2) \sqrt{2\pi \nu})^{\nu} \,\sum_{r=1}^{N_r} \sum_{r'=1}^{N_r} w_r w_{r'}
\exp \big (- (2\nu s^2)^{-1} \Vert \,\bfeta_\pexp^r - \bfeta_\pexp^{r'}\Vert^2 \big ) = 0$
$\Longleftrightarrow \bfw=(w_1,\ldots , w_{N_r}) = \bfzero_{N_r}$.
\end{proof}
\subsection{Illustration of the numerical behavior of the functional constraint}
\label{sec:Section5.3}
In order to illustrate the numerical behavior of the finite representation of the weak formulation of the functional constraint, we consider the following simple numerical case. We assume that $\bfH$ and $\bfH_\pexp = [V]^T\widetilde\bfQ_\pexp$, which are statistically independent, are Gaussian random vectors. Random vector $\bfH$ is centered and with an identity  covariance matrix $[I_\nu]$ (see Eq.~\eqref{eq:eq4}).
The mean value of random vector $\bfH_\pexp$ is written as $\underline\bfeta_\pexp = m_\pexp \, \bfa \in \RR^\nu$ in which $m_\pexp$ is given in $\RR$ and where $\bfa\in\RR^\nu$ is any given realization of a uniform random vector on $[0\, , 1]^\nu$ with independent components. The covariance matrix of $\bfH_\pexp$ is written as $\sigma_\pexp\, [I_\nu]$ in which $\sigma_\pexp$ is given in $]0\, , +\infty[$. For analyzing the numerical behavior of Eqs.~\eqref{eq:eq72} to \eqref{eq:eq75}, we introduce the function $(m_\pexp,\sigma_\pexp)\mapsto J(m_\pexp,\sigma_\pexp) = \Vert \, E\{\bfh^c(\bfH)\} - \bfb^c \, \Vert$ for
$m_\pexp \in [-3\, , 3]$ and $\sigma_   \pexp\in [0.1\, , 2.3]$.
For $\nu=100$, the values of function $J$ are estimated using the realizations $\{\bfeta_d^j,j=1,\ldots , N_d\}$ of $\bfH$ and
the realizations $\{\bfeta_\pexp^r,r=1,\ldots , N_r\}$ of $\bfH_\pexp$ with $N_d=1000$  and $N_r=100$. Fig.~\ref{fig:figure1} displays the graph of function $(m_\pexp,\sigma_\pexp)\mapsto J(m_\pexp,\sigma_\pexp)$. It can be seen that $J$ is effectively minimum in the region centered at point $(m_\pexp=0,\sigma_\pexp=1)$ (what was expected because $\bfH$ and $\bfH_\pexp$ are isonomic when $m_\pexp=0$ and $\sigma_\pexp=1$) while $J$ is larger when $\bfH_\pexp$ is not isonomic to $\bfH$, that is to say for $m_\pexp \not = 0$ and/or $\sigma_\pexp\not = 1$.
\begin{figure}[h!]
\centering
\includegraphics[width=5.5cm]{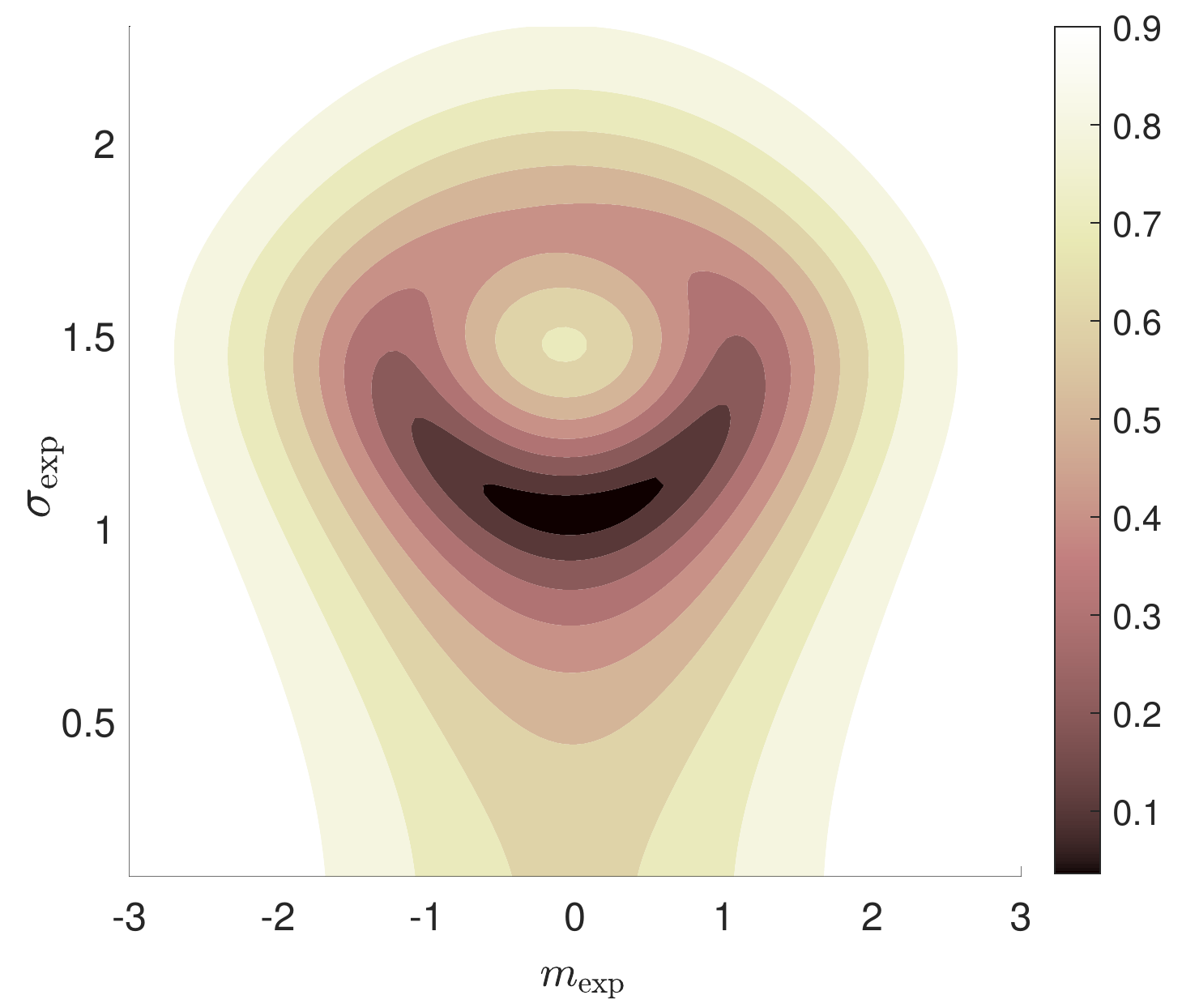}
\caption{Graph of function $(m_\pexp,\sigma_\pexp)\mapsto J(m_\pexp,\sigma_\pexp) = \Vert E\{\bfh^c(\bfH)\} - \bfb^c\Vert$ that illustrates the numerical behavior of the finite representation of the weak formulation of the functional constraint used to identify the posterior probability measure of $\bfH$ from the target set $D_\pexp$.}
\label{fig:figure1}
\end{figure}
\section{Kullback-Leibler minimum principle for estimating the posterior model}
\label{sec:Section6}
In this section we reuse part of the developments that we presented in paper \cite{Soize2022bb}. We do not want to limit ourselves to referring the reader to this reference, because the hypotheses are not the same, the Lemmas and Theorems must be reformulated, and their proofs must be adapted and modified. In addition, the presentation chosen makes it easier to read and understand, thus avoiding going back and forth with this reference.
\subsection{Prior probability measure of $\bfH$}
\label{sec:Section6.1}
Let $P_\bfH(d\bfeta) = p_\bfH(\bfeta)\, d\bfeta$ be the prior probability measure on $\RR^\nu$ of $\bfH$, whose probability density function
$\bfeta\mapsto p_\bfH(\bfeta): \RR^\nu\rightarrow \RR^+$ is estimated by using the Gaussian kernel-density estimation (KDE) with the training set $\curD_d= \{\bfeta_d^1,\ldots , \bfeta_d^{N_d}\}$, involving the modification proposed in \cite{Soize2015} of the classical formulation \cite{Bowman1997} for which $s$ is the Silverman bandwidth,
\begin{equation}\label{eq:eq100}
 p_\bfH(\bfeta) = c_\nu\,\zeta(\bfeta)\quad , \quad  \forall\bfeta\in\RR^\nu\quad  , \quad  c_\nu = (\sqrt{2\pi}\, \hat s)^{-\nu} \, ,
\end{equation}
in which $\hat s = s_\SB\, \left( s^2_\SB + (N_d-1)/N_d \right)^{-1/2}$ with $s_\SB = \left( 4/( N_d (2+\nu) ) \right)^{1/(\nu+4)}$, and where $\bfeta\mapsto\zeta(\bfeta): \RR^\nu\rightarrow \RR^+$ is written as
\begin{equation}\label{eq:eq101}
 \zeta(\bfeta) = \frac{1}{N_d}\sum_{j=1}^{N_d} \exp\bigg ( -\frac{1}{2\hat s^2}\,\Vert\,\frac{\hat s }{s_\SB}
  \, \bfeta^j_d - \bfeta\,\Vert^2 \bigg ) \, .
\end{equation}
We define the potential function $\bfeta\mapsto\phi(\bfeta): \RR^\nu\rightarrow \RR$, related to $p_\bfH$, which will be used in Lemma~\ref{lemma:8} and such that
\begin{equation}\label{eq:eq102}
\zeta(\bfeta) = \exp\{-\phi(\bfeta)\}\, .
\end{equation}
With such a modification and using Eq.~\eqref{eq:eq4}, the normalization of $\bfH$ is preserved for any value of $N_d$, that is to say,
\begin{equation}\label{eq:eq103}
E\{\bfH\} = \int_{\RR^\nu} \bfeta\, p_\bfH(\bfeta)\, d\bfeta = \frac{1}{2\hat s^2}\,\underline{\widehat\bfeta} = \bfzero_\nu\, ,
\end{equation}
\begin{equation}\label{eq:eq104}
E\{\bfH\otimes\bfH\} = \int_{\RR^\nu} \bfeta\otimes\bfeta\, p_\bfH(\bfeta)\, d\bfeta = \hat s^2 \,[I_\nu] +
     \frac{\hat s^2}{s^2} \frac{(N_d-1)}{N_d}\,[\widehat C_\bfH] = [I_\nu]\, .
\end{equation}
Theorem~3.1 in \cite{Soize2020c} proves that, for all $\bfeta$ fixed in $\RR^\nu$, Eq.~\eqref{eq:eq100} with Eq.~\eqref{eq:eq101} is a consistent estimation of the sequence $\{p_\bfH\}_{N_d}$ for $N_d\rightarrow +\infty$.
\subsection{Posterior estimate  using the Kullback-Leibler divergence minimum principle}
\label{sec:Section6.2}
The posterior probability density function $\bfeta\mapsto p_{\bfH}^\ppost(\bfeta)$ on $\RR^\nu$ of the $\RR^\nu$-valued random variable $\bfH_\ppost=(H_{\ppost,1},\ldots ,$ $H_{\ppost,\nu})$ is estimated by using the Kullback-Leibler divergence minimum principle  \cite{Kullback1951,Kapur1992,Cover2006,Soize2020a,Soize2022bb}. This estimation of $p_{\bfH}^\ppost$ is performed by using the prior pdf $\bfeta\mapsto p_{\bfH}(\bfeta)$ on $\RR^\nu$ in which $p_\bfH$ is defined by Eqs.~\eqref{eq:eq100} and
\eqref{eq:eq101}, and by using the constraint defined by  Eq.~\eqref{eq:eq72}.
The pdf $p_{\bfH}^\ppost$ on $\RR^\nu$, which satisfies the constraint defined by Eq.~\eqref{eq:eq72} and which is closest to $p_\bfH$ defined by Eq.~\eqref{eq:eq100}, is thus the solution of the following optimization problem,
\begin{equation}\label{eq:eq105}
p_{\bfH}^\ppost = \arg\,\min_{p\in\curC_{\pad,p}} \int_{\RR^\nu} p(\bfeta)\, \log\left (\frac{p(\bfeta)}{p_\bfH(\bfeta)}\right) \, d\bfeta \, ,
\end{equation}
in which the admissible set $\curC_{\ad,p}$ is defined by
\begin{equation}\label{eq:eq106}
\curC_{\ad,p} = \left\{\bfeta\mapsto p(\bfeta):\RR^\nu\rightarrow \RR^+ \, ,\int_{\RR^\nu} p(\bfeta)\, d\bfeta =1\, ,
\int_{\RR^\nu} \!\bfh^c(\bfeta)\, p(\bfeta)\, d\bfeta = \bfb^c \right\} \, .
\end{equation}
%
% Since we are interested in divergence of  $p(\bfeta)$ from $p_\bfH(\bfeta)$, which satisfies the constraints expressed in $\curC_{\ad,p}$,
% the symmetry condition is irrelevant in this case.
%
%
\subsection{Methodology for solving the optimization problem}
\label{sec:Section6.3}
The constraints defined in  admissible set $\curC_{\ad,p}$ are taken into account by introducing two Lagrange multipliers, $\lambda_0-1$ with $\lambda_0\in\RR^+$ associated with the normalization condition
and $\bflambda\in\curC_{\ad,\bflambda}\subset\RRNr$ associated with the functional constraint. The admissible set
$\curC_{\ad,\bflambda}$ of $\bflambda$ is, \textit{a priori}, a subset of $\RRNr$, which will be defined in Section~\ref{sec:Section6.4}
(in fact, we will see that $\curC_{\ad,\bflambda}=\RR^{N_r}$).
The Lagrange multiplier  $\lambda_0$ is eliminated as a function of $\bflambda$. In Eq.~\eqref{eq:eq105}, the posterior pdf $p_{\bfH}^\ppost$  is constructed as the limit of a sequence $\{p_{\bfH_\bflambda}\}_\bflambda$ of probability density functions of a sequence $\{\bfH_\bflambda\}_\bflambda$ of  $\RR^\nu$-valued random variables $\bfH_\bflambda = (H_{\bflambda, 1}, \ldots , H_{\bflambda, \nu})$ that depend on $\bflambda$.
For $\bflambda$ fixed in $\curC_{\ad,\bflambda}$, a MCMC algorithm is used for generating the constrained learned set $\curD_{\bfH_\bflambda} = \{\bfeta^1_\bflambda,\ldots \bfeta^N_\bflambda\}$ constituted of $N\gg N_d$ independent realizations $\{\bfeta_\bflambda^\ell, \ell=1,\ldots , N\}$ of $\bfH_\bflambda$.
When the convergence is reached with respect to $\bflambda$, the constrained learned set $\curD_{\bfH_\ppost} = \{\bfeta_\ppost^1,\ldots, \bfeta_\ppost^N\}$ is generated. This set is made up of $N$ independent realizations $\{\bfeta_\ppost^\ell, \ell=1,\ldots , N\}$ of $\bfH_\ppost$ whose probability measure is $p_{\bfH}^\ppost(\bfeta)\, d\bfeta$. The MCMC generator will be a nonlinear It\^o stochastic differential equation (ISDE) associated with the nonlinear stochastic dissipative Hamiltonian dynamical system proposed in \cite{Soize2008b} and based on \cite{Soize1994}. This MCMC generator allows for removing the transient part to rapidly reach the stationary response associated with the invariant measure for which measure $p_{\bfH}^\ppost(\bfeta)\, d\bfeta$ is the marginal measure. The ISDE is solved by using the St\"ormer-Verlet algorithm, which yields an efficient and accurate MCMC algorithm. This algorithm can then easily be parallelized for strongly decreasing the elapsed time on a multicore computer (See Algorithm~\ref{algorithm:1} in Section~\ref{sec:Section6.7}-(ii)). Note that this MCMC generator can be considered to belong to the class of Hamiltonian Monte Carlo methods \cite{Neal2011,Girolami2011} but is not similar due to the dissipative term, and is a MCMC algorithm \cite{Kaipio2005,Robert2005,Spall2005}.

Let us assumed that the optimization problem defined by Eq.~\eqref{eq:eq105} has one solution $p_{\bfH}^\ppost$ and that $p=p_{\bfH}^\ppost$ is a regular point of the continuously differentiable functional $p\mapsto \int_{\RR^\nu} \bfh^c(\bfeta)\, p(\bfeta)\, d\bfeta - \bfb^c$. For $\lambda_0\in\RR^+$ and $\bflambda\in\curC_{\ad,\bflambda}$, we define the Lagrangian,
\begin{equation}\nonumber
\hbox{\rm{Lag}}(p,\lambda_0,\bflambda) = \!\int_{\RR^\nu} p(\bfeta)\, \log\left (\frac{p(\bfeta)}{p_\bfH(\bfeta)}\right)  d\bfeta
                                 + (\lambda_0 - 1)\,(\! \int_{\RR^\nu}\! p(\bfeta)\, d\bfeta -1)
                                 + \langle \bflambda \, ,\! \int_{\RR^\nu} \!\bfh^c(\bfeta)\, p(\bfeta)\, d\bfeta - \bfb^c \rangle \, .
\end{equation}
We define the sequence $\{p_{\bfH_\bflambda}\}_\bflambda$ of pdf $\bfeta\mapsto p_{\bfH_\bflambda}(\bfeta\,;\bflambda)$ on $\RR^\nu$, indexed by $\bflambda$, such that $p_{\bfH_\bflambda}(.\, ; \bflambda)$ is an extremum of functional $p\mapsto \hbox{\rm{Lag}}(p,\lambda_0,\bflambda)$. Using the calculus of variations yields
\begin{equation}\label{eq:eq108}
 p_{\bfH_\bflambda}(\bfeta\,;\bflambda) = c_0(\bflambda)\, \zeta(\bfeta)\, \exp\{-\langle\bflambda\, ,\bfh^c(\bfeta) \rangle\} \quad , \quad \forall\, \bfeta\in \RR^\nu\, ,
\end{equation}
in which $c_0(\bflambda)$ is the constant of normalization that depends on $\bflambda$ (note that $\lambda_0$ is eliminated and we have $c_0(\bflambda) = c_\nu\, \exp\{-\lambda_0\}$).
Since Lemma~\ref{lemma:7} holds for any probability measure $P_\bfH$ on $\RR^\nu$ with support $\RR^\nu$, we can conclude that the $N_r$ constraints defined by the components of Eq.~\eqref{eq:eq72} are algebraically independent.
Consequently, there exists (see \cite{Luenberger2009})
$\bflambda^\psol$ in $\curC_{ad,\bflambda}$ such that the functional $(p,\lambda_0,\bflambda)\mapsto \hbox{\rm{Lag}}(p,\lambda_0,\bflambda)$
is stationary at point $p=p_{\bfH}^\ppost$ for $\bflambda=\bflambda^\psol$ and $\lambda_0 = -\log(c_0(\bflambda^\psol)/c_\nu)$.
Consequently, $p_{\bfH}^\ppost = p_{\bfH_{\!\bflambda^{\,\ppsol}}}(.\, ;\bflambda^\psol)$ and Eq.~\eqref{eq:eq108} yield
\begin{equation}\label{eq:eq110}
 p_{\bfH}^\ppost(\bfeta) = c_0(\bflambda^\psol)\, \zeta(\bfeta)\, \exp\{-\langle\bflambda^\psol ,\bfh^c(\bfeta)\rangle\} \quad , \quad \forall\, \bfeta\in \RR^\nu\, .
\end{equation}
Therefore, $p_{\bfH}^\ppost$ is the unique solution of the optimization problem defined by Eq.~\eqref{eq:eq105}, in which $\bflambda^\psol$ will be the unique solution of a convex optimization problem that will be defined by Theorem~\ref{theorem:3} in Section~\ref{sec:Section6.4}) and which will be the solution
of the following nonlinear algebraic equation in $\bflambda$,
$\int_{\RR^\nu} \bfh^c(\bfeta)\, p_{\bfH_\bflambda}(\bfeta\,;\bflambda) \, d\bfeta = \bfb^c$.
\subsection{Analysis of the optimization problem}
\label{sec:Section6.4}
In this section, we study the admissible set of the Lagrange multiplier, we analyze the integrability properties of the probability density function $p_{\bfH_\bflambda}$ of $\bfH_\bflambda$, and we give an explicit construction of $p_{\bfH_\bflambda}$.
%
%---LEMMA 8 ----------------------
\begin{lemma}[Admissible set $\curC_{\ad,\bflambda}$ of Lagrange's multiplier and integrability properties] \label{lemma:8}
Let $N_r$ be fixed. Let $\bfh^c(\bfeta)$ $ = (h_1^c(\bfeta),\ldots ,$ $h_{N_r}^c(\bfeta))$ be the function on $\RR^\nu$ with values in $\RR^{N_r}$, defined by Eq.~\eqref{eq:eq73}, and let $p_\bfH$ be the prior probability density function on $\RR^\nu$ of $\bfH$, defined by Eq.~\eqref{eq:eq100}.

\noindent (a) The admissible set $\curC_{\ad,\bflambda}$ of the Lagrange multiplier $\bflambda\,$,
which is defined by
\begin{equation}\label{eq:eq111}
\curC_{\ad,\bflambda} = \bigg \{
                         \bflambda\in\RRNr \,\, \vert \,\, 0 < E \{\, \exp \big(-\langle\,\bflambda\, ,\bfh^c(\bfH)\,\rangle \,\big) \,\}
                         \,\, < +\infty \,\bigg \},
\end{equation}
is such that $\curC_{\ad,\bflambda} = \RRNr$.

\noindent (b) For all $\bflambda$ in $\curC_{\ad,\bflambda}$, let $\bfeta\mapsto\curV_{\!\bflambda}(\bfeta)$ be the $\RR$-valued function on $\RR^\nu$ such that
\begin{equation}\label{eq:eq112}
\curV_{\!\bflambda}(\bfeta) = \phi(\bfeta) + \langle\bflambda\, , \bfh^c(\bfeta)\rangle \, ,
\end{equation}
in which $\phi(\bfeta) = -\log\bfzeta(\bfeta)$ (see Eq.~\eqref{eq:eq102}). We then have,
\begin{equation}\label{eq:eq113}
0 < \int_{\RR^\nu} \exp\{-\curV_{\!\bflambda}(\bfeta)\}\, d\bfeta \,\, < \, +\infty \, .
\end{equation}
\noindent (c) The pdf $\bfeta\mapsto p_{\bfH_\bflambda}(\bfeta\,;\bflambda)$, defined by Eq.~\eqref{eq:eq108}, which can be written as
\begin{equation}\label{eq:eq114}
p_{\bfH_\bflambda}(\bfeta\,;\bflambda) = c_0(\bflambda)\, \exp\{-\curV_{\!\bflambda}(\bfeta)\} \quad , \quad \forall\bfeta\in\RR^\nu \, ,
\end{equation}
is such that the constant $c_0(\bflambda)$ of normalization verifies
\begin{equation}\label{eq:eq115}
0 < c_0(\bflambda) < +\infty \quad , \quad \forall\bflambda \in \curC_{\ad,\bflambda}\, .
\end{equation}
\noindent (d) For all $\bflambda$ in $\curC_{\ad,\bflambda}$, we have
$\curV_{\!\bflambda}(\bfeta) \rightarrow +\infty$ if $\Vert\,\bfeta\,\Vert \rightarrow +\infty$, and
\begin{equation}\label{eq:eq116}
\int_{\RR^\nu} \Vert\,\bfh^c(\bfeta)\,\Vert^2\, \exp\{-\curV_{\!\bflambda}(\bfeta)\}\, d\bfeta \,\, < \, +\infty      \quad , \quad
\int_{\RR^\nu} \Vert\,[\nabla_{\!\bfeta}\bfh^c(\bfeta) ]\,\Vert_F\, \exp\{-\curV_{\!\bflambda}(\bfeta)\}\, d\bfeta \,\, < \, +\infty\, .
\end{equation}
\end{lemma}
%
%------ PROOF LEMMA 8------------------
\begin{proof} (Lemma~\ref{lemma:8}).

\noindent (a) For all $\bfeta\in\RR^\nu$ and for all $r\in\{1,\ldots ,N_r\}$ , Eq.~\eqref{eq:eq73} shows that
$0 < h_r^c(\bfeta) \leq 1$. It can then be deduced that, for all $\bflambda\in\RRNr$, we have
$0 < E\{\, \exp \big ( -\langle \, \bflambda\, ,\bfh^c(\bfH)\, \rangle \,\big ) \,\}\,\, < +\infty$, which proves that $\curC_{\ad,\bflambda} = \RRNr$.

\noindent (b) Using Eqs.~\eqref{eq:eq100}, \eqref{eq:eq102}, and \eqref{eq:eq112}, yields
$\int_{\RR^\nu} \exp\{-\curV_{\!\bflambda}(\bfeta)\}\, d\bfeta  =c_\nu^{-1} \int_{\RR^\nu} \exp\{-\langle\bflambda\, , \bfh^c(\bfeta)\rangle\}\,
p_\bfH(\bfeta)\, d\bfeta  = c_\nu^{-1} E\{\,\exp\{-\langle\bflambda\, ,$ $\bfh^c(\bfH)\rangle\} \,\}$, which is positive and finite due to Eq.~\eqref{eq:eq111} and to $0 < c_\nu < +\infty$. We have thus proven Eq.~\eqref{eq:eq113}.

\noindent (c) Using Eqs.~\eqref{eq:eq113} and \eqref{eq:eq114}, and since we need to have $\int_{\RR^\nu} p_{\bfH_\bflambda}(\bfeta)\, d\bfeta = 1$, we deduce Eq.~\eqref{eq:eq115}.

\noindent (d)  As $\bfh^c$ is continuous on $\RR^\nu$, (see Eq.~\eqref{eq:eq73}), $\forall\bflambda\in\curC_{\ad,\bflambda}$, $\bfeta\mapsto \exp\{-\curV_{\!\bflambda}(\bfeta)\}$ is continuous on $\RR^\nu$ and then is locally integrable on $\RR^\nu$. Eq.~\eqref{eq:eq113} implies the integrability at infinity of $\bfeta\mapsto \exp\{-\curV_{\!\bflambda}(\bfeta)\}$. Since $\bfeta\mapsto \curV_{\!\bflambda}(\bfeta)$ is continuous on $\RR^\nu$, it can  be deduced that $\curV_{\!\bflambda}(\bfeta) \rightarrow +\infty$ if $\Vert\,\bfeta\,\Vert \rightarrow +\infty$.
Using Eq.~\eqref{eq:eq73}, Eq.~\eqref{eq:eq116} can easily be proven.
\end{proof}
%
%
%---- THEOREM 3 ----------------------
\begin{theorem}[Construction of the probability measure of $\bfH_\bflambda$] \label{theorem:3}
For all $\bflambda$ in $\curC_{\ad,\bflambda}$, let
\begin{equation} \label{eq:eq117}
p_{\bfH_\bflambda}(\bfeta\,;\bflambda)= c_0(\bflambda)\,\zeta(\bflambda)
                         \exp \big (- \langle \, \bflambda\, , \bfh^c(\bfeta)\, \rangle\,\big )
\end{equation}
be the pdf of $\bfH_\bflambda$ (see Eq.~\eqref{eq:eq108}) with
$c_0(\bflambda)$ satisfying Eq.~\eqref{eq:eq115}).

\noindent (a) The $\RRNr$-valued random variable $\bfh^c(\bfH_\bflambda)$ is a second-order random variable,
\begin{equation}\label{eq:eq118}
E\{\Vert\,\bfh^c(\bfH_\bflambda)\,\Vert^2\} < +\infty \, .
\end{equation}
\noindent (b) Let $\bflambda\mapsto \Gamma(\bflambda)$ be the real-valued function defined on $\curC_{\ad,\bflambda}$ such that
\begin{equation}\label{eq:eq119}
\Gamma(\bflambda) = \langle\,\bflambda\, , \bfb^c\,\rangle -\log c_0(\bflambda)\, ,
\end{equation}
in which $\bfb^c$ is given in $\RRNr$. For all $\bflambda$ in $\curC_{\ad,\bflambda}$, we have
\begin{equation}\label{eq:eq120}
\nabla_{\!\bflambda}\Gamma(\bflambda) = \bfb^c - E\{\bfh^c(\bfH_\bflambda) \} \in \RRNr\, ,
\end{equation}
\begin{equation}\label{eq:eq121}
[\Gamma{\,''}(\bflambda)] = [\cov\{\bfh^c(\bfH_\bflambda)\}] \in \MM_{N_r}^+ \, ,
\end{equation}
where the positive-definite covariance matrix $[\Gamma{\,''}(\bflambda)]$ of $\bfh^c(\bfH_\bflambda)$ is such that
$[\Gamma{\,''}(\bflambda)]_{kk'} = \partial^2\Gamma(\bflambda)/\partial\lambda_k\partial\lambda_{k'}$.

\noindent (c) $\Gamma$ is a strictly convex function on $\curC_{\ad,\bflambda}$. There is a unique solution $\bflambda^\psol$ in $\curC_{\ad,\bflambda}$ of the convex optimization problem,
\begin{equation}\label{eq:eq122}
\bflambda^\psol = \arg\, \min_{\bflambda\in\curC_{\ad,\bflambda}} \Gamma(\bflambda) \, ,
\end{equation}
which is the unique solution in $\bflambda$ of the following equation,
\begin{equation}\label{eq:eq123}
\nabla_{\!\bflambda} \Gamma(\bflambda) = \bfzero_{N_r} \, .
\end{equation}
The pdf $p_{\bfH}^\ppost$ of $\bfH^\ppost$, which satisfies the constraint $E\{\bfh^c(\bfH^\ppost)\} = \bfb^c$,  is written (see Eq.~\eqref{eq:eq117}) as
\begin{equation}\label{eq:eq124}
p_{\bfH}^\ppost(\bfeta) = p_{\bfH_{\!\bflambda^\ppsol}}(\bfeta\,;\bflambda^\psol) \quad, \quad \forall\bfeta\in\RR^\nu\, .
\end{equation}
\end{theorem}
%
%------ PROOF OF THEOREM 3 ------------------
\begin{proof} (Theorem~\ref{theorem:3}).

\noindent (a) Using Eq.~\eqref{eq:eq114}, Eq.~\eqref{eq:eq115}, and the first equation Eq.~\eqref{eq:eq116} yield
\begin{equation} \nonumber
E\{\Vert\,\bfh^c(\bfH_\bflambda)\,\Vert^2\} =\int_{\RR^\nu} \Vert\,\bfh^c(\bfeta)\,\Vert^2\, c_0(\bflambda)\, \exp\{-\curV_{\!\bflambda}(\bfeta)\} \, d\bfeta < +\infty\, .
\end{equation}

\noindent (b) Let us prove Eqs.~\eqref{eq:eq120} and \eqref{eq:eq121} using a similar proof to the one introduced in the discrete case for finding the maximum entropy probability measure \cite{Agmon1979,Kapur1992}.
Eq.~\eqref{eq:eq112} yields $\nabla_{\!\bflambda}\curV_{\!\bflambda}(\bfeta) = \bfh^c(\bfeta)$ and from Eq.~\eqref{eq:eq114}, it can be deduced that
\begin{equation}\label{eq:eq125}
\nabla_{\!\bflambda} p_{\bfH_\bflambda}(\bfeta\,;\bflambda) = \left (c_0(\bflambda)^{-1}\,\nabla_{\!\bflambda} c_0(\bflambda) - \bfh^c(\bfeta)\right)\,p_{\bfH_{\bflambda}}(\bfeta\,;\bflambda) \, .
\end{equation}
By integrating the two members of Eq.~\eqref{eq:eq125} with respect to $\bfeta$ on $\RR^\nu$, we obtain
\begin{equation}\label{eq:eq126}
c_0(\bflambda)^{-1}\,\nabla_{\!\bflambda} c_0(\bflambda) = \int_{\RR^\nu} \bfh^c(\bfeta)\, p_{\bfH_{\bflambda}}(\bfeta\,;\bflambda)\, d\bfeta = E\{\bfh^c(\bfH_\bflambda)\} \, .
\end{equation}
Eq.~\eqref{eq:eq119} yields $\nabla_{\!\bflambda}\Gamma(\bflambda) = \bfb^c - c_0(\bflambda)^{-1}\,\nabla_{\!\bflambda} c_0(\bflambda)$, which proves Eq.~\eqref{eq:eq120} by using Eq.~\eqref{eq:eq126}. Note that Eq.~\eqref{eq:eq118} implies the existence of the mean value $E\{\bfh^c(\bfH_\bflambda)\}$. Taking the derivative of Eq.~\eqref{eq:eq120} with respect to $\bflambda$ yields
\begin{equation}\label{eq:eq127}
[\Gamma{\,''}(\bflambda)] = - \int_{\RR^\nu} \bfh^c(\bfeta) \otimes \nabla_{\!\bflambda} p_{\bfH_\bflambda}(\bfeta\,;\bflambda)
\, d\bfeta\, .
\end{equation}
Substituting Eq.~\eqref{eq:eq126} into Eq.~\eqref{eq:eq125} yields
$\nabla_{\!\bflambda} p_{\bfH_\bflambda}(\bfeta\,;\bflambda) = ( E\{\bfh^c(\bfH_\bflambda)\} -\bfh^c(\bfeta)\,) \, p_{\bfH_\bflambda}(\bfeta\,;\bflambda)$,
which with Eq.~\eqref{eq:eq127}, gives
$[\Gamma{\,''}(\bflambda)] = E\{ \bfh^c(\bfH_\bflambda) \otimes \bfh^c(\bfH_\bflambda)\} - (E\{\bfh^c(\bfH_\bflambda)\})\otimes (E\{\bfh^c(\bfH_\bflambda)\})$
that is the covariance matrix of the $\RRNr$-valued random variable $\bfh^c(\bfH_\bflambda)$. Again Eq.~\eqref{eq:eq118} proves the existence of matrix $[\Gamma{\,''}(\bflambda)]$ as a covariance matrix, which is semi-positive definite.
Using Lemma~\ref{lemma:7}, this matrix is positive definite.

\noindent (c) Since $[\Gamma{\,''}(\bflambda)]$ is a positive-definite matrix for all $\bflambda$ in $\curC_{\ad,\bflambda}$, it can then be deduced that $\bflambda\mapsto \Gamma(\bflambda)$ is strictly convex on $\curC_{\ad,\bflambda}$. Therefore, Eq.~\eqref{eq:eq122} holds, $\bflambda^\psol$ is unique, and Eq.~\eqref{eq:eq120} shows that $E\{\bfh^c(\bfH_{\bflambda^\psol})\}=\bfb^c$. Taking into account Eq.~\eqref{eq:eq110}, the solution is given by Eq.~\eqref{eq:eq124} and is unique due to the uniqueness  of solution  $\bflambda^\psol$ of $\nabla_{\!\bflambda} \Gamma(\bflambda) = \bfzero_{N_r}$.
\end{proof}
\subsection{Dissipative stochastic Hamiltonian system as a MCMC generator of $\bfH_\bflambda$}
\label{sec:Section6.5}
For the reasons given in Section~\ref{sec:Section6.3}, the chosen MCMC generator is based on a nonlinear It\^o stochastic differential equation (ISDE) associated with the nonlinear stochastic dissipative Hamiltonian dynamical system proposed in \cite{Soize2008b} and based on \cite{Soize1994}.

Let $\{\bfW^\wien(t) = (W_1^\wien(t),\ldots ,W_\nu^\wien(t)), t\geq 0\}$ be the Wiener process, defined on $(\Theta,\curT,\curP)$, indexed by $\RR^+$, with values in $\RR^\nu$, such that $W_1^\wien,\ldots ,W_\nu^\wien$ are mutually independent, $\bfW^\wien(0) =\bfzero_\nu$ a.s., $\bfW^\wien$ is a process with independent increments such that, for all $0\leq t'  < t < +\infty$, the increment $\bfW^\wien(t)-\bfW^\wien(t')$ is a $\RR^\nu$-valued second-order, Gaussian, centered random variable whose covariance matrix is $(t-t')\,[I_\nu]$.
%
%---- THEOREM 4 ----------------------
\begin{theorem}[MCMC generator of $\bfH_\bflambda$] \label{theorem:4}
Let $\bfh^c=(h_1^c,\ldots , h_{N_r}^c)$ be the function whose component $h_r^c$ is defined by Eq.~\eqref{eq:eq73}. Let $\bflambda$ be fixed in $\curC_{\ad,\bflambda}$. Consequently,
 Eq.~\eqref{eq:eq116} of Lemma~\ref{lemma:8} holds. Let $\{ ( \bfU_{\!\bflambda}(t),\bfV_{\!\bflambda}(t) ), t\geq 0\}$ be the stochastic process, defined on $(\Theta,\curT,\curP)$, indexed by $\RR^+$, with values in $\RR^\nu\times\RR^\nu$, which verifies the following ISDE for $t > 0$, with the initial condition $(\bfu_0,\bfv_0)$ given in $\RR^\nu\times\RR^\nu$,
\begin{align}
d\bfU_{\!\bflambda}(t) & = \bfV_{\!\bflambda}(t)\, dt \, , \label{eq:eq128} \\
d\bfV_{\!\bflambda}(t) & = \bfL_{\bflambda}(\bfU_\bflambda(t))\, dt - \frac{1}{2} f_0\,\bfV_{\!\bflambda}(t) \, dt
                        + \sqrt{f_0}\, d\bfW^\wien(t) \, ,\label{eq:eq129} \\
 \bfU_{\!\bflambda}(0) & = \bfu_0 \,\, , \,\, \bfV_{\!\bflambda}(0)= \bfv_0 \,\,  a.s. \label{eq:eq130}
\end{align}
\noindent (a) The initial condition $\bfu_0\in\RR^\nu$ is chosen from the points of the training set $\curD_d= \{\bfeta_d^1,\ldots , \bfeta_d^{N_d}\}$ (see Section~\ref{sec:Section6.7}-(i)) while the initial condition $\bfv_0$ is chosen as any realization of a normalized Gaussian $\RR^\nu$-valued random variable $\bfV_G$, independent of $\bfW^\wien$, whose probability density function with respect to $d\bfv$ is
$p_{\bfV_G}(\bfv) = (2\pi)^{-\nu/2}\, \exp\{-\Vert\,\pbfv\,\Vert^2 /2\}$.

\noindent (b) The parameter $f_0 > 0$ allows the dissipation term in the dissipative Hamiltonian system to be controlled and to rapidly reach the stationary response associated with the invariant measure (the value $f_0=4$ is generally a good choice).

\noindent (c) For all $\bfu =(u_1,\ldots ,u_\nu)$ in $\RR^\nu$, the vector $\bfL_\bflambda(\bfu)$ in $\RR^\nu$ is defined by
$\bfL_\bflambda(\bfu) = -\nabla_{\!\bfu}\curV_{\!\bflambda}(\bfu)$, which can be written as
\begin{equation}\label{eq:eq131}
 \bfL_\bflambda(\bfu) = \frac{1}{\zeta(\bfu)}\,\nabla_{\!\bfu}\zeta(\bfu) - [\nabla_{\!\bfu} \bfh^c(\bfu)]\,\bflambda\, .
\end{equation}
\noindent (d) The stochastic solution $\{ ( \bfU_{\!\bflambda}(t),\bfV_{\!\bflambda}(t) ), t\geq 0\}$ of the ISDE defined by
Eqs.~\eqref{eq:eq128} to \eqref{eq:eq130} is unique,  has almost-surely continuous trajectories, and is a second-order diffusion stochastic process. For $t\rightarrow +\infty$, this diffusion process converges to a stationary second-order diffusion stochastic process
$\{ ( \bfU^\st_{\!\bflambda}(\tau),\bfV^\st_{\!\bflambda}(\tau) ), \tau\geq 0\}$  associated with the unique invariant probability measure on $\RR^\nu\times\RR^\nu$,
\begin{equation}\label{eq:eq132}
 p_{\bfH_\bflambda,\bfV_G}(\bfeta,\bfv\, ; \bflambda)\, d\bfeta \otimes d\bfv =
 (p_{\bfH_\bflambda}(\bfeta \, ; \bflambda)\, d\bfeta) \otimes (p_{\bfV_G}(\bfv)\, d\bfv) \, ,
\end{equation}
in which $p_{\bfH_\bflambda}(\bfeta \, ; \bflambda)$ is the pdf defined by Eq.~\eqref{eq:eq114}.

\noindent (e)  For $t_s$ sufficiently large,  $\bfH_\bflambda$ is chosen as $\bfU_{\!\bflambda}(t_s)$. The generation of the constrained learned set $\curD_{\bfH_\bflambda} = \{\bfeta_\bflambda^1,\ldots , \bfeta_\bflambda^N\}$, made up of $N\gg N_d$ independent realizations of $\bfH_\bflambda$ whose probability density function is $p_{\bfH_\bflambda}(\bfeta \, ; \bflambda)$, consists in solving Eqs.~\eqref{eq:eq128} to \eqref{eq:eq130} for $t\in[0\, ,t_s]$ and then using the realizations of $\bfU_{\!\bflambda}(t_s)$ (see the numerical aspects in Section~\ref{sec:Section6.7}).
\end{theorem}
%
%------ PROOF OF THEOREM 4 ------------------
\begin{proof} (Theorem~\ref{theorem:4}).
For $r\in\{1,\ldots ,N_r\}$, function $\bfeta\mapsto h^c_r(\bfeta)$ defined be Eq.~\eqref{eq:eq73} is twice continuously differentiable.
Since $\phi(\bfu)= -\log\zeta(\bfu)$ with $\zeta(\bfu)$ given by Eq.~\eqref{eq:eq101}, it can be deduced that function $\bfu\mapsto \curV_{\!\bflambda}(\bfu)$ defined by Eq.~\eqref{eq:eq112} is also twice continuously differentiable. Consequently,
$\bfu\mapsto \Vert\,\nabla_{\!\bfu}\curV_{\!\bflambda}(\bfu)\,\Vert$ is locally bounded on $\RR^\nu$. Using Eqs.~\eqref{eq:eq112} and \eqref{eq:eq113}, it can be seen that, for all $\bflambda\in\curC_{\ad,\bflambda}$,  $\inf_{\Vert\,\bfu\,\Vert > R}  \curV_{\!\bflambda}(\bfu) \rightarrow +\infty$ if $R\rightarrow +\infty$, and $\inf_{\bfu\in\RR^\nu} \curV_{\!\bflambda}(\bfu)$ is a finite real number. Using Eqs.~\eqref{eq:eq101}, \eqref{eq:eq102}, and \eqref{eq:eq112} yields
\begin{equation}\label{eq:eq133}
 \int_{\RR^\nu} \Vert\,\nabla_{\!\bfu}\curV_{\!\bflambda}(\bfu)\,\Vert \, p_{\bfH_\bflambda}(\bfu \, ;\bflambda)\, d\bfu \leq
 \int_{\RR^\nu} \frac{1}{\zeta(\bfu)}\Vert\,\nabla_{\!\bfu}\zeta(\bfu)\,\Vert \, p_{\bfH_\bflambda}(\bfu \, ; \bflambda)\, d\bfu
 + \int_{\RR^\nu} \Vert\, [\nabla_{\!\bfu}\bfh^c(\bfu)]\,\Vert_F \, \Vert\, \bflambda\, \Vert\, p_{\bfH_\bflambda}(\bfu \, ;\bflambda)\, d\bfu \, ,
\end{equation}
because $\Vert\, [\nabla_{\!\bfu}\bfh^c(\bfu)]\,\bflambda\, \Vert
\,\,\leq \Vert\, [\nabla_{\!\bfu}\bfh^c(\bfu)]\,\Vert \, \Vert\, \bflambda\, \Vert$
and  $\Vert\, [\nabla_{\!\bfu}\bfh^c(\bfu)]\,\Vert  \,\,\leq \Vert\, [\nabla_{\!\bfu}\bfh^c(\bfu)]\,\Vert_F$.
From Eqs.~\eqref{eq:eq117} and \eqref{eq:eq101}, the first term in the right-hand side member of Eq.~\eqref{eq:eq133} is finite, while from the second equation~\eqref{eq:eq116}, the second term is also finite. It can then be deduced that the left-hand side member of Eq.~\eqref{eq:eq133} is finite.
Consequently, Theorems 6, 7, and 9 in Pages 214 to 216 of \cite{Soize1994}, and the expression of the invariant measure given by Theorem~4 in Page 211 of the same reference, for which the Hamiltonian is $\curH(\bfu,\bfv)=\Vert\,\bfv\,\Vert^2/2 + \curV_{\!\bflambda}(\bfu)$, prove that the solution of  Eqs.~\eqref{eq:eq128} to \eqref{eq:eq130} is unique and is a second-order diffusion stochastic process with almost-surely continuous trajectories, which converges for $t\rightarrow +\infty$ to a second-order stationary diffusion process with almost surely continuous trajectories $\{ ( \bfU^\st_{\!\bflambda}(\tau),\bfV^\st_{\!\bflambda}(\tau) ), \tau\geq 0\}$  associated with the invariant probability measure given by Eq.~\eqref{eq:eq132}. For any $\tau > 0$,
$\bfU_{\!\bflambda}^\st(\tau) = \lim_{t\rightarrow +\infty}\bfU_{\!\bflambda}(t+\tau)$ in probability measure.
\end{proof}
\subsection{Iterative algorithm for calculating $\bflambda^\psol$}
\label{sec:Section6.5}
Let us consider Theorem~\ref{theorem:3}. For $\bflambda$ fixed in $\curC_{\ad,\bflambda}$, the value of $\Gamma(\bflambda)$ cannot be evaluated in high dimension using Eq.~\eqref{eq:eq119} due to the presence of the normalization constant $c_0(\bflambda)$. Consequently, $\bflambda^\psol$ cannot directly be estimated using, for instance, the gradient descent algorithm applied to the convex optimization problem defined by Eq.~\eqref{eq:eq122}.
We will then calculate $\bflambda^\psol$ as the unique solution in $\bflambda$ of equation $\nabla_{\!\bflambda}\Gamma(\bflambda)= \bfzero_{N_r}$
(see Eq.~\eqref{eq:eq123}), that is to say (see Eq.~\eqref{eq:eq120}), solving the following equation in $\bflambda$ on $\RRNr$,
\begin{equation}\label{eq:eq134}
E\{\bfh^c(\bfH_\bflambda)\} - \bfb^c = \bfzero_{N_r}\, .
\end{equation}
This equation is solved by using the Newton iterative method \cite{Kelley2003} applied to function
$\bflambda\mapsto \nabla_{\!\bflambda}\Gamma(\bflambda)$ as proposed in \cite{Batou2013,Soize2017b}, that is to say,
\begin{equation}\label{eq:eq135}
\bflambda^{\,i+1} = \bflambda^{\,i} - \alpha_{\prelax}(i)\, [\Gamma{\,''}(\bflambda^{\,i})]^{-1} \, \nabla_{\!\bflambda}\Gamma(\bflambda^{\,i}) \quad , \quad i=0,1,\ldots ,i_\pmax \, ,
\end{equation}
with $\bflambda^{\,0} = \bfzero_{N_r}$, in which $\alpha_{\prelax}(i)\in ]0\, , 1]$ is a relaxation factor, where $\nabla_{\!\bflambda}\Gamma(\bflambda)$ and $[\Gamma{\,''}(\bflambda)]$ are defined by
Eqs.~\eqref{eq:eq120} and \eqref{eq:eq121}, and where $i_\pmax$ is a given integer sufficiently large.
An estimation of $\bflambda^\psol$ is chosen as
\begin{equation}\label{eq:eq136}
\bflambda^\psol = \bflambda^{i_\psol} \quad , \quad  i_\psol = \arg \min_{i =1,\ldots ,i_\pmax} \error(i)\, ,
\end{equation}
in which the error function $i\mapsto\error(i)$ from $\{1,\ldots ,i_\pmax\}$ into $\RR^+$ is defined by
\begin{equation}\label{eq:eq137}
\error(i) = \frac{1}{\Vert\, \bfb^{c}\, \Vert}\, \Vert\, \bfb^{c} - E\{\bfh^{c}(\bfH_{\bflambda^{\,i}})\}\Vert \,  .
\end{equation}
\subsection{Numerical implementation}
\label{sec:Section6.7}
A time-discretization scheme (see for instance \cite{Kloeden1992,Talay1990}) must be used to solve the ISDE defined by Eqs.~\eqref{eq:eq128} to \eqref{eq:eq130} for $t\in[0\, , t_s]$ with the initial condition at $t=0$ defined in Theorem~\ref{theorem:4}, in order to generate the constrained learned set $\curD_{\bfH_{\!\bflambda}} = \{\bfeta_{\bflambda}^1,\ldots , \bfeta_{\bflambda}^N \}$ with $N\gg N_d$.
It is assumed that $N$ is written as $N=N_d\times n_\pMC$ with $n_\pMC\gg 1$. The case of Hamiltonian dynamical systems has been analyzed in \cite{Talay2002} by using an implicit Euler scheme. Presently we propose to use the St\"ormer-Verlet scheme (see \cite{Hairer2003} for the deterministic case and \cite{Burrage2007} for the stochastic case), which is an efficient scheme that allows for having a long-time energy conservation for non-dissipative Hamiltonian dynamical systems. In \cite{Soize2012b}, we have proposed to use an extension of the St\"ormer-Verlet scheme for stochastic dissipative Hamiltonian systems, that we have also used in \cite{Guilleminot2013a, Soize2015,Soize2016,Soize2020a,Soize2021a,Soize2022a,Soize2022bb}.\\

\noindent (i) \textit{St\"ormer-Verlet scheme and computation of the constrained learned set $\curD_{\bfH_{\bflambda^{\,i}}}$}.
Let $i$ be the index of the sequence $\{\bflambda^{\, i},i=0,1,\ldots, i_\pmax\}$ of the Lagrange multipliers that are computed using Eq.~\eqref{eq:eq135} with $\bflambda^{\,0} = \bfzero_{N_r}$. For $m=0,1,\ldots , M_s$ (with $M_s > 1$ an integer), let $t_m = m\, \Delta t$  be the time sampling, which is such that $t_{M_s} = t_s$ with $t_s = M_s \, \Delta t$.
Let $\Delta\bfW_{m+1}^{\wien} = \bfW^\wien(t_{m+1}) - \bfW^\wien(t_{m})$ be the Gaussian, second-order, centered, $\RR^\nu$-valued random variable  such that $E\{\Delta\bfW_{m+1}^{\wien}\otimes \Delta\bfW_{m+1}^{\wien}\} = \Delta t \, [I_\nu]$. Let $\{\theta_\ell, \ell=1,\ldots , N\}$ be $N$ independent realizations in $\Theta$. For $m=0,1,\ldots, M_s-1$, let $\Delta\WW_{m+1}^\ell = \Delta\bfW_{m+1}^{\wien}(\theta_\ell)$ be the realization $\theta_\ell$ of $\Delta\bfW_{m+1}^{\wien}$.
Following the choice of $(\bfu_0,\bfv_0)$ defined in Theorem~\ref{theorem:4}, let $\bfu_0^1,\ldots, \bfu_0^N$ in $\RR^\nu$ such that
for $k=1,\ldots,n_\pMC$ and for $j=1,\ldots , N_d$, we take $\bfu_0^\ell = \bfeta_d^j$  with $\ell=j+(k-1)\times N_d$.
Let $\bfv_0^1,\ldots ,\bfv_0^N$ in $\RR^\nu$ be $N$ independent realizations of the $\RR^\nu$-valued random variable $\bfV_G$ also defined in Theorem~\ref{theorem:4}. For $\ell=1,\ldots , N$, the realizations $\Delta\WW_{m+1}^\ell$, $\bfu_0^\ell$, and $\bfv_0^\ell$  are independent of $\{\bflambda^{\,i}\}_i$.
For $i\in\{0,1,\ldots , i_\pmax\}$ and for $\ell\in\{1,\ldots , N\}$, we introduce the realizations
$\UU_m^{i,\ell} = \bfU_{\!\bflambda^{\, i}}(t_m\, ; \theta_\ell)$ and
$\VV_m^{i,\ell} = \bfV_{\!\bflambda^{\, i}}(t_m\, ; \theta_\ell)$.
For $m\in\{0,1,\ldots , M_s -1\}$, the St\"ormer-Verlet scheme applied to realization $\theta_\ell$ of  Eqs.~\eqref{eq:eq128} to \eqref{eq:eq130} yields the following recurrence,
\begin{align}
&\UU_{m+1/2}^{i,\ell}   =\UU_m^{i,\ell}  + \frac{\Delta t}{2}\, \VV_m^{i,\ell} \, , \label{eq:eq138}\\
&\VV_{m+1}^{i,\ell}     =\frac{1-\gamma}{1+\gamma}\, \VV_m^{i,\ell} + \frac{\Delta t}{1+\gamma}\,
                           \bfL_{\bflambda^{\, i-1}}(\UU_{m+1/2}^{i,\ell})
                            + \frac{\sqrt{f_0}}{1+\gamma}\, \Delta\WW_{m+1}^\ell   \, , \label{eq:eq139} \\
&\UU_{m+1}^{i,\ell}     = \UU_{m+1/2}^{i,\ell} + \frac{\Delta t}{2} \, \VV_{m+1}^{i,\ell}   \, , \label{eq:eq139-1}
\end{align}
with the initial condition
\begin{equation}\label{eq:eq139-2}
\UU_{0}^{i,\ell}   =  \bfu_0^\ell \quad , \quad \VV_{0}^{i,\ell}   =  \bfv_0^\ell  \, ,
\end{equation}
in which $\gamma =f_0\, \Delta t /4$ and where, using Eq.~\eqref{eq:eq131},
\begin{equation}\label{eq:eq139-3}
\bfL_{\bflambda^{\, i-1}}(\bfu) = \frac{1}{\zeta(\bfu)}\,\nabla_{\!\bfu}\zeta(\bfu) - [\nabla_{\!\bfu}\bfh^c(\bfu)]\, \bflambda^{\, i-1}
\end{equation}
Using Eq.~\eqref{eq:eq73}, for $\alpha\in \{1,\ldots , \nu \}$ and $r \in \{1,\ldots , N_r\}$, the entry  $[\nabla_{\!\bfu}\bfh^c(\bfu)]_{\alpha r}$ of matrix
$[\nabla_{\!\bfu}\bfh^c(\bfu)] \in\MM_{\nu,N_r}$ is written as
\begin{equation}\label{eq:eq139-4}
[\nabla_{\!\bfu}\bfh^c(\bfu)]_{\alpha r} = \frac{2}{\nu s^2} ( [V]^T \tilde\bfq_\pexp^r - u_\alpha) \,
                        \exp\bigg ( -\frac{1}{\nu s^2} \,\Vert \,[V]^T \tilde\bfq_\pexp^r - \bfu\, \Vert^2\bigg ) \, .
\end{equation}
The recurrence defined by  Eqs.~\eqref{eq:eq138} to \eqref{eq:eq139-4} allows $\curD_{\bfH_{\!\bflambda^{\,i}}}$ to be calculated as
\begin{equation}\label{eq:eq139-5}
\curD_{\bfH_{\!\bflambda^{\,i}}} = \{\bfeta_{\bflambda^{\,i}}^1,\ldots , \bfeta_{\bflambda^{\,i}}^N \} \quad , \quad
\bfeta_{\bflambda^{\,i}}^\ell = \bfU_{\!\bflambda^{\, i}}(t_s\, ; \theta_\ell) = \UU_{M_s}^{i,\ell} \, .
\end{equation}
\noindent (ii) \textit{Summary of the algorithm}.
The algorithm for calculating $\bflambda^\psol$ and the $\curD_{\bfH^\ppost} = \{ \bfeta_\ppost^1,\ldots,\bfeta_\ppost^N \}$ with
$\bfeta_\ppost^\ell = \bfeta_{\bflambda^\psol}^\ell$ for $\ell=1,\ldots, N$ is summarized in Algorithm~\ref{algorithm:1}.
We then obtained the $N$ independent realizations $\bfeta_\ppost^1,\ldots,\bfeta_\ppost^N $ of the posterior $\RR^\nu$-valued random variable $\bfH_\ppost$. Then, the realizations $\bfq_\ppost^1,\ldots , \bfq_\ppost^N$ of the posterior observations $\bfQ_\ppost$ and the realizations $\bfw_\ppost^1,\ldots , \bfw_\ppost^N$ of the posterior control parameter $\bfW_\ppost$ are computed using
Eq.~\eqref{eq:eq2bis}, that is to say,
\begin{equation}\label{eq:eq139-6}
\bfQ_\ppost=\underline\bfq + [\Phi_q]\, [\kappa]^{1/2} \, \bfH_\ppost \quad , \quad
\bfW_\ppost=\underline\bfw + [\Phi_w]\, [\kappa]^{1/2} \, \bfH_\ppost \, .
\end{equation}
%
%
%------ALGORITHM 1
%
\begin{algorithm}
\caption{Algorithm for calculating $\bflambda^\psol$ and $\curD_{\bfH_\pppost} = \{ \bfeta_\ppost^1,\ldots,\bfeta_\ppost^N \}$.}
\label{algorithm:1}
\begin{algorithmic}[1]
\State{\textbf{Data:}$N_d$, $\curD_d = \{ \bfeta_d^1,\ldots,\bfeta_d^{N_d} \}$, $N$, $i_\pmax$, $M_s$, $t_s$, $\Delta t$, $f_0$, $\gamma=f_0\,\Delta t / 4$}
\State{\textbf{Init:} $\,\,\Delta\WW_{m+1}^{\ell}, \ell \in\{ 1,\ldots , N \}, m\in \{1,\ldots , M_s-1\}$,
                  $\,\,\bfu_0^\ell$ and $\bfv_0^\ell$ for $\ell\in\{1,\ldots , N\}$, $\bflambda^{\,0} = \bfzero_{N_r}$ }
\For{$i=1:i_\pmax$}
    \For{$\ell=1:N\, (loop \, in \,  parallel \, computation)$}
    \State{$\curD_{\bfH_{\!\bflambda^{\,i}}} = \{\bfeta_{\bflambda^{\,i}}^1,\ldots , \bfeta_{\bflambda^{\,i}}^N \}$ from Eq.~\eqref{eq:eq139-5},
    using Eqs.~\eqref{eq:eq138} to \eqref{eq:eq139-4} and $\curD_{\bfH_{\!\bflambda^{\, i-1}}}$ ($\curD_{\bfH_{\!\bflambda^{\, 0}}}$ not used for $i=1$) }
    \EndFor
    \State{$\nabla_{\!\bflambda}\Gamma(\bflambda^{\, i})$ and $[\Gamma{\,''}(\bflambda^{\, i})]$ using
           Eqs.~\eqref{eq:eq120} and \eqref{eq:eq121}, and  $\curD_{\bfH_{\!\bflambda^{\, i}}}$ }
    \State{$\error(i)$ using Eq.~\eqref{eq:eq137}}
    \State{$\bflambda^{\, i+1} = \bflambda^{\, i} - \alpha_{\prelax}(i)\, [\Gamma{\,''}(\bflambda^{\, i})]^{-1}\, \nabla_{\!\bflambda}\Gamma(\bflambda^{\, i})$ using  Eq.~\eqref{eq:eq135} with a relaxation factor $\alpha_{\prelax}(i)\in ]0\, , 1]$ }
    \State{$\bflambda^{\, i} \leftarrow \bflambda^{\, i+1}$}
    \State{$\curD_{\bfH_{\!\bflambda^{\,i-1}}} \leftarrow \curD_{\bfH_{\!\bflambda^{\,i}}}$ }
\EndFor
\State{$\bflambda^\psol = \bflambda^{i_\ppsol}$, $i_\psol = \arg\,\min_i\error(i)\,\,$ from Eq.~\eqref{eq:eq136}}
\State{$\curD_{\bfH_\pppost} \leftarrow \curD_{\bfH_{\bflambda^\ppsol}}$}
\end{algorithmic}
\end{algorithm}
\section{Numerical illustration}
\label{sec:Section7}
We consider a supervised case. The training set $D_d = \{ \bfx^1,\ldots , \bfx^{N_d}\}$ with $\bfx^j\!\! = \!(\bfq_d^j,\bfw_d^j) \in \RRnx \!\! = \!\RRnq\!\times\!\RRnw$
is made up of $N_d$ independent realizations of random variable $\bfX=(\bfQ,\bfW)$, which are generated as explained in \ref{appendix:A} for which $n_x=430\, 098$, $n_q=10\, 098$, $n_w= 420\, 000$, and $N_d\in\{100,200,300,400\}$.
The target set $D_\pexp$ is made up of $N_r$ independent realizations $\bfq_\pexp^r \in\RRnq$ of random variable $\bfQ_\pexp$ for which $N_r \in [50\, , N_\pexp]$ with $N_\pexp\in\{100,200,300,400\}$.
The $\RRnq$-valued random variable $\bfQ$ corresponds to the finite element discretization of a $\RR^3$-valued random field
$\{\UU(\bfomega) = (\UU_1(\bfomega), \UU_2(\bfomega), \UU_3(\bfomega) \, ,\bfomega\in\overline\Omega\}$ and
the $\RRnw$-valued random variable $\bfW$ is a nonlinear transformation of the finite element discretization of a
$\MM_6^+$-valued random field $\{[\bfG(\bfomega)] \, , \bfomega\in\Omega\}$, in which $\Omega$ is the open bounded set of $\RR^3$ defined in \ref{appendix:A.1}, and where $\bfQ$ and $\bfW$ are constructed in \ref{appendix:A.2}.
Regarding the presentation of the results, and having to limit the number of figures, the probability density functions and the convergence curves when they will be relative to  $\bfQ$, will be limited to $3$ components, denoted by  $Q_{\pobs,1}$, $Q_{\pobs,2}$, and $Q_{\pobs,3}$ that are also defined at the end of \ref{appendix:A.2}.
\subsection{Training set}
\label{sec:Section7.1}
The training set is generated as explained in \ref{appendix:A.2} with the stochastic boundary value problem defined in \ref{appendix:A.1}.
For illustration, Fig.~\ref{fig:figure2} shows one realization $\theta\in\Theta$ of the components $(1,1)$, $(1,2)$, and $(4,4)$ of the $\MM^+_6$-valued random field $(\omega_1,\omega_2)\mapsto [\bfG(\omega_1,\omega_2,\omega_3)]$ in the plane $\omega_3 = 0.095774$
and Fig.~\ref{fig:figure3} shows the corresponding realization of the components $k=1,2,3$ of the real-valued random field
$(\omega_1,\omega_2)\mapsto \UU_k(\omega_1,\omega_2,\omega_3)$ in the plane $\omega_3 = 0.095774$.
\begin{figure}[!h]
    \centering
    \begin{subfigure}[b]{0.30\textwidth}
    \centering
        \includegraphics[width=\textwidth]{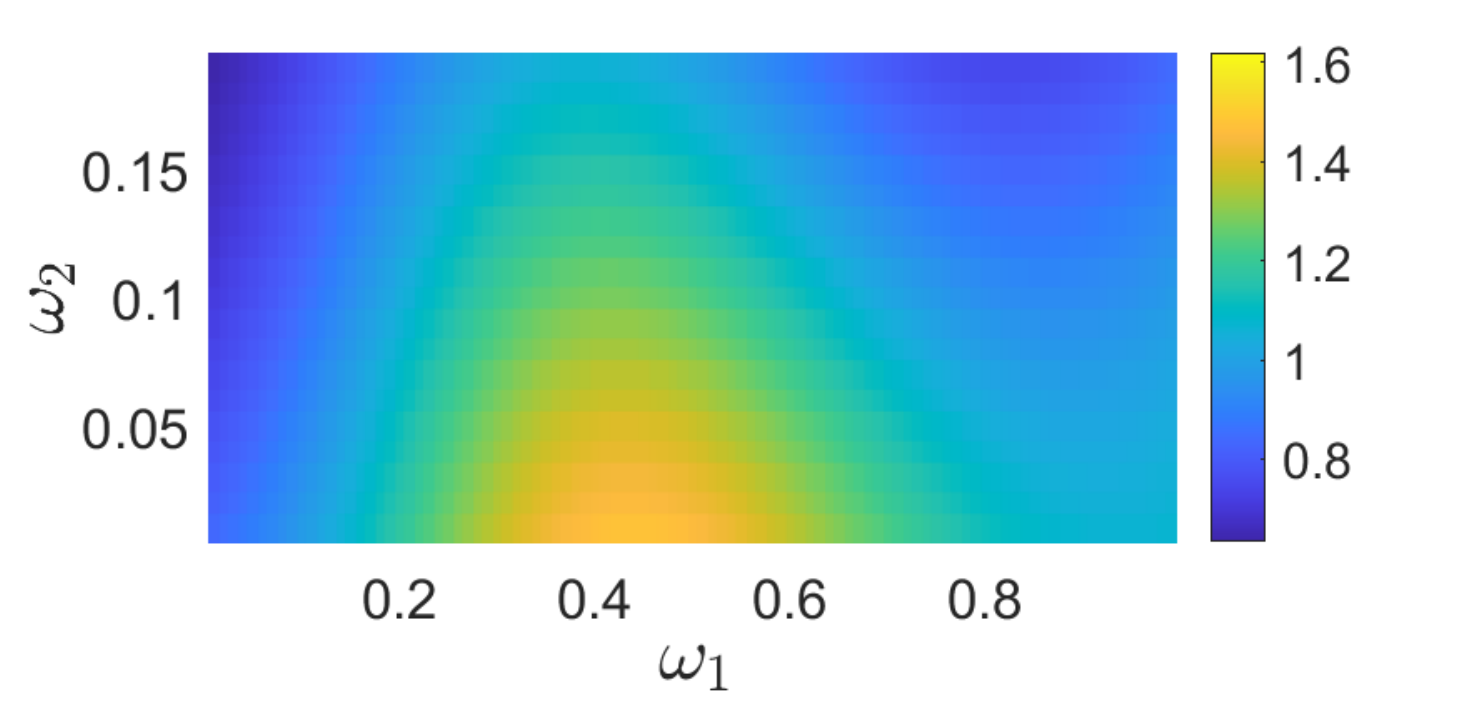}
        \caption{ $(\omega_1,\omega_2)\mapsto [\bfG(\omega_1,\omega_2,\omega_3 ; \theta)]_{11}$}
        \label{fig:figure2a}
    \end{subfigure}
    \begin{subfigure}[b]{0.30\textwidth}
        \centering
        \includegraphics[width=\textwidth]{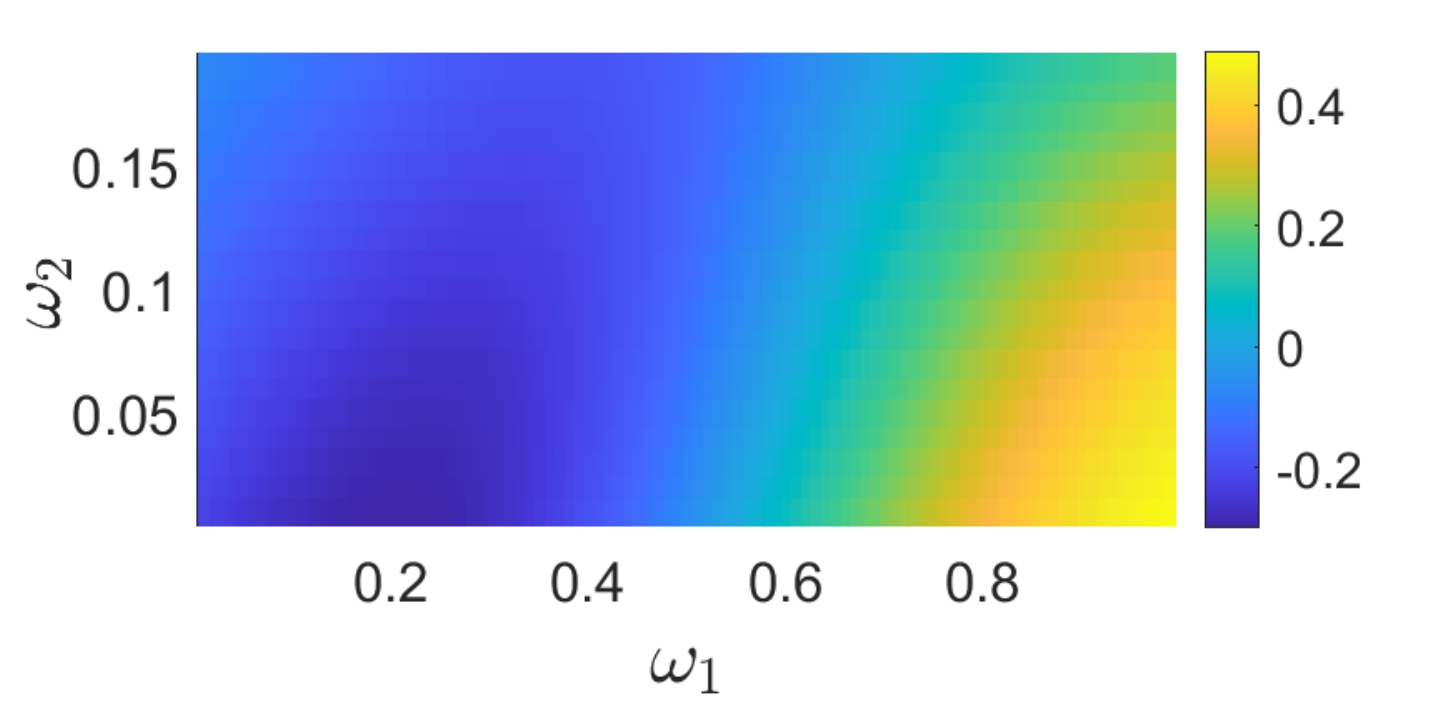}
        \caption{ $(\omega_1,\omega_2)\mapsto [\bfG(\omega_1,\omega_2,\omega_3 ;\theta)]_{12}$}
        \label{fig:figure2b}
    \end{subfigure}
    \begin{subfigure}[b]{0.30\textwidth}
        \centering
        \includegraphics[width=\textwidth]{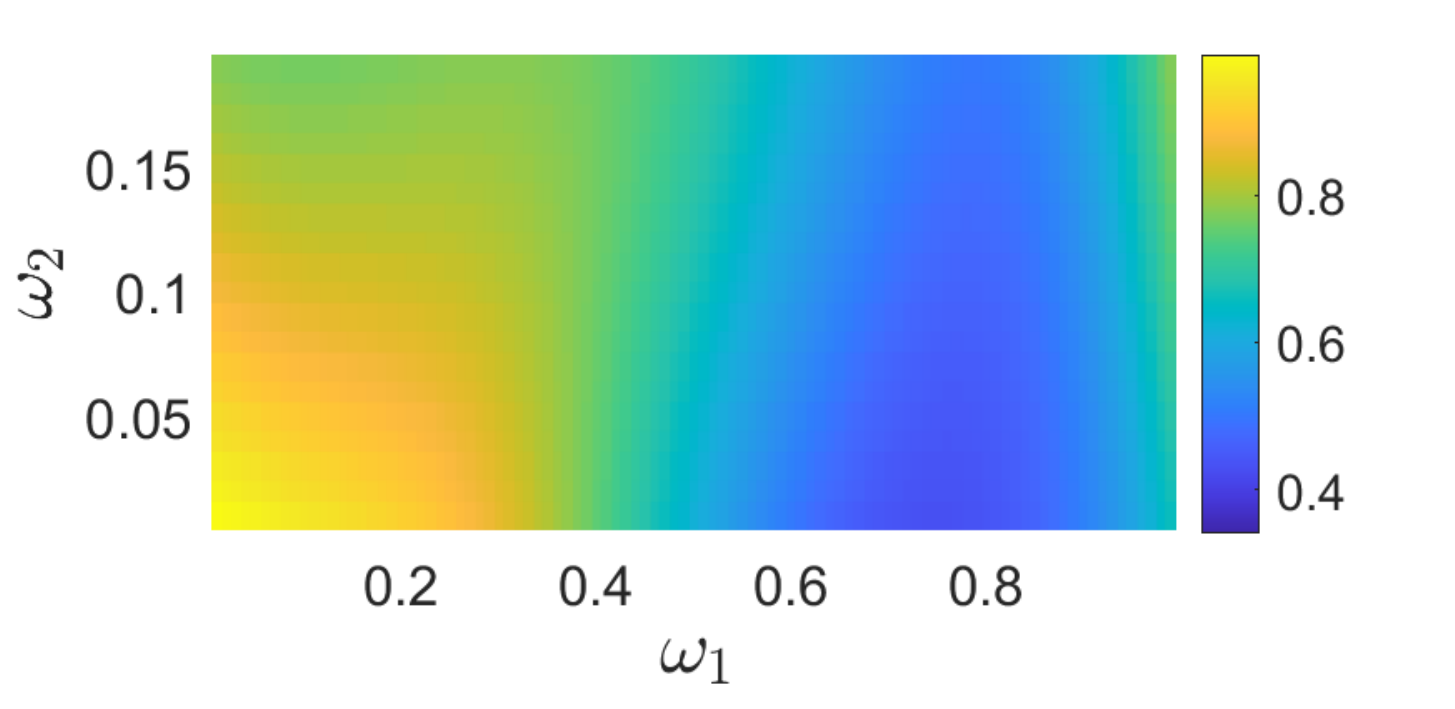}
        \caption{ $(\omega_1,\omega_2)\mapsto [\bfG(\omega_1,\omega_2,\omega_3 ;\theta)]_{44}$}
        \label{fig:figure2c}
    \end{subfigure}
    \caption{For the training set, example of one realization $\theta\in\Theta$ of the components $(1,1)$, $(1,2)$, and $(4,4)$ of the random field $(\omega_1,\omega_2)\mapsto [\bfG(\omega_1,\omega_2,\omega_3)]$ in the plane $\omega_3 = 0.095774$.}
    \label{fig:figure2}
\end{figure}
\begin{figure}[!h]
    \centering
    \begin{subfigure}[b]{0.30\textwidth}
    \centering
        \includegraphics[width=\textwidth]{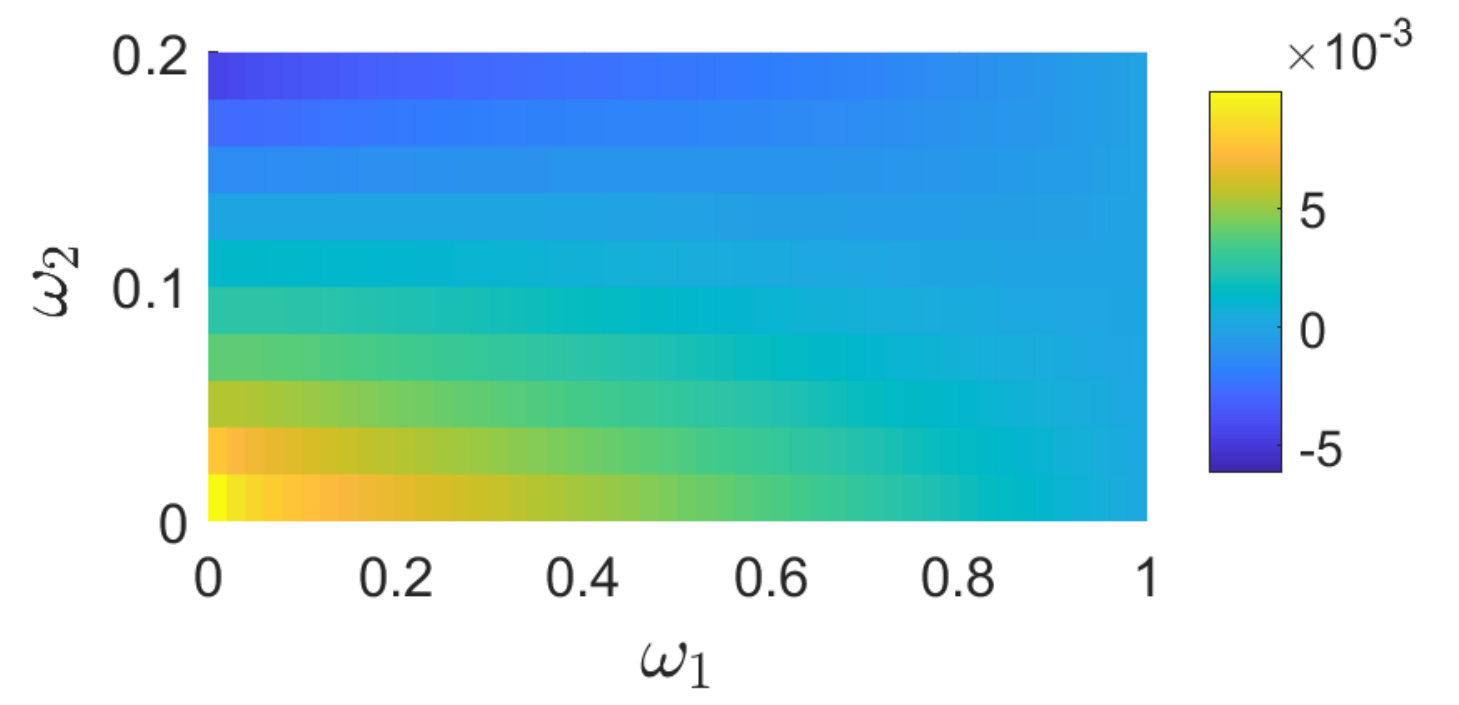}
        \caption{ $(\omega_1,\omega_2)\mapsto \UU_1(\omega_1,\omega_2,\omega_3 ;\theta)$}
        \label{fig:figure3a}
    \end{subfigure}
    \begin{subfigure}[b]{0.30\textwidth}
        \centering
        \includegraphics[width=\textwidth]{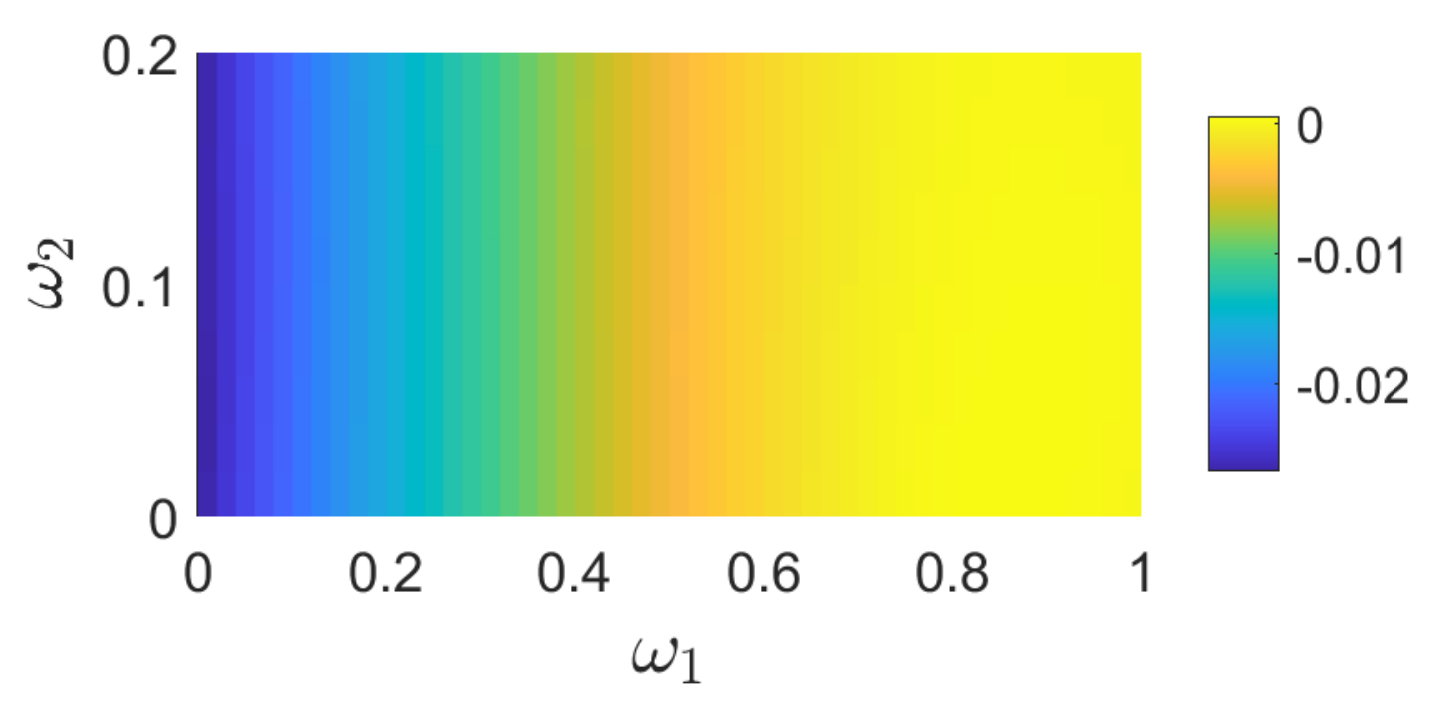}
        \caption{ $(\omega_1,\omega_2)\mapsto \UU_2(\omega_1,\omega_2,\omega_3 ;\theta)$}
        \label{fig:figure3b}
    \end{subfigure}
    \begin{subfigure}[b]{0.30\textwidth}
        \centering
        \includegraphics[width=\textwidth]{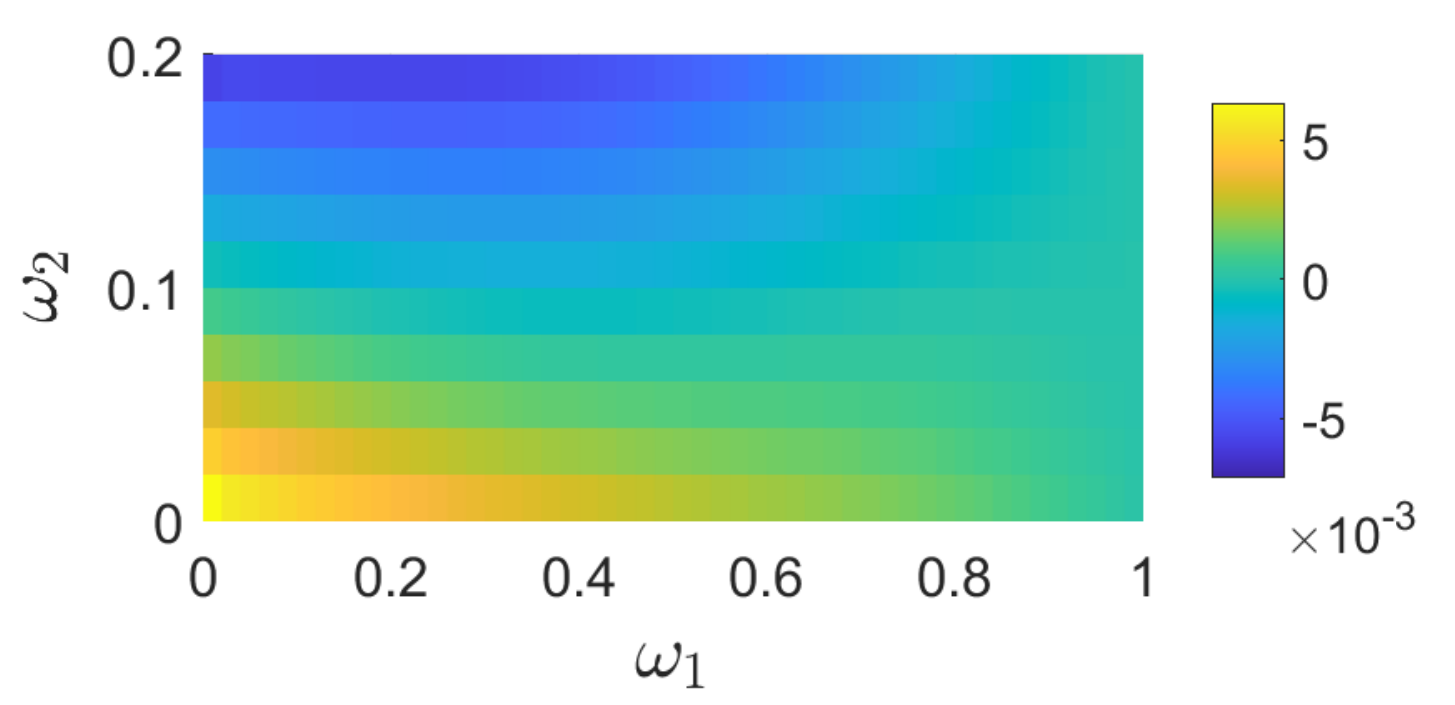}
        \caption{ $(\omega_1,\omega_2)\mapsto \UU_3(\omega_1,\omega_2,\omega_3 ;\theta)$}
        \label{fig:figure3c}
    \end{subfigure}
    \caption{For the training set, example of one realization $\theta\in\Theta$ of the components $1$, $2$, and $3$ of the random field
                $(\omega_1,\omega_2)\mapsto \UU(\omega_1,\omega_2,\omega_3)$ in the plane $\omega_3 = 0.095774$.}
    \label{fig:figure3}
\end{figure}
\subsection{Reduced representation}
\label{sec:Section7.2}
The reduced representation is constructed by using a PCA of $\bfX=(\bfQ,\bfW)$ as explained in Section~\ref{sec:Section2}. With $\varepsilon_\pPCA = 0.0001$, for $N_d=100$, $200$, $300$, and $400$,  we have, respectively, $\nu =99$, $192$, $271$, and $331$.
For $N_d=100$, Fig.~\ref{fig:figure4} displays the graph of the eigenvalues $\alpha\mapsto\kappa_\alpha$ of $[\widehat C_\bfX]$ and the graph of the error function $\nu\mapsto \perr_\pPCA(\nu\, ;N_d)$ defined by Eq.~\eqref{eq:eq2bis0}.
\begin{figure}[!h]
    \centering
    \begin{subfigure}[b]{0.25\textwidth}
    \centering
        \includegraphics[width=\textwidth]{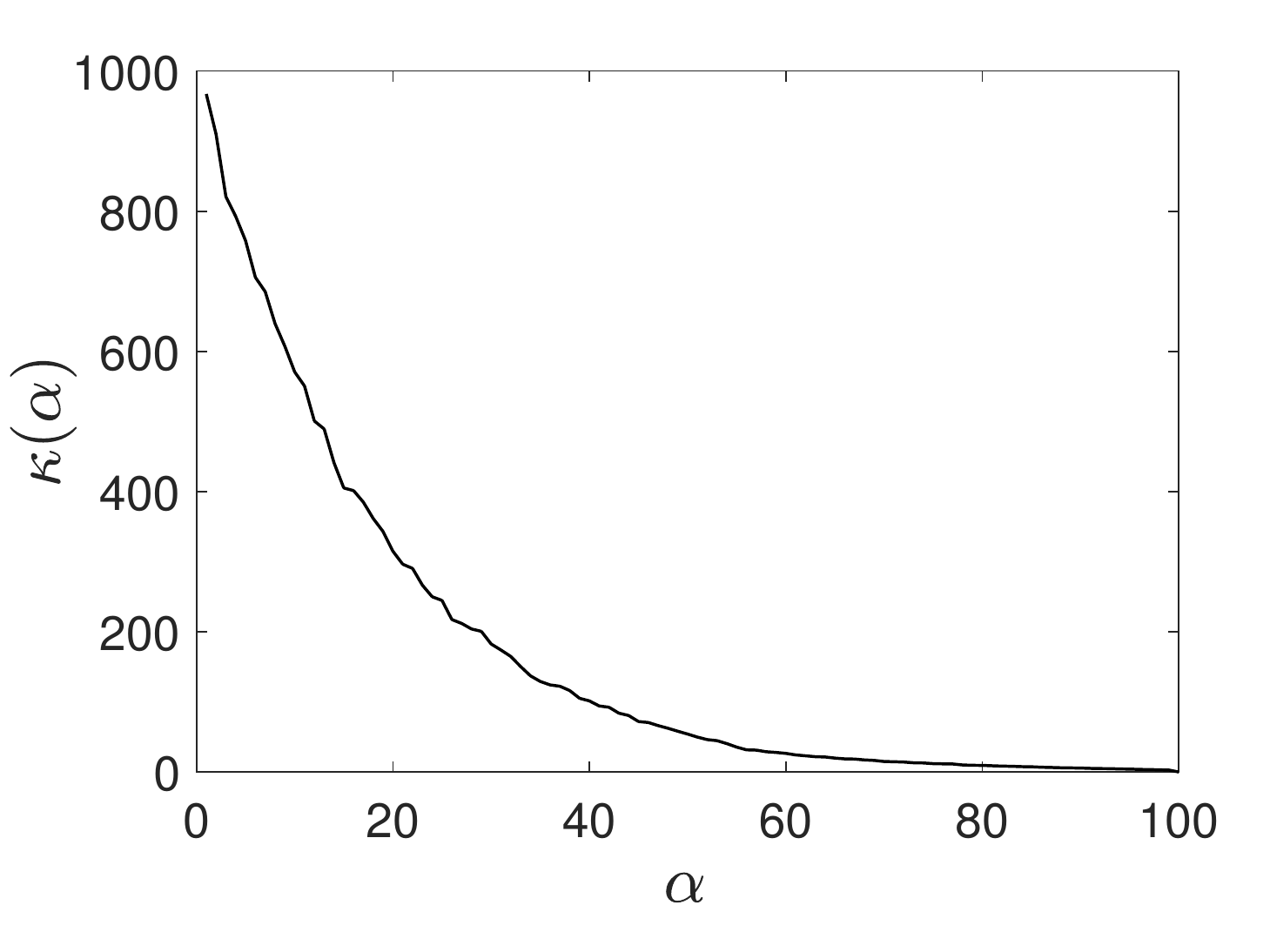}
        \caption{Graph of $\alpha\mapsto\kappa_\alpha$}
        \label{fig:figure4a}
    \end{subfigure}
    \begin{subfigure}[b]{0.25\textwidth}
        \centering
        \includegraphics[width=\textwidth]{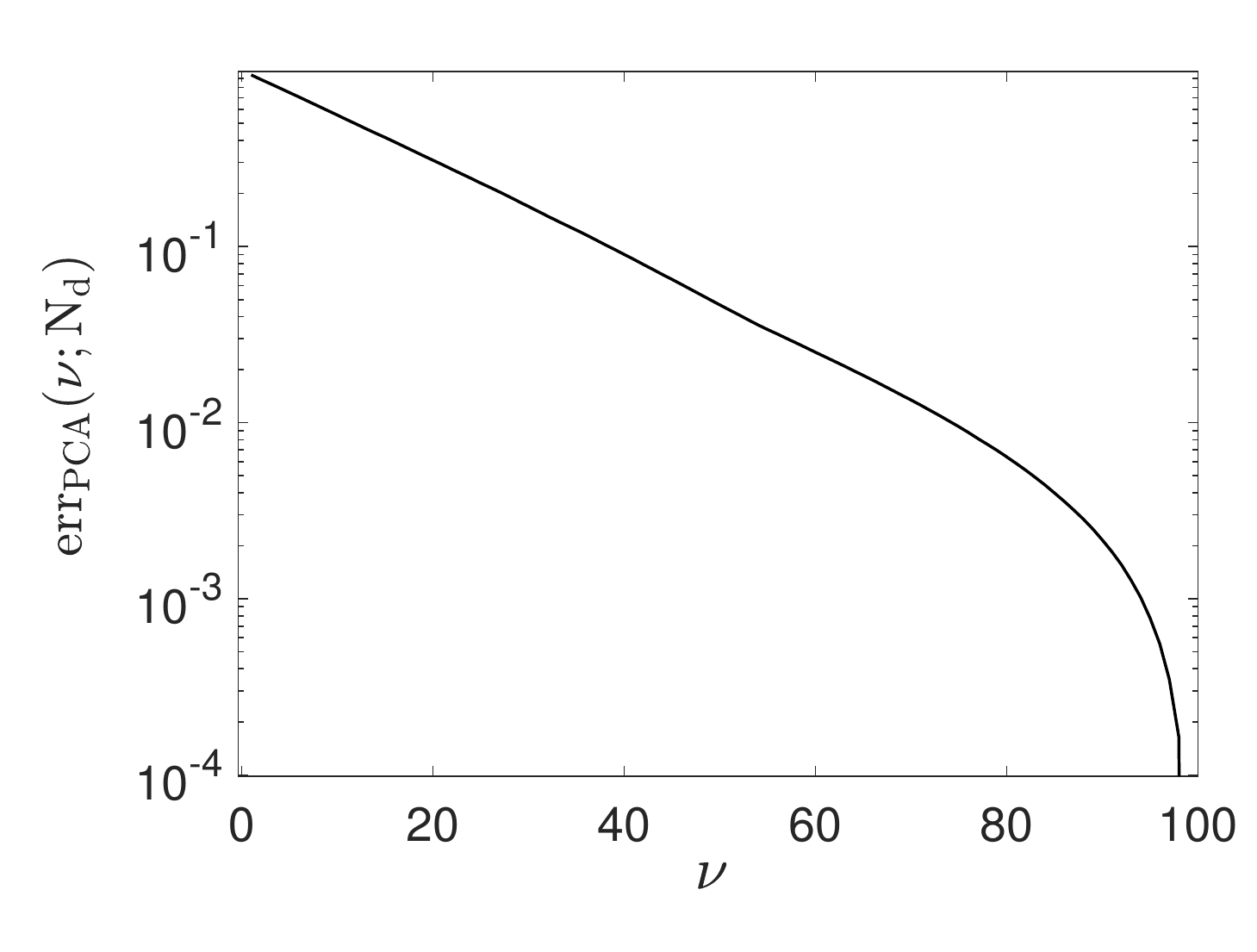}
        \caption{Graph of $\nu\mapsto \perr_\pPCA(\nu\, ;N_d)$}
        \label{fig:figure4b}
    \end{subfigure}
    \caption{Principal component analysis of the training set performed  for $N_d=100$. Graph of the eigenvalues $\alpha\mapsto\kappa_\alpha$ of $[\widehat C_\bfX]$ and error function $\nu\mapsto \perr_\pPCA(\nu\, ;N_d)$ defined by Eq.~\eqref{eq:eq2bis0}.}
    \label{fig:figure4}
\end{figure}
\subsection{Numerical values of the algorithm parameters for computing the constrained learned set.}
\label{sec:Section7.3}
In all the calculations and for any value of $\bflambda$, the number $N$ of the independent realizations $\{\bfeta_\bflambda^1,\ldots ,\bfeta_\bflambda^N\}$ of the constrained learned set $\curD_{\bfH_\bflambda}$ (see Section~\ref{sec:Section6.7}) is fixed to the value $N=1\, 000$ (this value has been obtained from a convergence analysis with respect to $N$).
The parameters of the St\"ormer-Verlet scheme are $M_s = 30$, $\Delta t = 0.2188$, and $f_0=4$. Due to the choice $f_0=4$, the stationary regime of the ISDE is obtained from instant $21\times \Delta t$ and the realizations is extracted at  $30\times \Delta t$.
\subsection{Iterative algorithm for computing $\bflambda^\psol$.}
\label{sec:Section7.4}
For computing the solution $\bflambda^\psol$ of the Lagrange multiplier, Algorithm~\ref{algorithm:1} is used.
For $N_d=100$ and $N_r=20$, Fig.~\ref{fig:figure5} displays the graph of the relaxation factor $i\mapsto \alpha_\prelax(i)$ and the graph of the error function $i\mapsto \error(i)$ defined by Eq.~\eqref{eq:eq137}. It can be seen a fast convergence of the algorithm as a function of the iteration number.
Fig.~\ref{fig:figure6} shows the graph of function  $r\mapsto b^c_r$ defined by Eq.~\eqref{eq:eq75}, representing the components of vector $\bfb^c =(b_1^c,\ldots , b_{N_r}^c)$, and the graph of function $r\mapsto \lambda^\psol_r$, representing the components of vector $\bflambda^\psol = (\lambda_1^\psol,\ldots , \lambda_{N_r}^\psol)$ defined by Eq.~\eqref{eq:eq136}.
\begin{figure}[!h]
    \centering
    \begin{subfigure}[b]{0.25\textwidth}
    \centering
        \includegraphics[width=\textwidth]{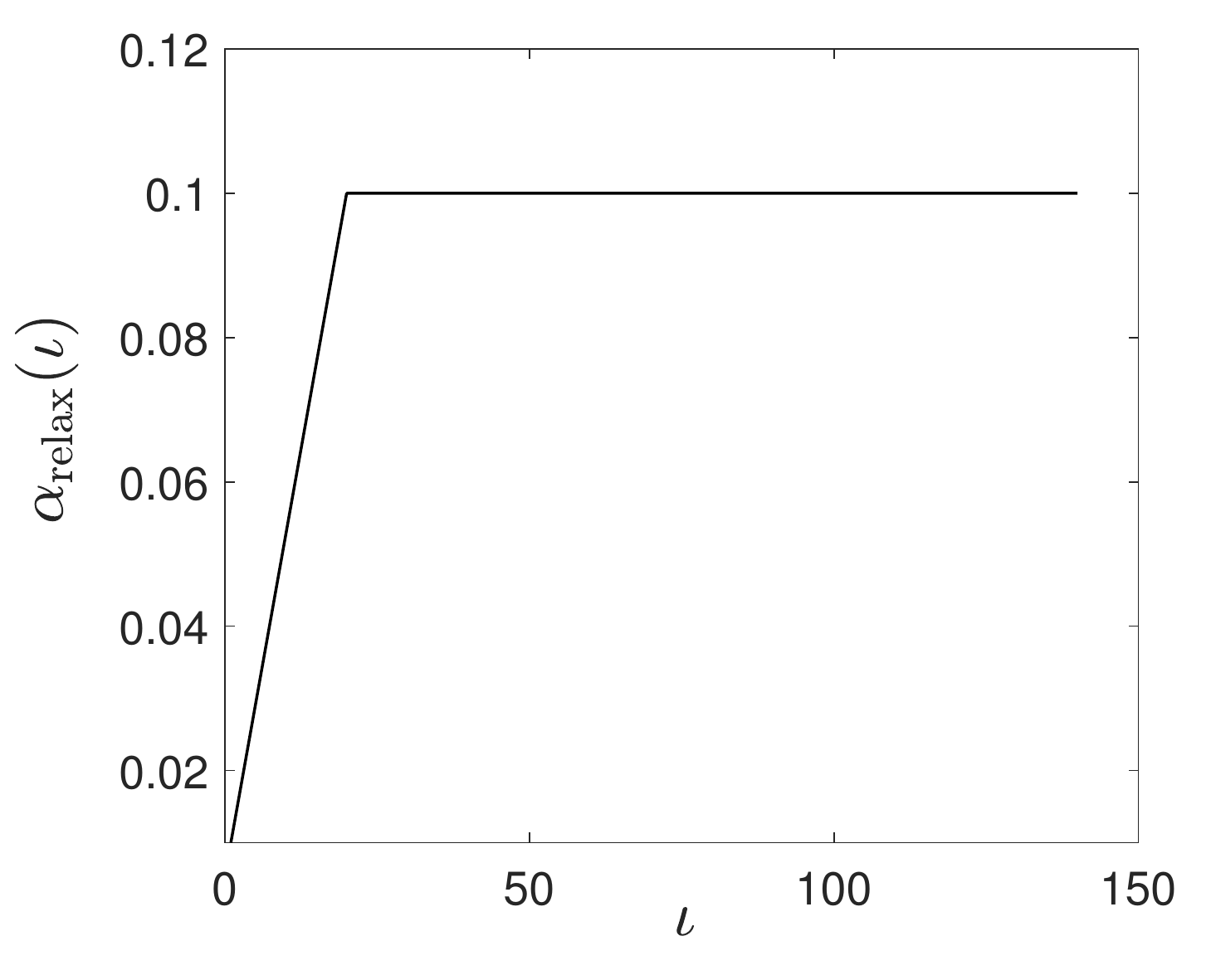}
        \caption{Graph of $i\mapsto \alpha_\prelax(i)$}
        \label{fig:figure5a}
    \end{subfigure}
    \begin{subfigure}[b]{0.25\textwidth}
        \centering
        \includegraphics[width=\textwidth]{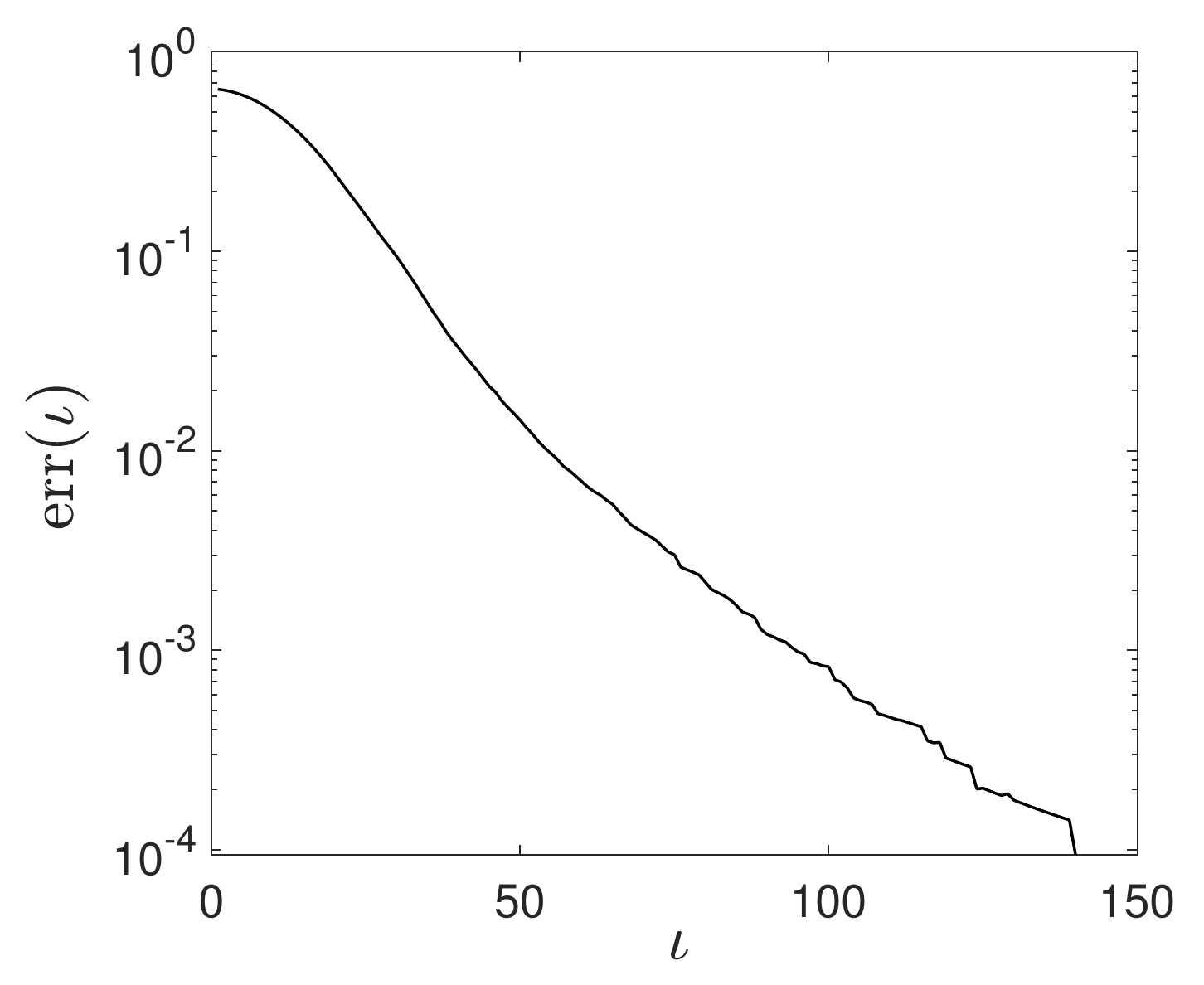}
        \caption{Graph of $i\mapsto \error(i)$}
        \label{fig:figure5b}
    \end{subfigure}
    \caption{Iterative algorithm for calculating $\bflambda^\psol$ with the constrained learned set for $N_d=100$ and $N_r=20$. Graph of the relaxation factor $i\mapsto \alpha_\prelax(i)$ defined in Algorithm~\ref{algorithm:1} and graph of the error function $i\mapsto \error(i)$ defined by Eq.~\eqref{eq:eq137}.}
    \label{fig:figure5}
\end{figure}
\begin{figure}[!h]
    \centering
    \begin{subfigure}[b]{0.25\textwidth}
    \centering
        \includegraphics[width=\textwidth]{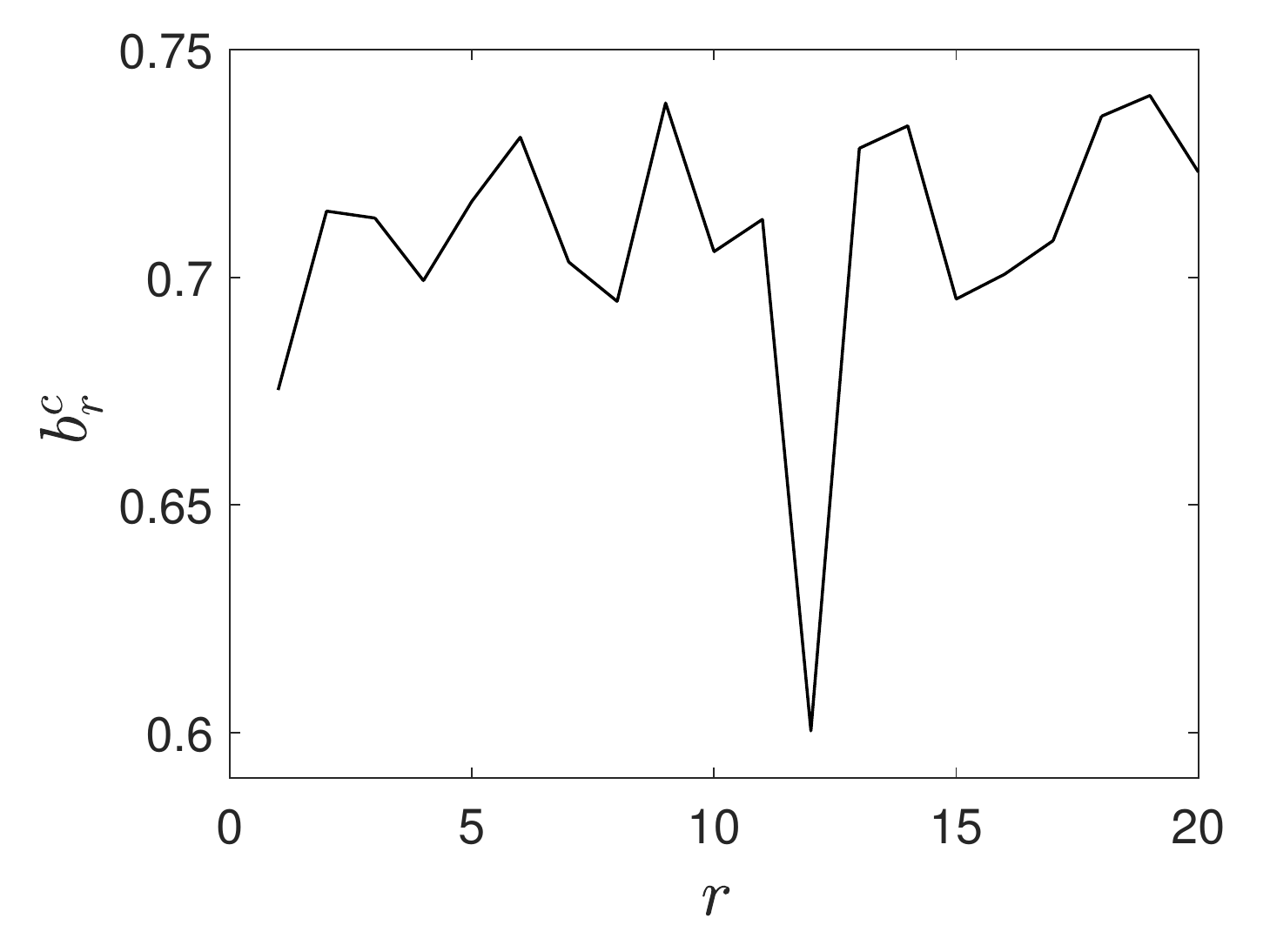}
        \caption{Graph of $r\mapsto b^c_r$}
        \label{fig:figure6a}
    \end{subfigure}
    \begin{subfigure}[b]{0.25\textwidth}
        \centering
        \includegraphics[width=\textwidth]{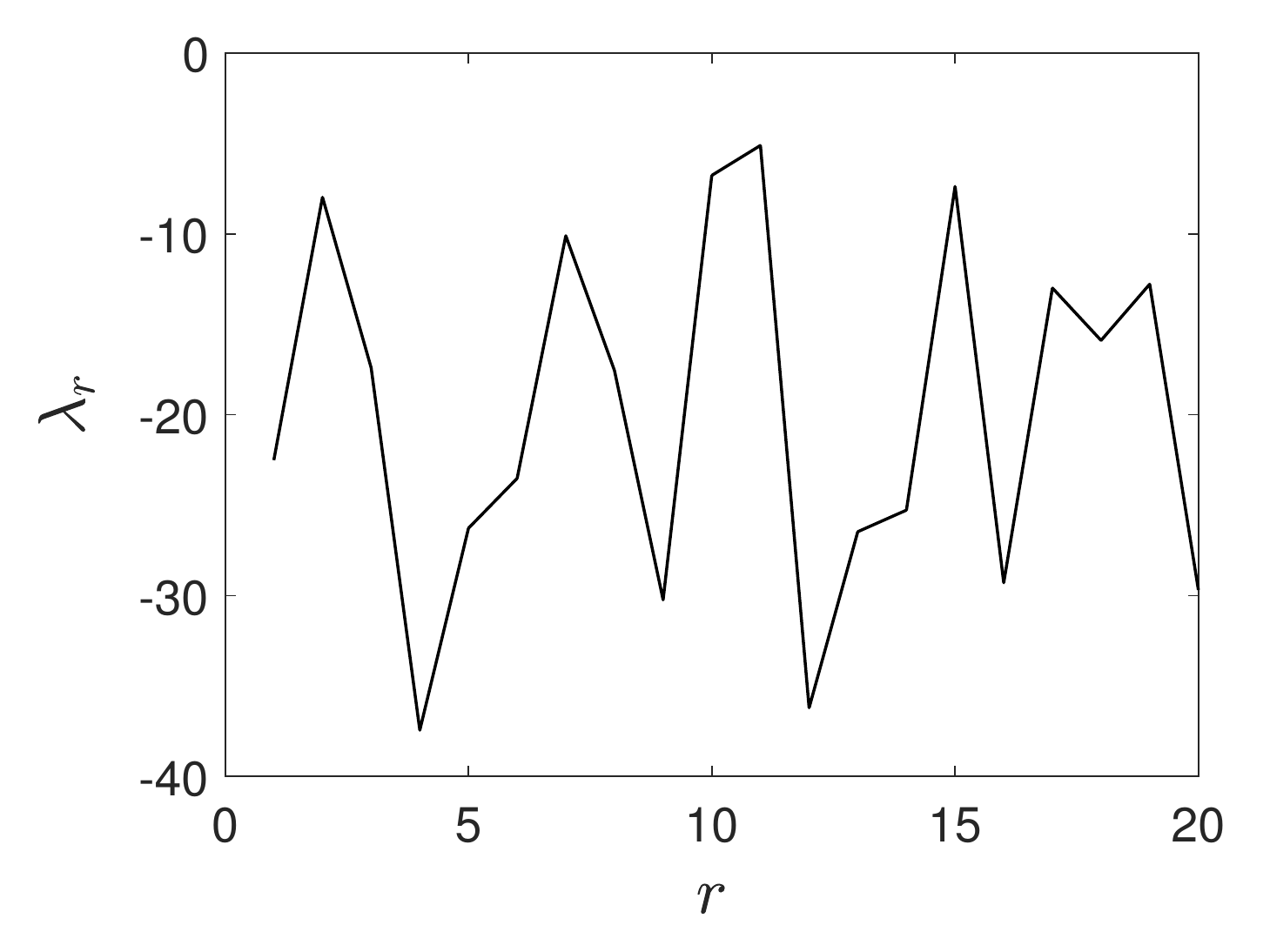}
        \caption{Graph of $r\mapsto \lambda^\psol_r$}
        \label{fig:figure6b}
    \end{subfigure}
    \caption{Constraint and optimal Lagrange multiplier estimated with the constrained learned set for $N_d=100$ and $N_r=20$. Graph of function $r\mapsto b^c_r$ defined by Eq.~\eqref{eq:eq75} and graph of function $r\mapsto \lambda^\psol_r$ with $\bflambda^\psol$ defined by Eq.~\eqref{eq:eq136}.}
    \label{fig:figure6}
\end{figure}
\subsection{Posterior probability measure of $\bfQ_\ppost$ estimated with the constrained learned set}
\label{sec:Section7.5}
For $N_d=100$, the convergence with respect to $N_r$ of the posterior probability measure of $\bfQ_\ppost(N_r)$ estimated with the constrained learned set has been analyzed by studying, for $k=1,2,3$, the  mean-square norm $|||Q_{\pobs,k}(N_r)||| = \{ E\{ Q_{\pobs,k}(N_r)^2\}\}^{1/2}$ of random component $Q_{\pobs,k}(N_r)$ of $\bfQ_\ppost(N_r)$ (which depends on $N_r$).
For $k=1,2,3$, Fig.~\ref{fig:figure7} shows the graph of function $N_r\mapsto |||Q_{\pobs,k}(N_r)|||$ as well as  the corresponding value for the training set and for the reference, which are both independent of $N_r$. This figure shows the good convergence with respect to $N_r$, which is reached for $N_r=20$.
Always for $N_d=100$, Fig.~\ref{fig:figure8} (a), (b), and (c) related to $N_r=20$, and  Fig.~\ref{fig:figure8} (d), (e), and (f) related to $N_r=100$, display the probability density functions of the random variables $Q_{\pobs,1}$, $Q_{\pobs,2}$, and $Q_{\pobs,3}$, estimated with the training set, with the constrained learned set (the posterior), and the reference.
Figs.~\ref{fig:figure7} and  \ref{fig:figure8} show that the posteriors are close to the targets (this good result holds for all the components of $\bfQ$). For this supervised framework, it can be seen that the proposed method performs very well for the quantity of interest $\bfQ$ for which a target has been given, which means that there would also be very good behavior of the method if used in an unsupervised setting.
\begin{figure}[!h]
    \centering
    \begin{subfigure}[b]{0.30\textwidth}
    \centering
        \includegraphics[width=\textwidth]{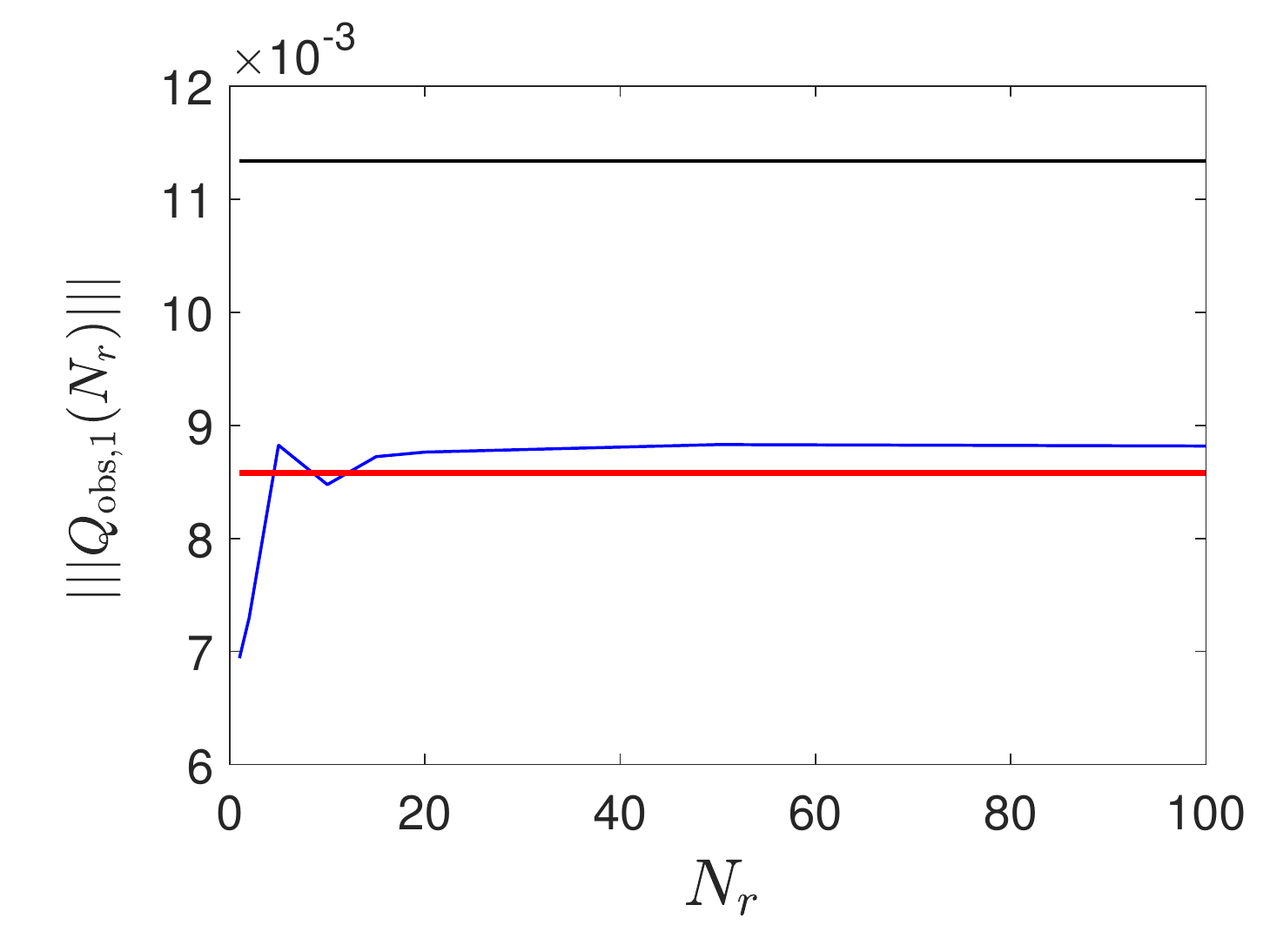}
        \caption{$N_r\mapsto |||Q_{\pobs,1}(N_r)|||$}
        \label{fig:figure7a}
    \end{subfigure}
    \begin{subfigure}[b]{0.30\textwidth}
        \centering
        \includegraphics[width=\textwidth]{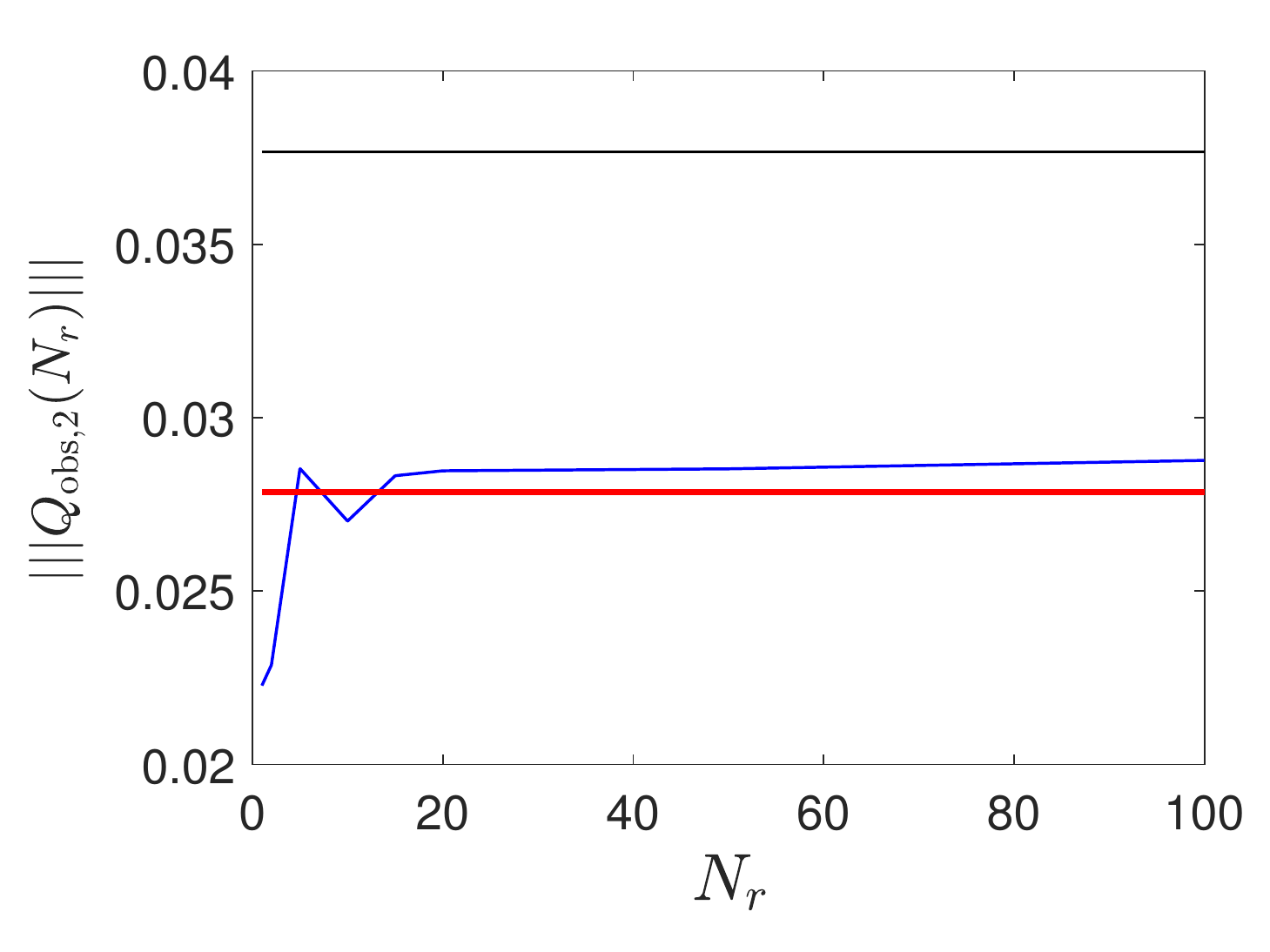}
        \caption{$N_r\mapsto |||Q_{\pobs,2}(N_r)|||$}
        \label{fig:figure7b}
    \end{subfigure}
    \begin{subfigure}[b]{0.30\textwidth}
        \centering
        \includegraphics[width=\textwidth]{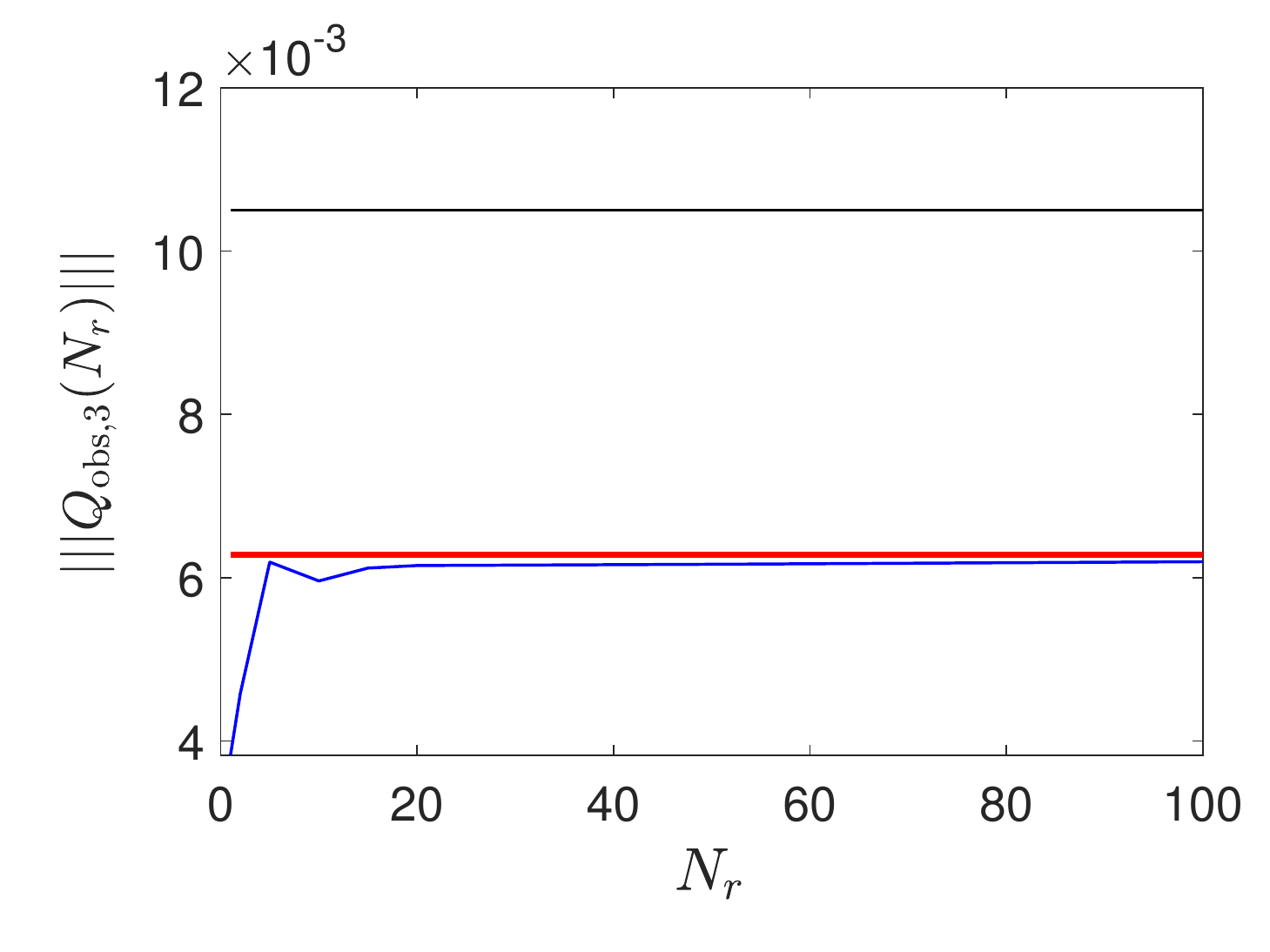}
        \caption{$N_r\mapsto |||Q_{\pobs,3}(N_r)|||$}
        \label{fig:figure7c}
    \end{subfigure}
    \caption{Convergence of the mean-square norm with respect to $N_r$ for $N_d=100$. Graph of $N_r\mapsto |||Q_{\pobs,1}(N_r)|||$, $|||Q_{\pobs,2}(N_r)|||$, and $|||Q_{\pobs,3}(N_r)|||$ estimated with the training set (top black line), with the constrained learned set (the posterior) (blue line), and the reference (red thick line).}
    \label{fig:figure7}
\end{figure}
\begin{figure}[!h]
    \centering
    \begin{subfigure}[b]{0.25\textwidth}
    \centering
        \includegraphics[width=\textwidth]{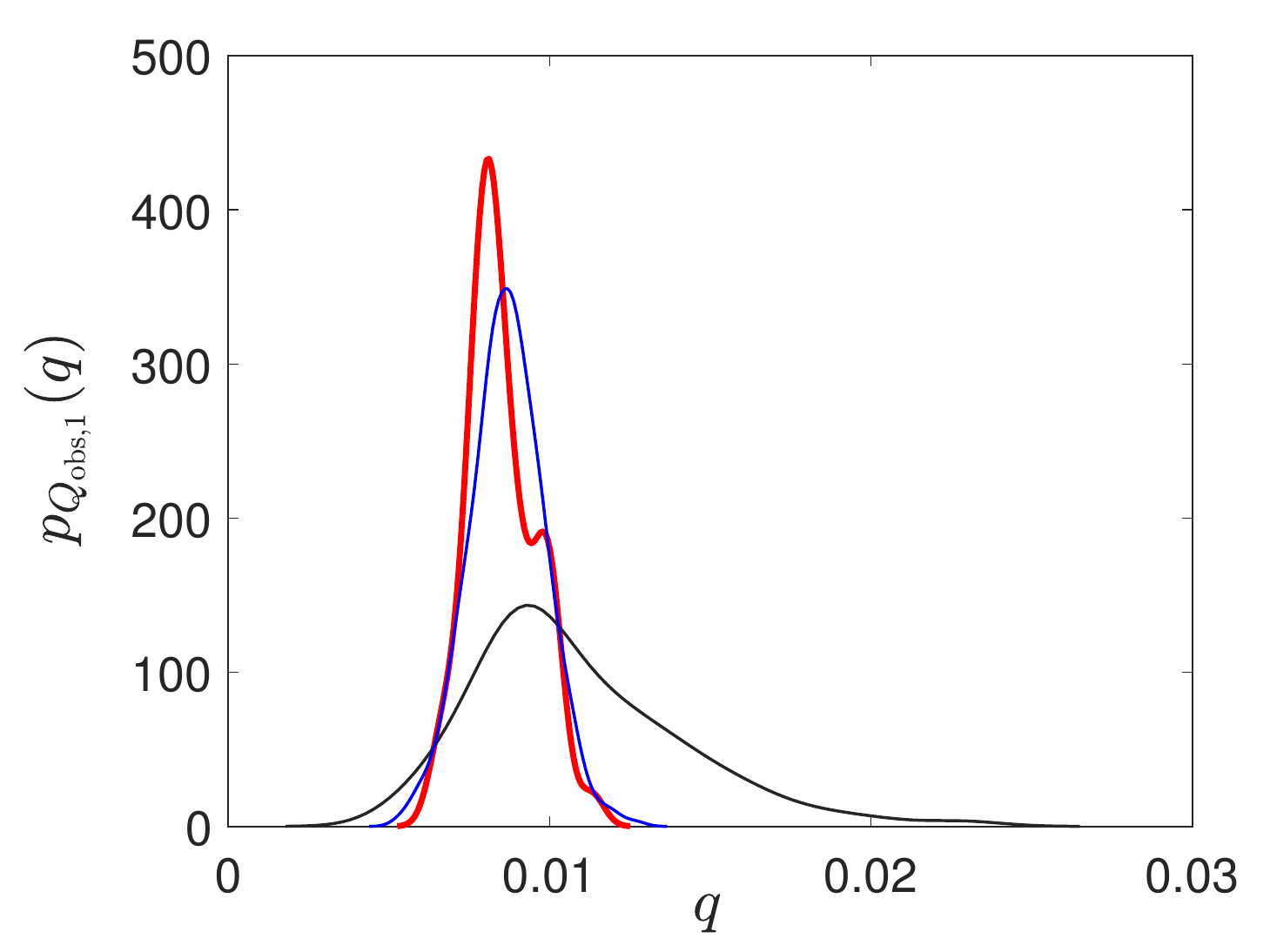}
        \caption{pdf $q\mapsto p_{Q_{\ppobs,1}}(q)$ for $N_r=20$}
        \label{fig:figure8a}
    \end{subfigure}
    \begin{subfigure}[b]{0.25\textwidth}
        \centering
        \includegraphics[width=\textwidth]{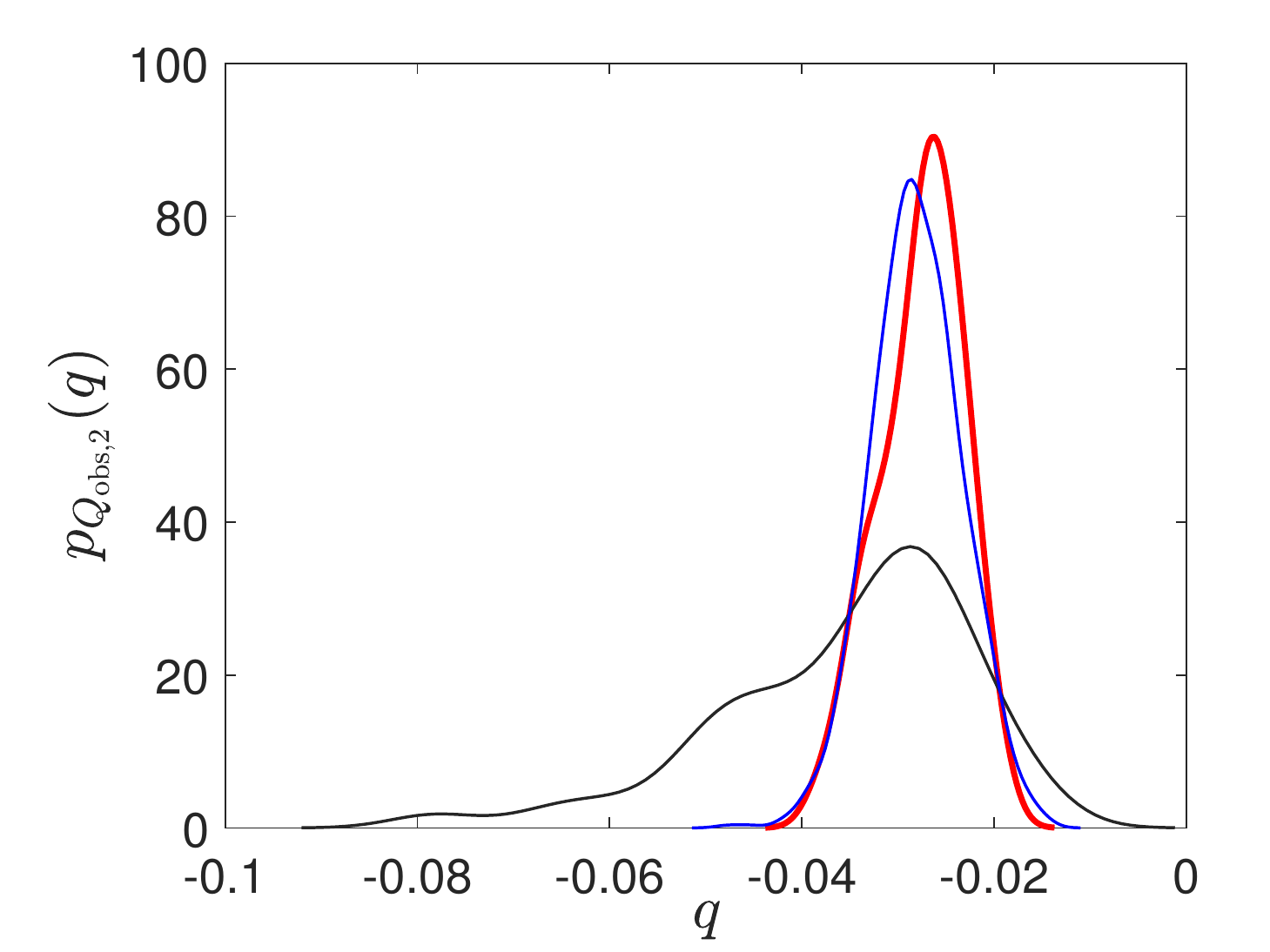}
        \caption{pdf $q\mapsto p_{Q_{\ppobs,2}}(q)$ for $N_r=20$}
        \label{fig:figure8b}
    \end{subfigure}
    \begin{subfigure}[b]{0.25\textwidth}
        \centering
        \includegraphics[width=\textwidth]{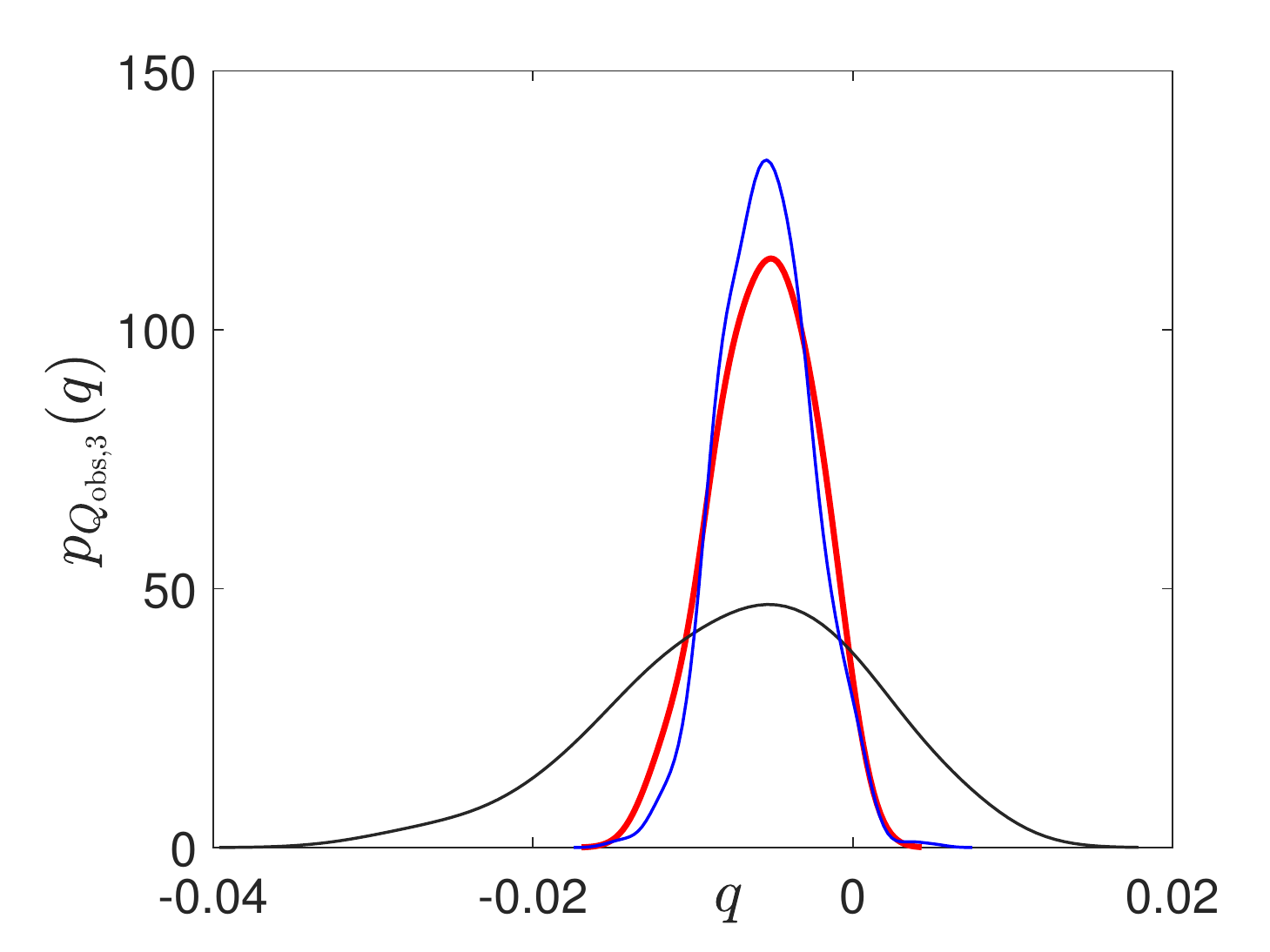}
        \caption{pdf $q\mapsto p_{Q_{\ppobs,3}}(q)$ for $N_r=20$}
        \label{fig:figure8c}
    \end{subfigure}
    \centering
    \begin{subfigure}[b]{0.25\textwidth}
    \centering
        \includegraphics[width=\textwidth]{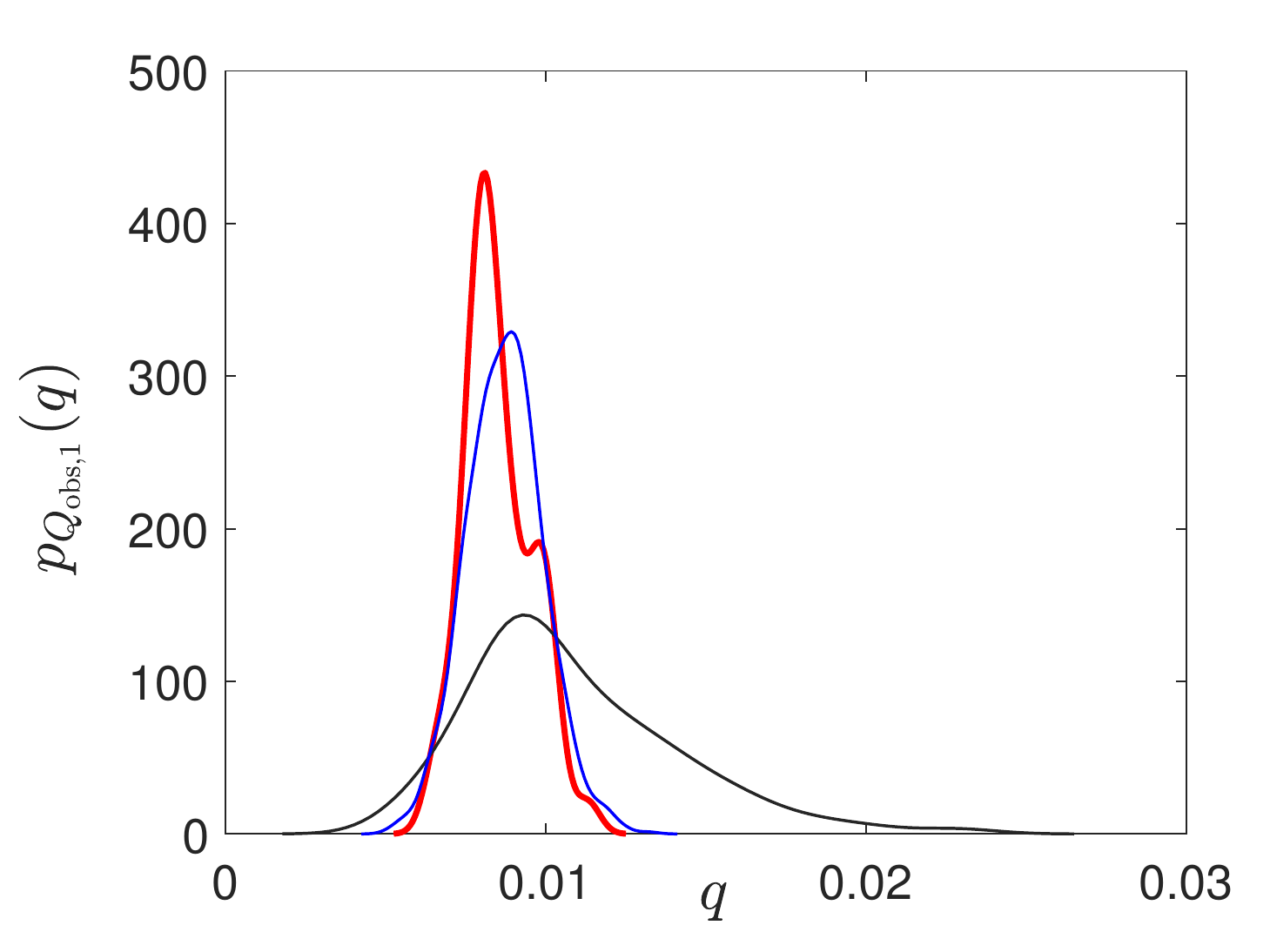}
        \caption{pdf $q\mapsto p_{Q_{\ppobs,1}}(q)$ for $N_r=100$}
        \label{fig:figure8d}
    \end{subfigure}
    \begin{subfigure}[b]{0.25\textwidth}
        \centering
        \includegraphics[width=\textwidth]{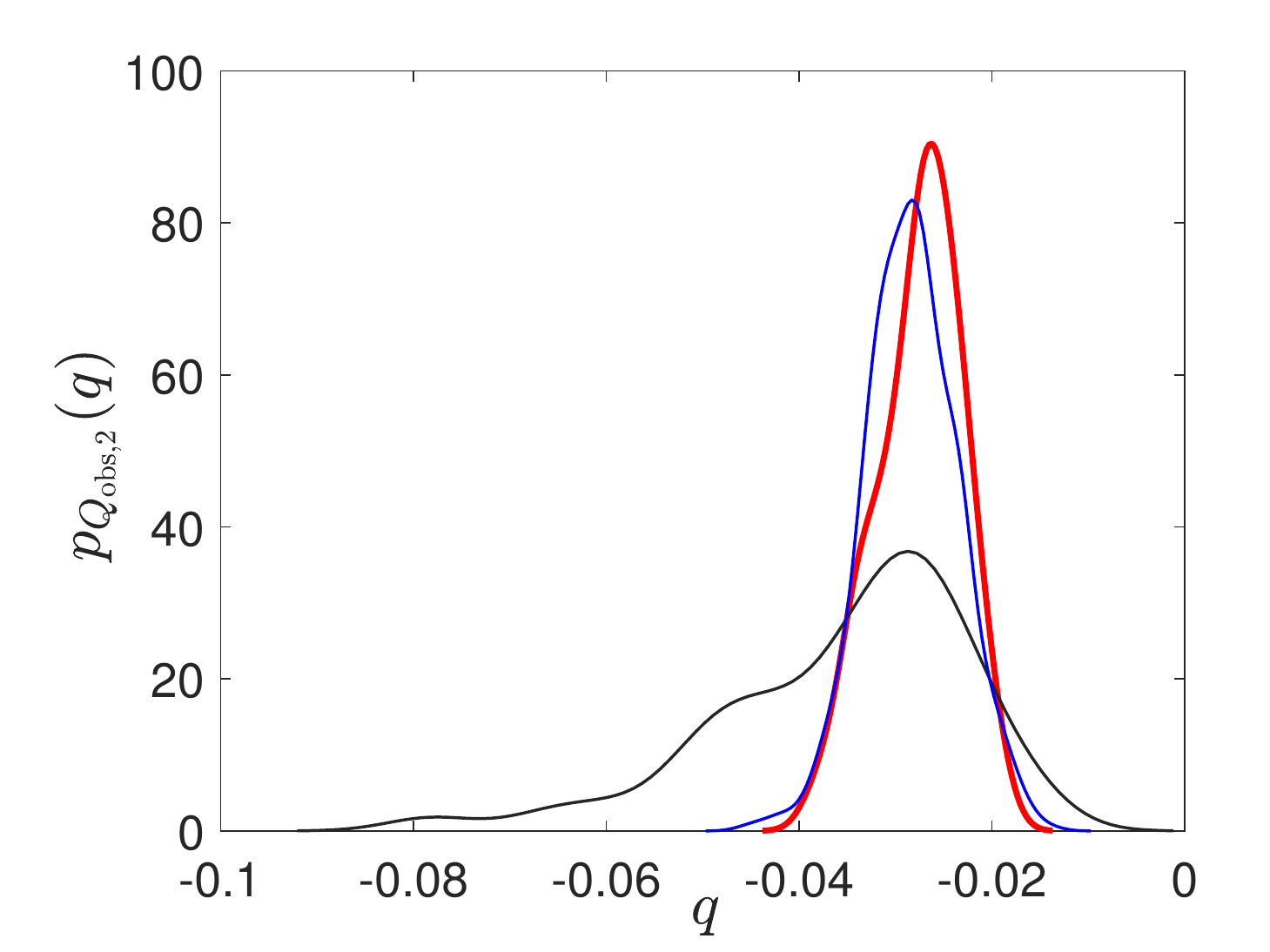}
        \caption{pdf $q\mapsto p_{Q_{\ppobs,2}}(q)$ for $N_r=100$}
        \label{fig:figure8e}
    \end{subfigure}
    \begin{subfigure}[b]{0.25\textwidth}
        \centering
        \includegraphics[width=\textwidth]{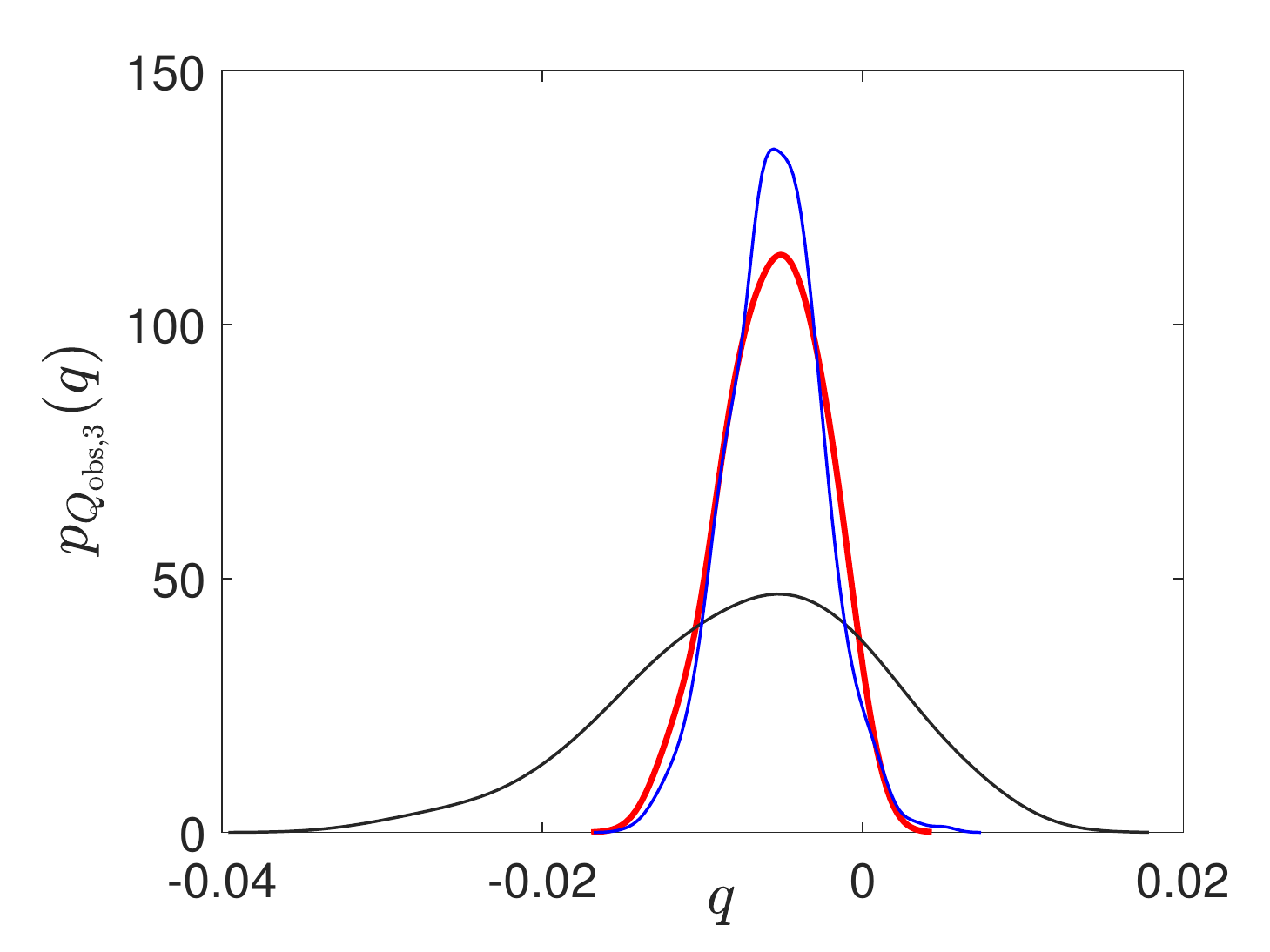}
        \caption{pdf $q\mapsto p_{Q_{\ppobs,3}}(q)$ for $N_r=100$}
        \label{fig:figure8f}
    \end{subfigure}
    \caption{For $N_d=100$, for $N_r=20$ (Figs. (a), (b), and (c)) and for $N_r=100$ (Figs. (d), (e), and (f)), probability density functions of the random variables $Q_{\pobs,1}$, $Q_{\pobs,2}$, and $Q_{\pobs,3}$, estimated with the training set (black line), with the constrained learned set (the posterior) (blue line), and the reference (red thick line).}
    \label{fig:figure8}
\end{figure}
\subsection{Posterior probability measure of $\bfW_\ppost$ estimated with the constrained learned set}
\label{sec:Section7.6}
Fig.~\ref{fig:figure9} is related to the standard deviation fields $(\omega_1,\omega_2)\mapsto \sigma_{G11}(\omega_1,\omega_2,\omega_3)$, $\sigma_{G12}(\omega_1,\omega_2,\omega_3)$, and $\sigma_{G44}(\omega_1,\omega_2,\omega_3)$ in the plane $\omega_3 = 0.095774$ of the components $(1,1)$, $(1,2)$, and $(4,4)$ of the random field  $(\omega_1,\omega_2)\mapsto [\bfG(\omega_1,\omega_2,\omega_3)]$ for the training set with $N_d=100$, for the transformation of the posterior $\bfW_\ppost$ (see \ref{appendix:A.2}) computed with the constrained learned set for which $N_d=100$ and $N_r=20$, and finally, for the reference.
\begin{figure}[!h]
    \centering
    \begin{subfigure}[b]{0.29\textwidth}
    \centering
        \includegraphics[width=\textwidth]{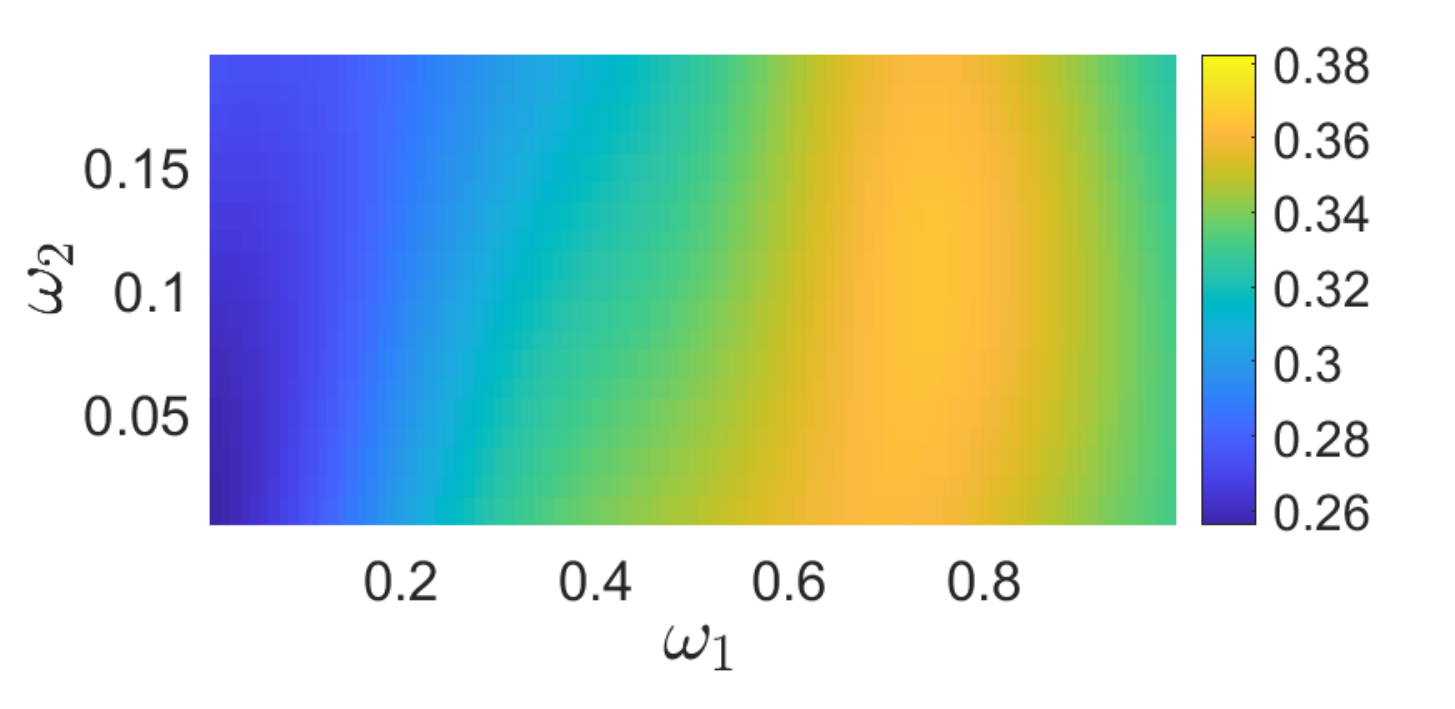}
        \caption{Training set, $(\omega_1,\omega_2)\mapsto \sigma_{G11}(\omega_1,\omega_2,\omega_3)$}
        \label{fig:figure9a}
    \end{subfigure}
    \begin{subfigure}[b]{0.29\textwidth}
        \centering
        \includegraphics[width=\textwidth]{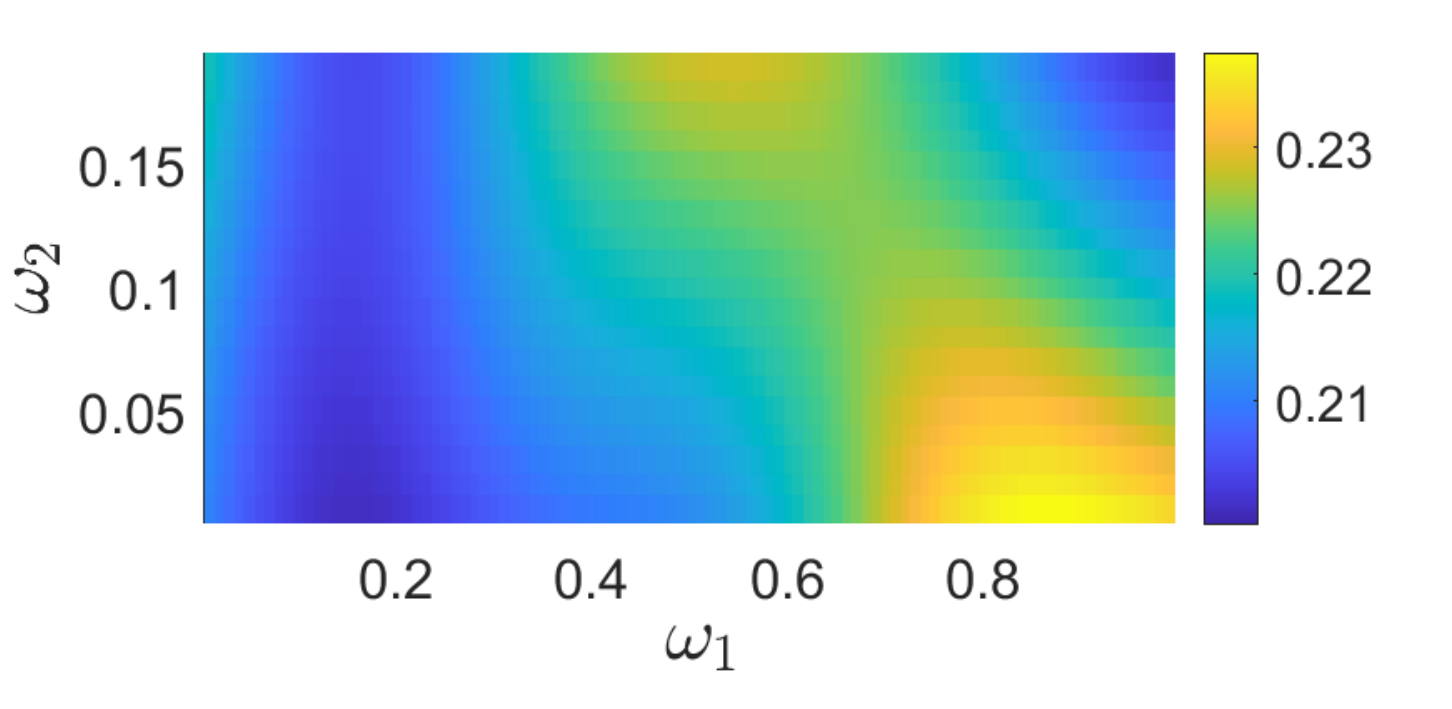}
        \caption{Training set, $(\omega_1,\omega_2)\mapsto \sigma_{G12}(\omega_1,\omega_2,\omega_3)$}
        \label{fig:figure9b}
    \end{subfigure}
    \begin{subfigure}[b]{0.29\textwidth}
        \centering
        \includegraphics[width=\textwidth]{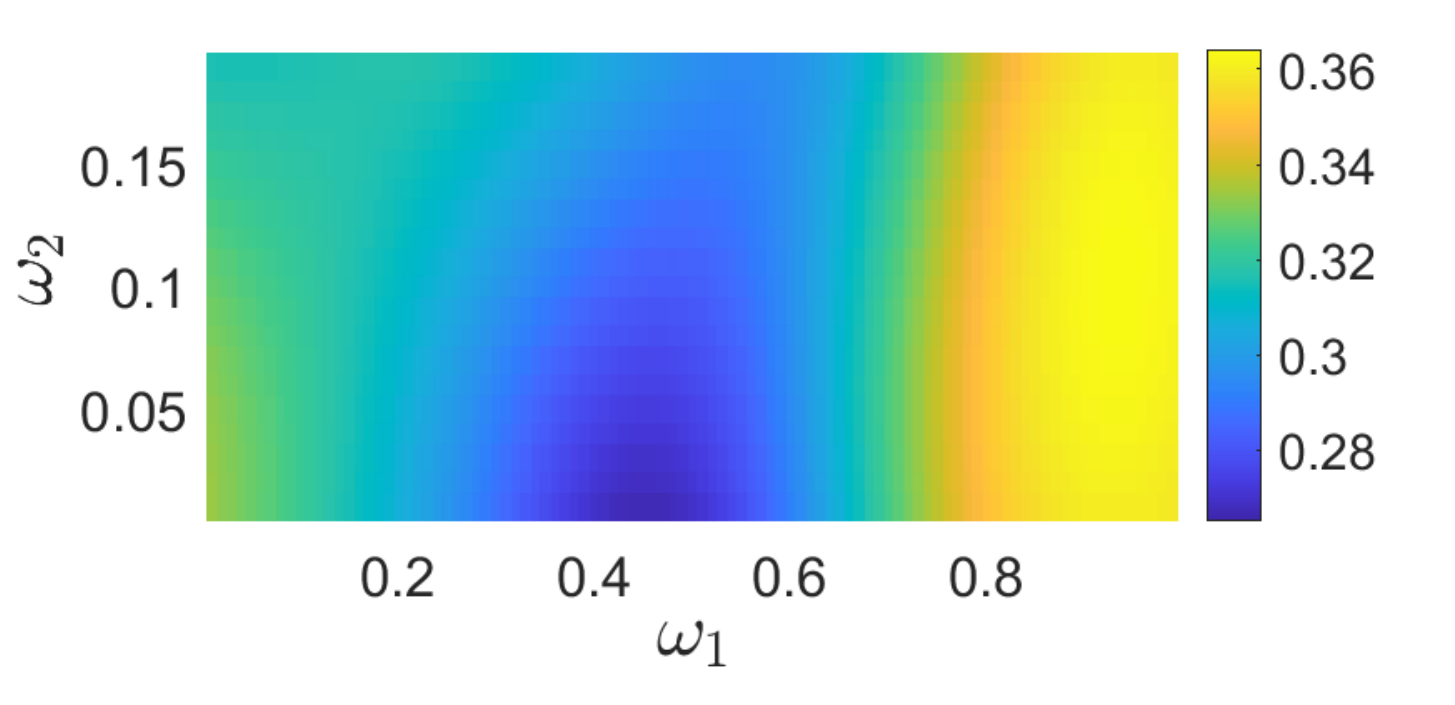}
        \caption{Training set, $(\omega_1,\omega_2)\mapsto \sigma_{G44}(\omega_1,\omega_2,\omega_3)$}
        \label{fig:figure9c}
    \end{subfigure}
    \begin{subfigure}[b]{0.29\textwidth}
    \centering
        \includegraphics[width=\textwidth]{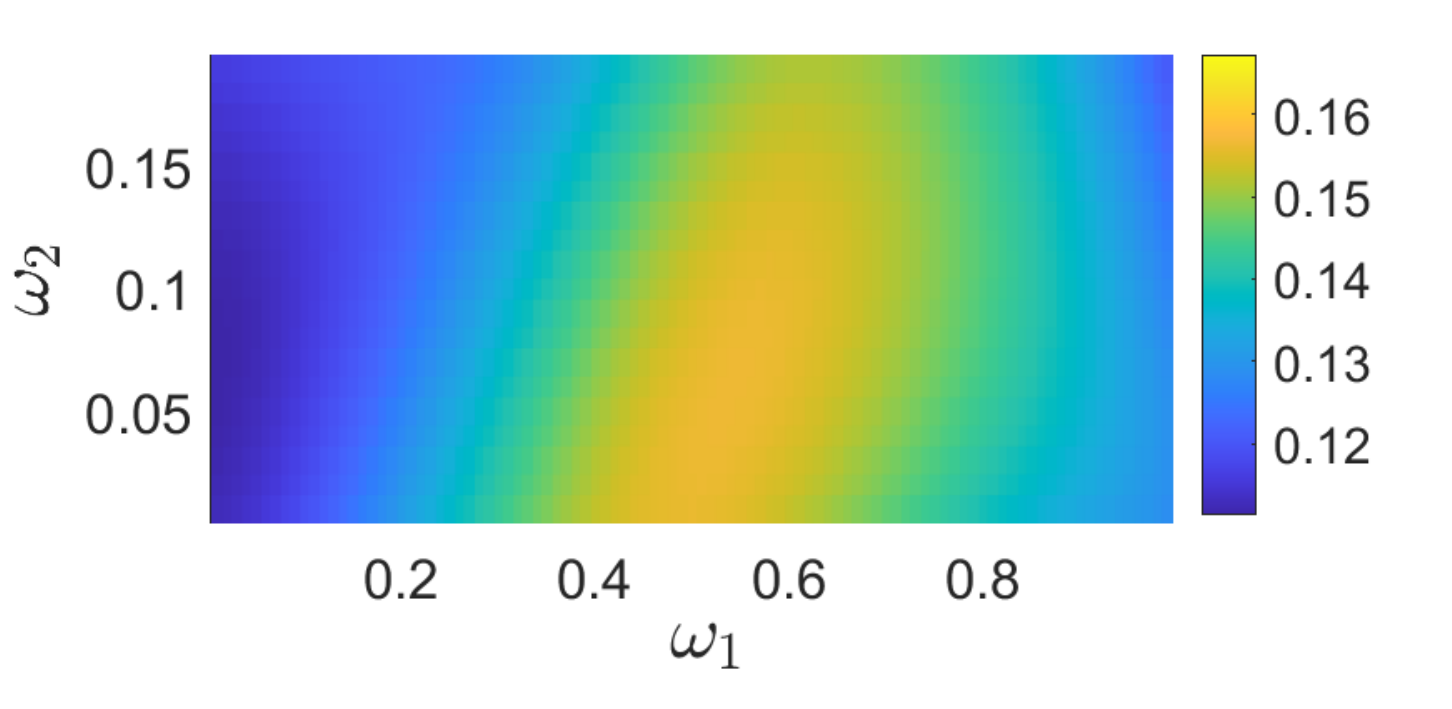}
        \caption{Posterior, $(\omega_1,\omega_2)\mapsto \sigma_{G11}(\omega_1,\omega_2,\omega_3)$}
        \label{fig:figure9d}
    \end{subfigure}
    \begin{subfigure}[b]{0.29\textwidth}
        \centering
        \includegraphics[width=\textwidth]{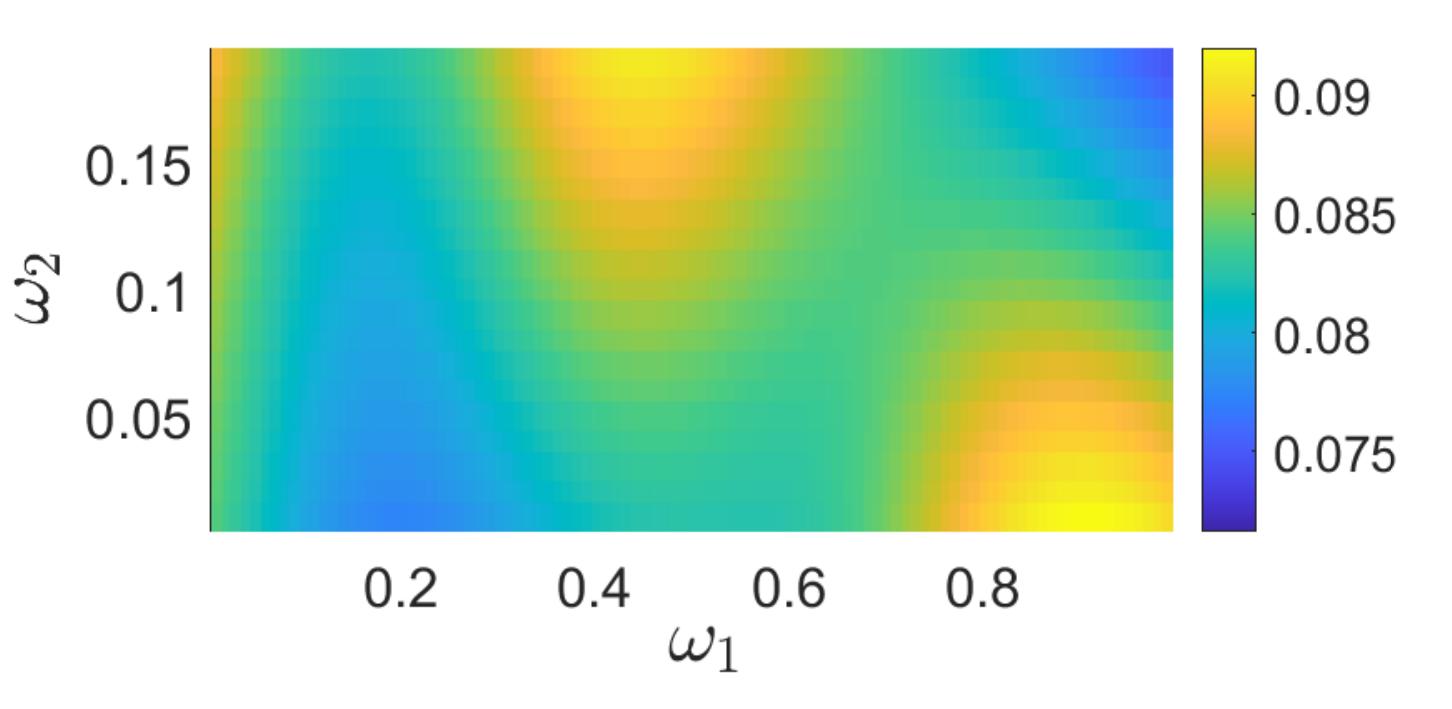}
        \caption{Posterior, $(\omega_1,\omega_2)\mapsto \sigma_{G12}(\omega_1,\omega_2,\omega_3)$}
        \label{fig:figure9e}
    \end{subfigure}
    \begin{subfigure}[b]{0.29\textwidth}
        \centering
        \includegraphics[width=\textwidth]{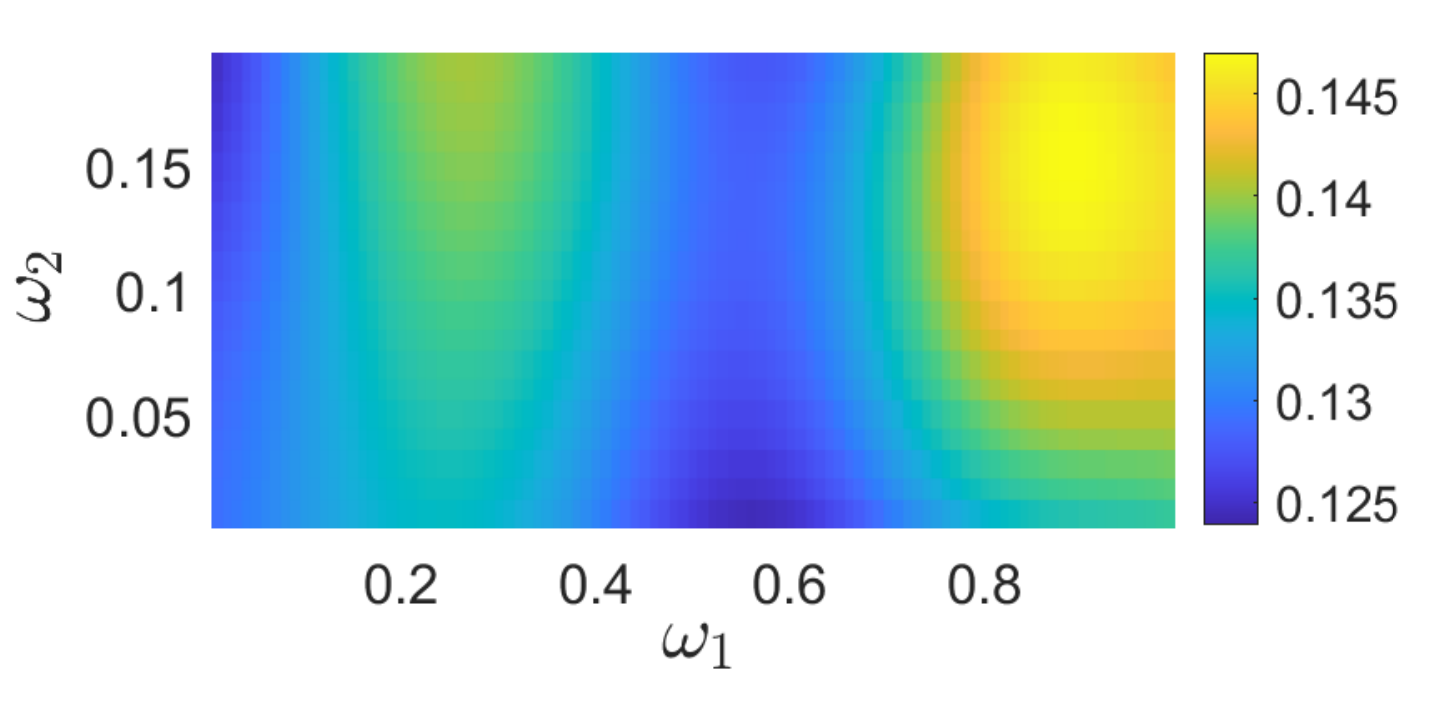}
        \caption{Posterior, $(\omega_1,\omega_2)\mapsto \sigma_{G44}(\omega_1,\omega_2,\omega_3)$}
        \label{fig:figure9f}
    \end{subfigure}
    \begin{subfigure}[b]{0.29\textwidth}
    \centering
        \includegraphics[width=\textwidth]{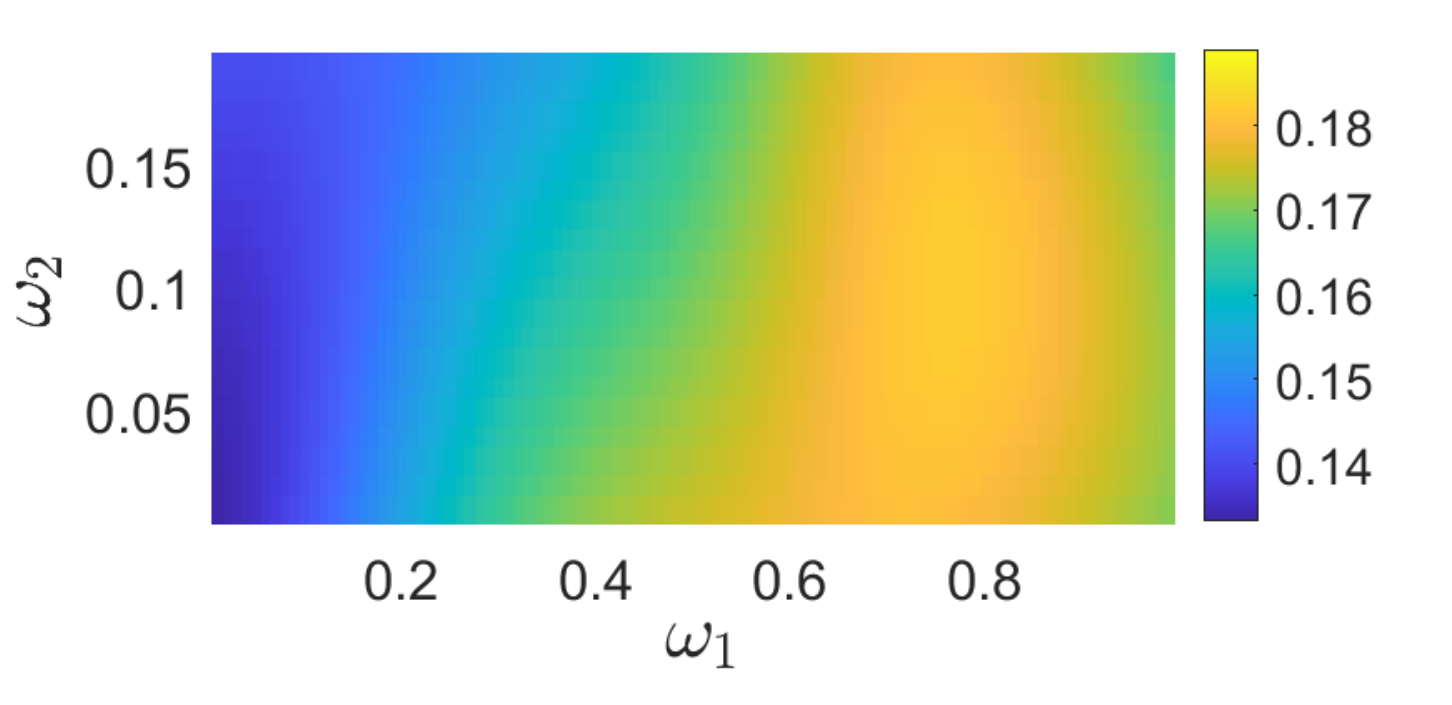}
        \caption{Reference, $(\omega_1,\omega_2)\mapsto \sigma_{G11}(\omega_1,\omega_2,\omega_3)$}
        \label{fig:figure9g}
    \end{subfigure}
    \begin{subfigure}[b]{0.29\textwidth}
        \centering
        \includegraphics[width=\textwidth]{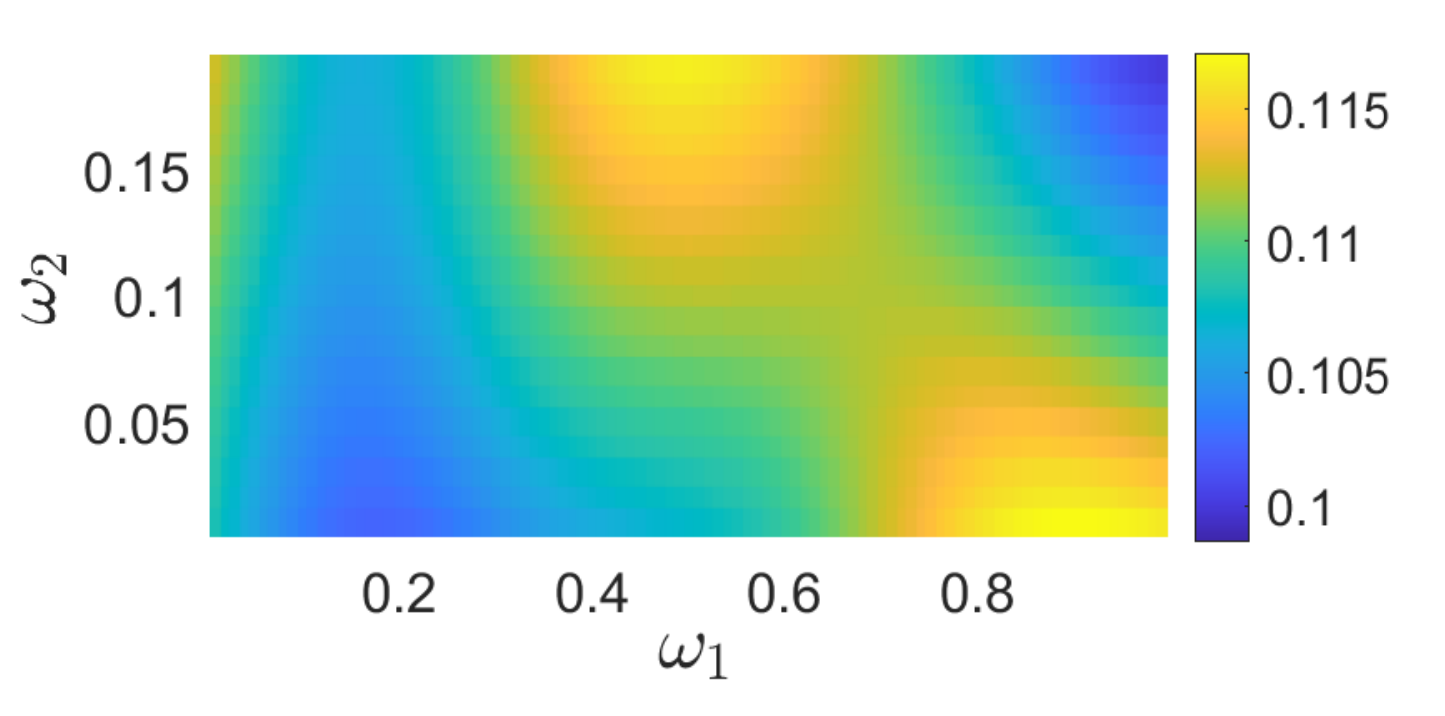}
        \caption{Reference, $(\omega_1,\omega_2)\mapsto \sigma_{G12}(\omega_1,\omega_2,\omega_3)$}
        \label{fig:figure9h}
    \end{subfigure}
    \begin{subfigure}[b]{0.29\textwidth}
        \centering
        \includegraphics[width=\textwidth]{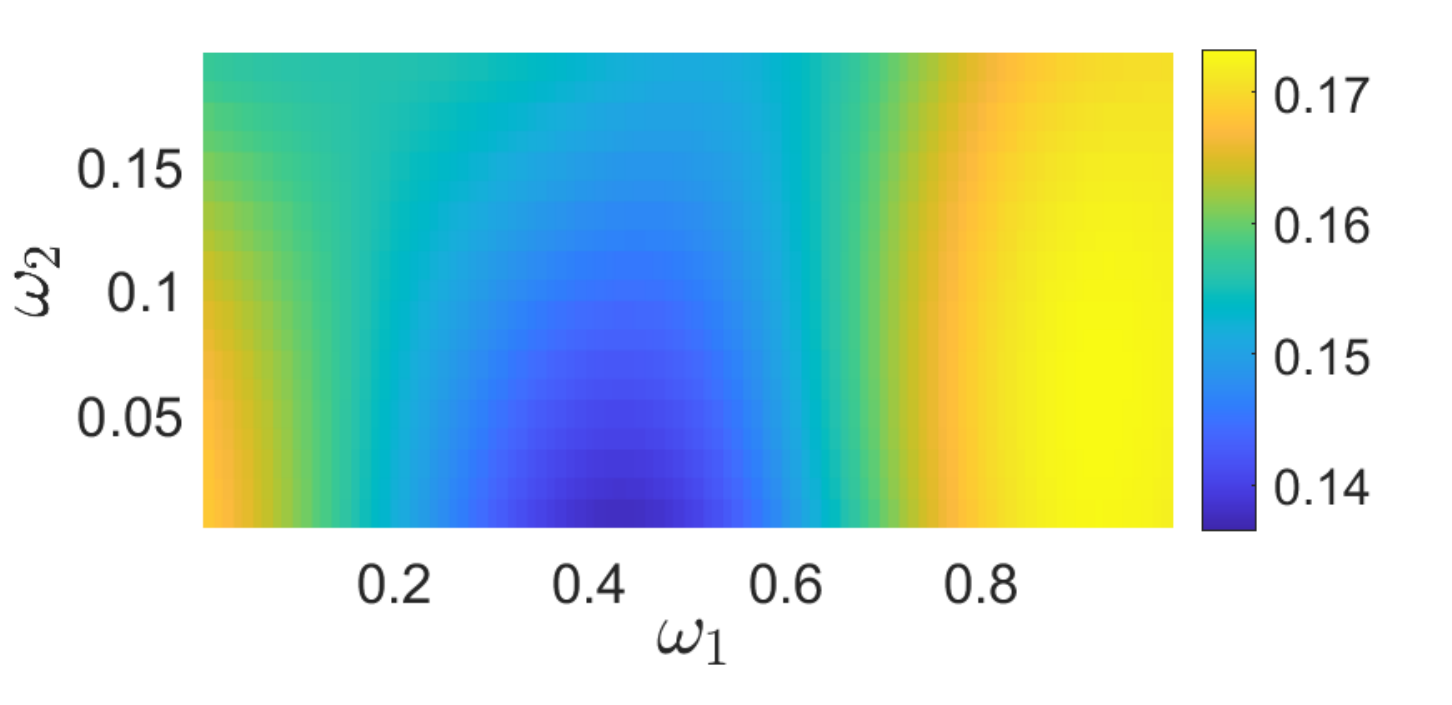}
        \caption{Reference, $(\omega_1,\omega_2)\mapsto \sigma_{G44}(\omega_1,\omega_2,\omega_3)$}
        \label{fig:figure9i}
    \end{subfigure}
    \caption{For the training set with $N_d=100$ (Figs. (a, b, and c), for the posterior estimated with the constrained learned set with $N_d=100$ and $N_r=20$ (Figs. (d, e, and f), and for the reference (Figs. (g,h, and i), standard deviation fields $(\omega_1,\omega_2)\mapsto \sigma_{G11}(\omega_1,\omega_2,\omega_3)$, $\sigma_{G12}(\omega_1,\omega_2,\omega_3)$, and $\sigma_{G44}(\omega_1,\omega_2,\omega_3)$  in the plane $\omega_3 = 0.095774$ of the components $(1,1)$, $(1,2)$, and $(4,4)$ of the random field  $(\omega_1,\omega_2)\mapsto [\bfG(\omega_1,\omega_2,\omega_3)]$.}
    \label{fig:figure9}
\end{figure}
Since there is no target for the control variable, we cannot directly compare $\bfW_\ppost$ (that is to say the random field $\{[\bfG_\ppost(\bfomega)] ,\bfomega\in\Omega\}$) with a target.
In the case of the numerical illustration that we present, the map $\bff$ such that $\bfQ=\bff(\bfW)$ is known numerically, that is, for $\bfw^j_d $ given, we have calculated $\bfq^j_d = \bff(\bfw_d^j)$ with the boundary value problem to generate the training set. This situation is particular and is not that of the general framework that we have given ourselves. Nevertheless, this particular situation allows us to use another method to qualify the quality of the probability measure of $\bfW_\ppost$ estimated with the constrained learned set, as follows.
The proposed constrained-learned-set algorithm allows for computing the realizations $\{\bfeta_\ppost^1,\ldots, \bfeta_\ppost^N\}$ of the posterior random variable  $\bfH_\ppost$, which constitute the points of the constrained learned set.
Using Eq.~\eqref{eq:eq139-6}, the realizations $\bfq_\ppost^1,\ldots ,\bfq_\ppost^{N}$ and $\bfw_\ppost^1,\ldots ,\bfw_\ppost^{N}$ are computed by the equations $\bfQ_\ppost =\underline\bfq + [\Phi_q]\, [\kappa]^{1/2} \, \bfH_\ppost $ and $\bfW_\ppost=\underline\bfw + [\Phi_w]\, [\kappa]^{1/2} \, \bfH_\ppost$. We can then compare $\bfQ_\ppost=\underline\bfq + [\Phi_q]\, [\kappa]^{1/2} \, \bfH_\ppost$
with $\bfQ^\pQA= \bff(\bfW_\ppost)$ in which $\bfW_\ppost=\underline\bfw + [\Phi_w]\, [\kappa]^{1/2} \, \bfH_\ppost$
and where mapping $\bff$ is evaluated with the computational model.

For $N_r=100$, the convergence of $\bfQ^\pQA = \bff(\bfW_\ppost)$ with respect to $N_d$ has been analyzed by studying, for $k=1,2,3$, the  mean-square norm $|||Q^\pQA_{\pobs,k}||| = \{ E\{ (Q^\pQA_{\pobs,k})^2\}\}^{1/2}$ of random component $Q^\pQA_{\pobs,k}$ of $\bfQ^\pQA$ (which depends on $N_d$). For $k=1,2,3$, Table~\ref{table:table1} yields the values of $|||Q^\pQA_{\pobs,k}|||$. The expected convergence can be viewed with respect to $N_d$ (this result is consistent with the convergence of the pdf's shown in Fig.~\ref{fig:figure10}).
\begin{table}[h]
  \caption{For $N_r=100$, convergence of the mean-square norm of $|||Q^\pQA_{\pobs,1}|||$, $|||Q^\pQA_{\pobs,2}|||$, and
    $|||Q^\pQA_{\pobs,3}|||$ as a function of $N_d$.}\label{table:table1}
\begin{center}
 \begin{tabular}{|c|c|c|c|c|} \hline
    $N_d$                    &   100                & 200                  & 300                  & 400                 \\
    \hline
    $|||Q^\pQA_{\pobs,1}|||$ & $7.59\times 10^{-3}$ & $8.15\times 10^{-3}$ & $8.12\times 10^{-3}$ & $8.20\times 10^{-3}$\\
    $|||Q^\pQA_{\pobs,2}|||$ & $2.44\times 10^{-2}$ & $2.60\times 10^{-2}$ & $2.63\times 10^{-2}$ & $2.67\times 10^{-2}$\\
    $|||Q^\pQA_{\pobs,3}|||$ & $5.11\times 10^{-3}$ & $5.46\times 10^{-3}$ & $6.35\times 10^{-3}$ & $5.56\times 10^{-3}$\\
    \hline
  \end{tabular}
\end{center}
\end{table}
\begin{figure}[!h]
    \centering
    \begin{subfigure}[b]{0.25\textwidth}
    \centering
        \includegraphics[width=\textwidth]{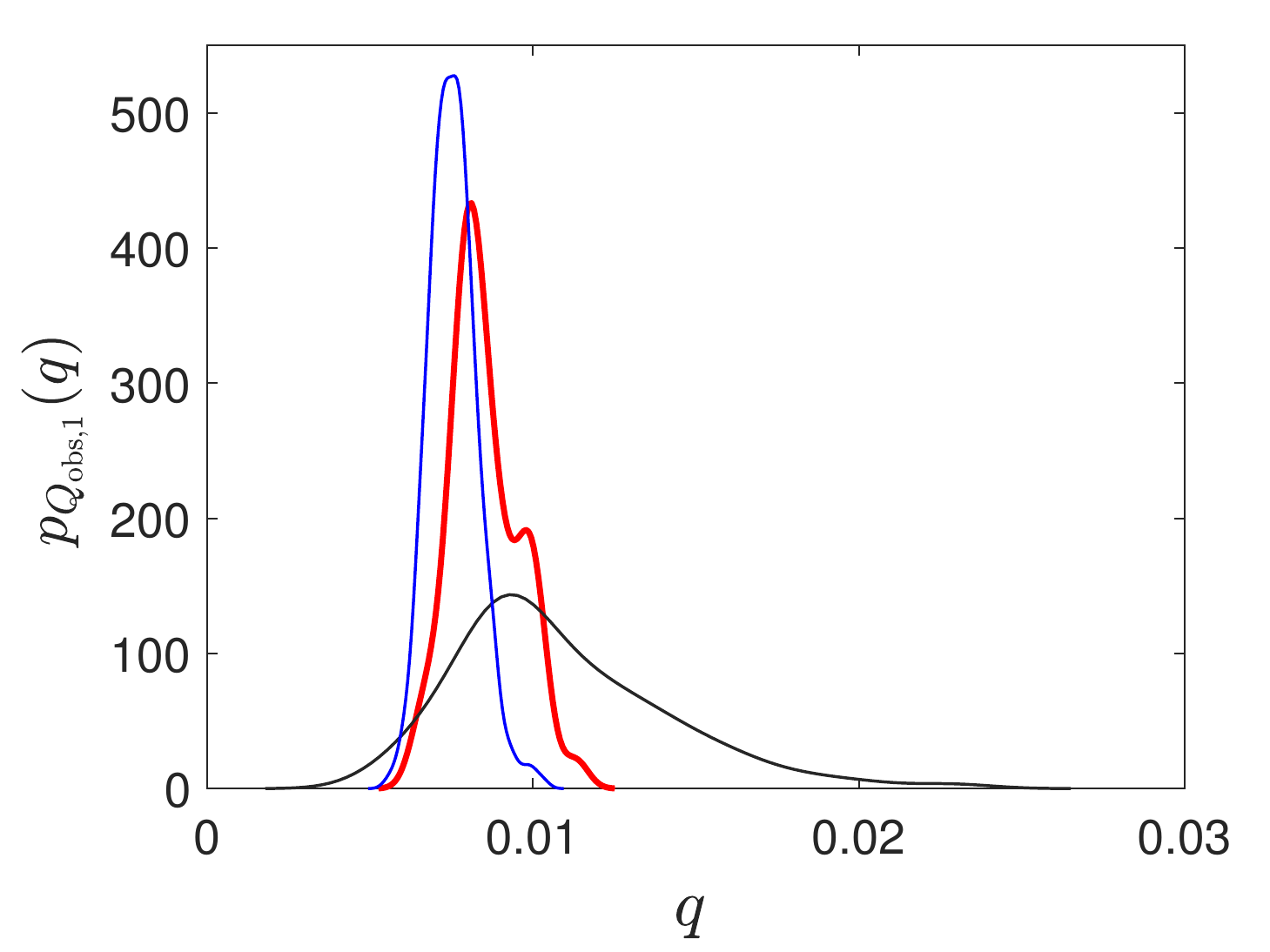}
        \caption{pdf $q\mapsto p_{Q_{\ppobs,1}}(q)$ for $N_d=100$}
        \label{fig:figure10a}
    \end{subfigure}
    \begin{subfigure}[b]{0.25\textwidth}
        \centering
        \includegraphics[width=\textwidth]{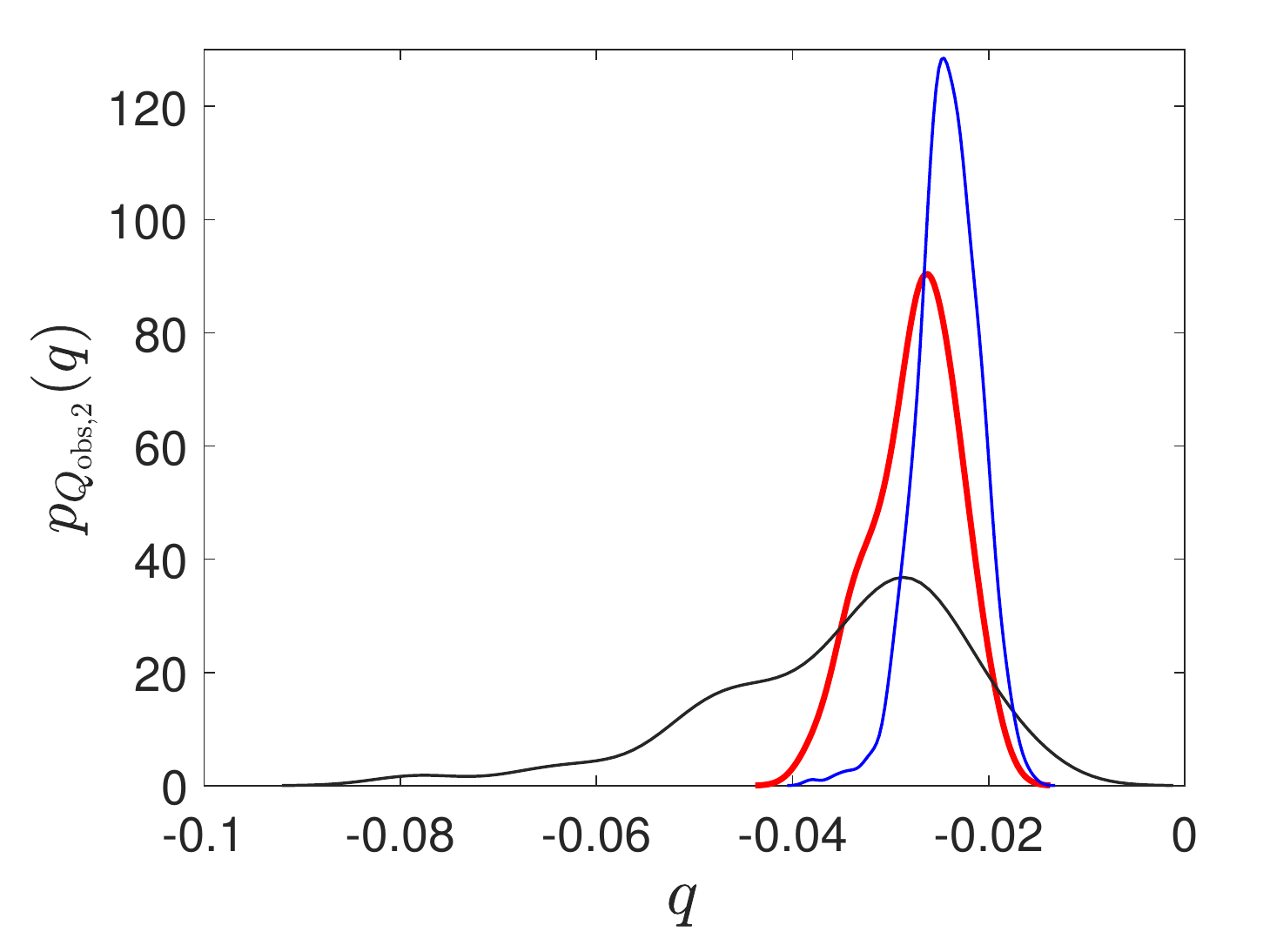}
        \caption{pdf $q\mapsto p_{Q_{\ppobs,2}}(q)$ for $N_d=100$}
        \label{fig:figure10b}
    \end{subfigure}
    \begin{subfigure}[b]{0.25\textwidth}
        \centering
        \includegraphics[width=\textwidth]{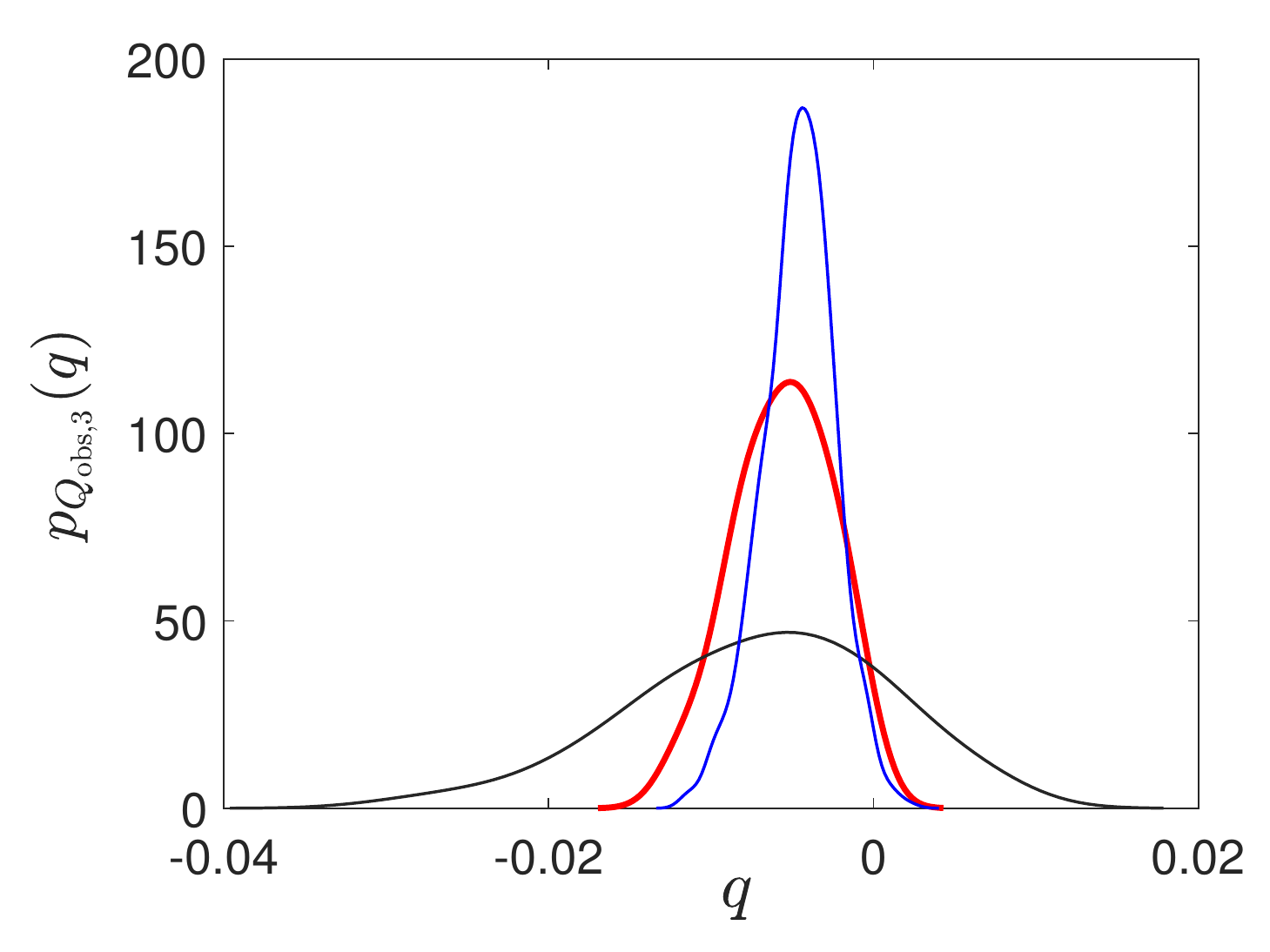}
        \caption{pdf $q\mapsto p_{Q_{\ppobs,3}}(q)$ for $N_d=100$}
        \label{fig:figure10c}
    \end{subfigure}
    \centering
    \begin{subfigure}[b]{0.25\textwidth}
    \centering
        \includegraphics[width=\textwidth]{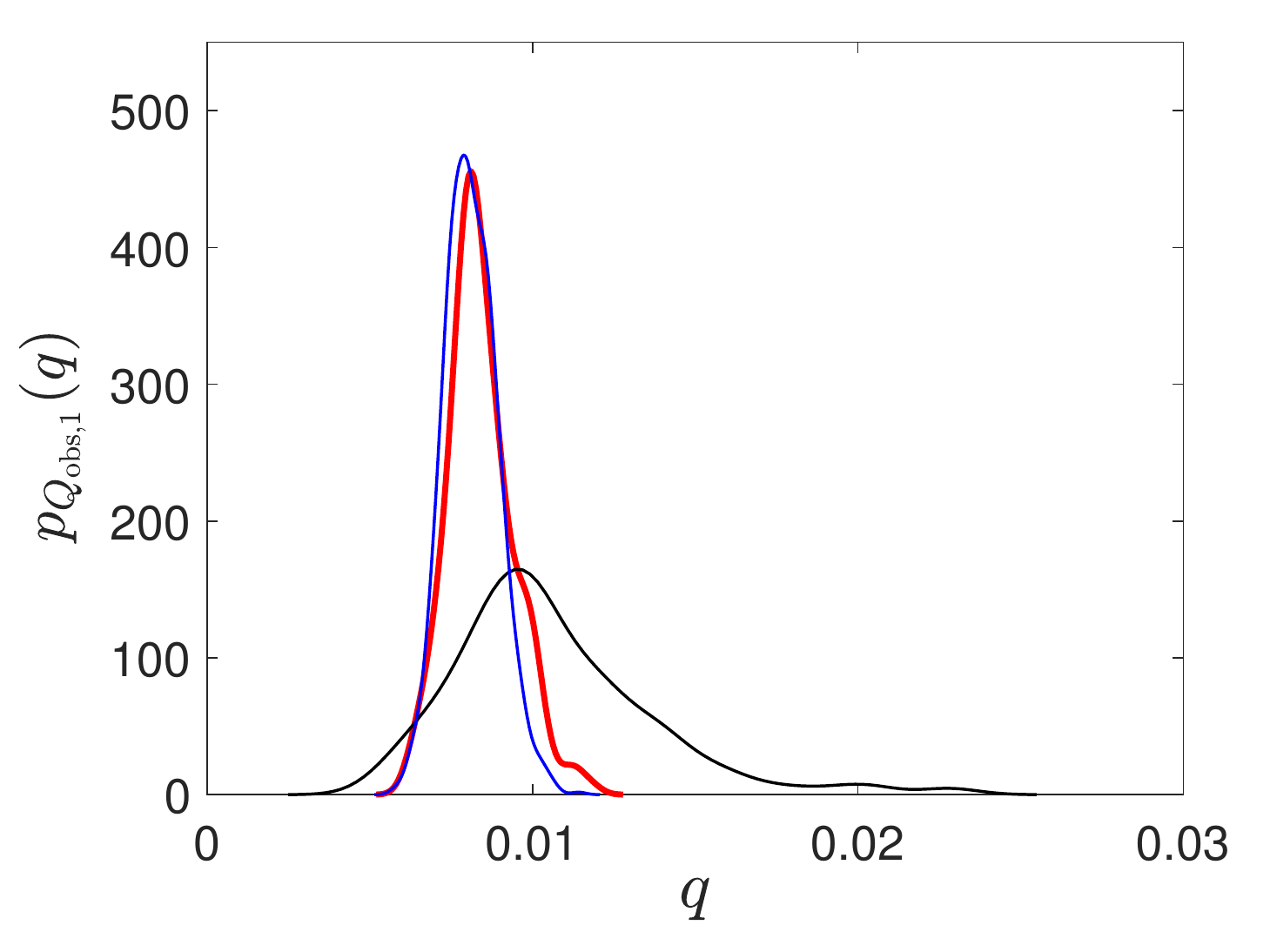}
        \caption{pdf $q\mapsto p_{Q_{\ppobs,1}}(q)$ for $N_d=200$}
        \label{fig:figure10d}
    \end{subfigure}
    \begin{subfigure}[b]{0.25\textwidth}
        \centering
        \includegraphics[width=\textwidth]{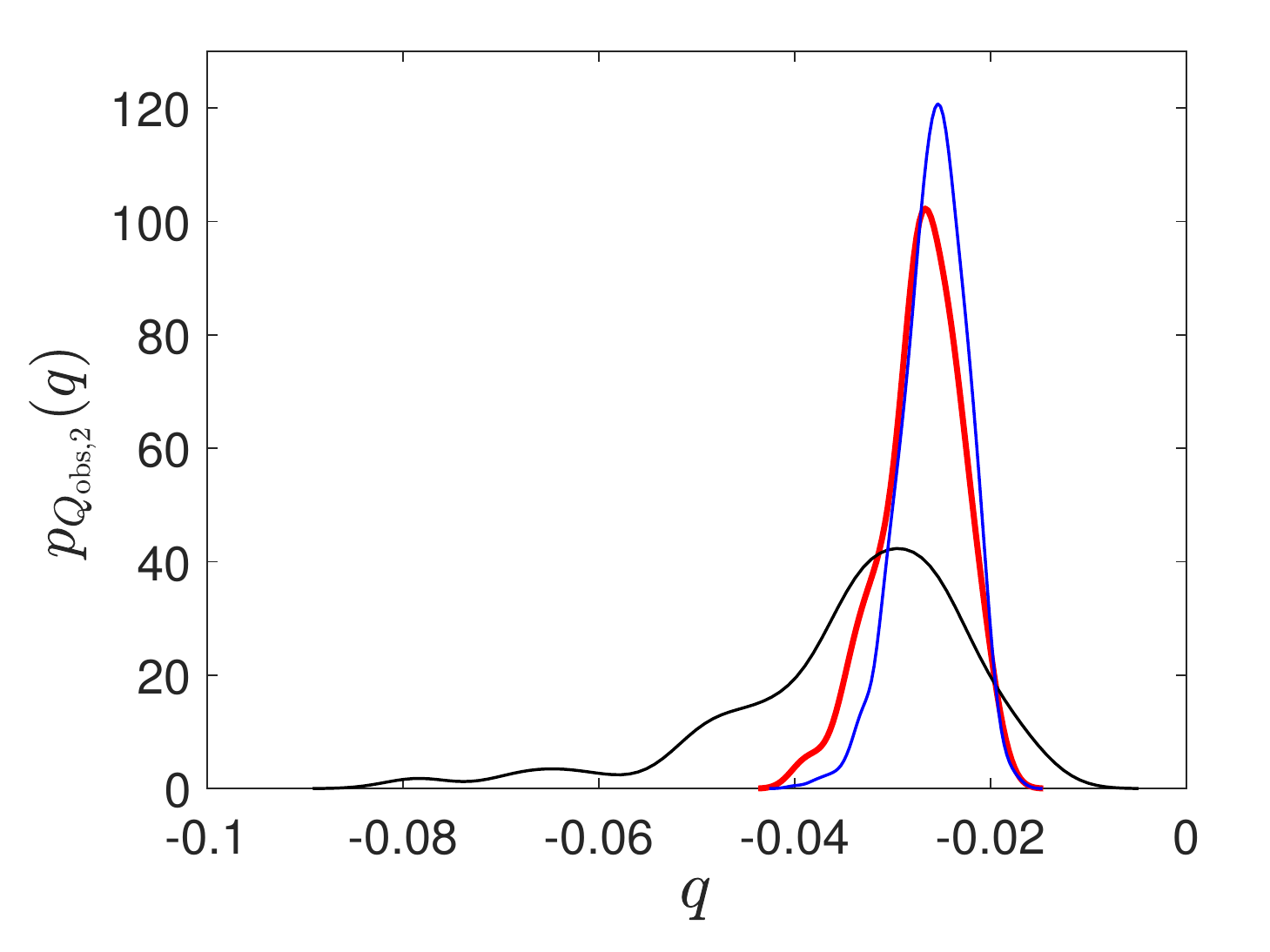}
        \caption{pdf $q\mapsto p_{Q_{\ppobs,2}}(q)$ for $N_d=200$}
        \label{fig:figure10e}
    \end{subfigure}
    \begin{subfigure}[b]{0.25\textwidth}
        \centering
        \includegraphics[width=\textwidth]{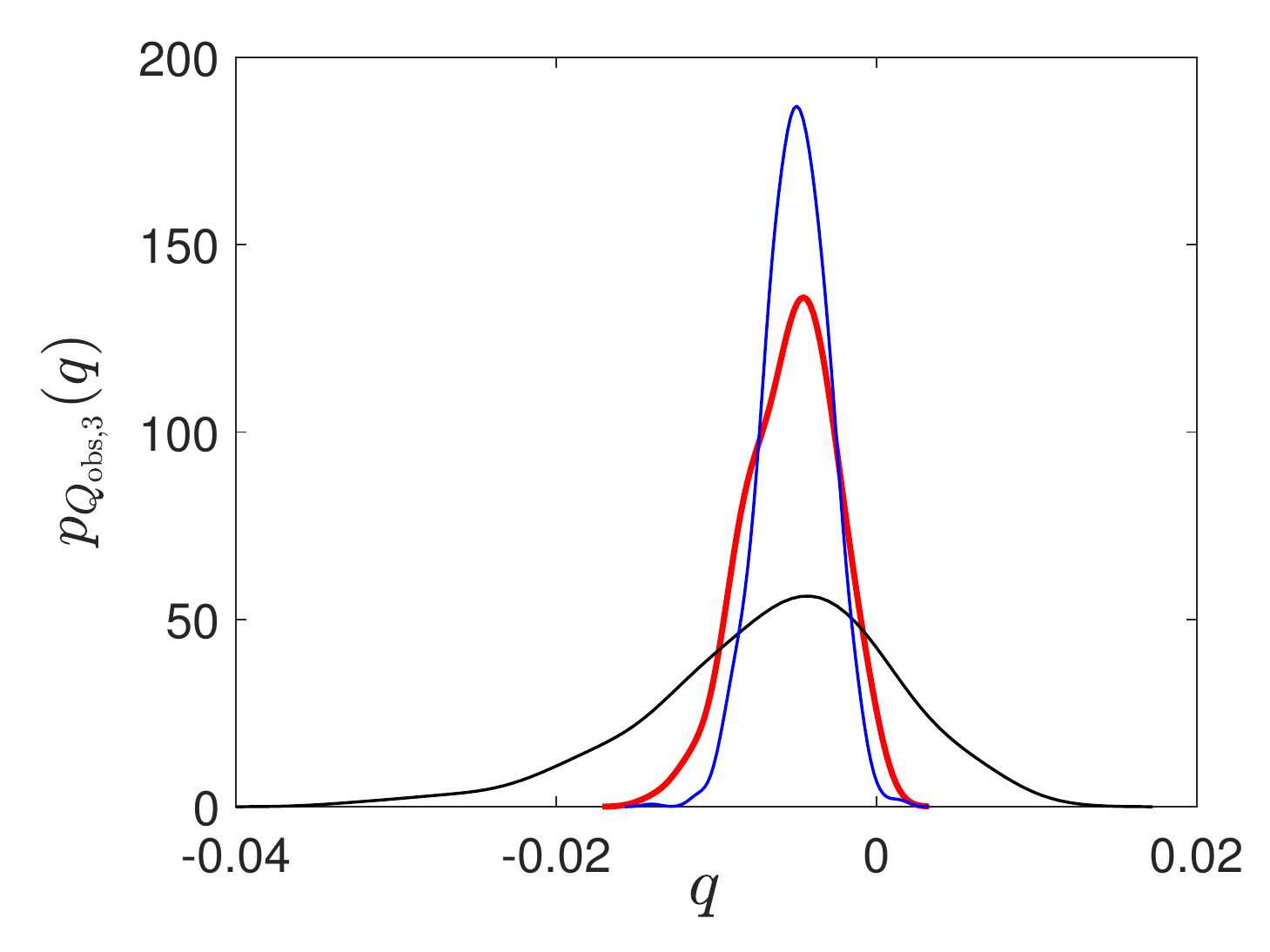}
        \caption{pdf $q\mapsto p_{Q_{\ppobs,3}}(q)$ for $N_d=200$}
        \label{fig:figure10f}
    \end{subfigure}
    \begin{subfigure}[b]{0.25\textwidth}
    \centering
        \includegraphics[width=\textwidth]{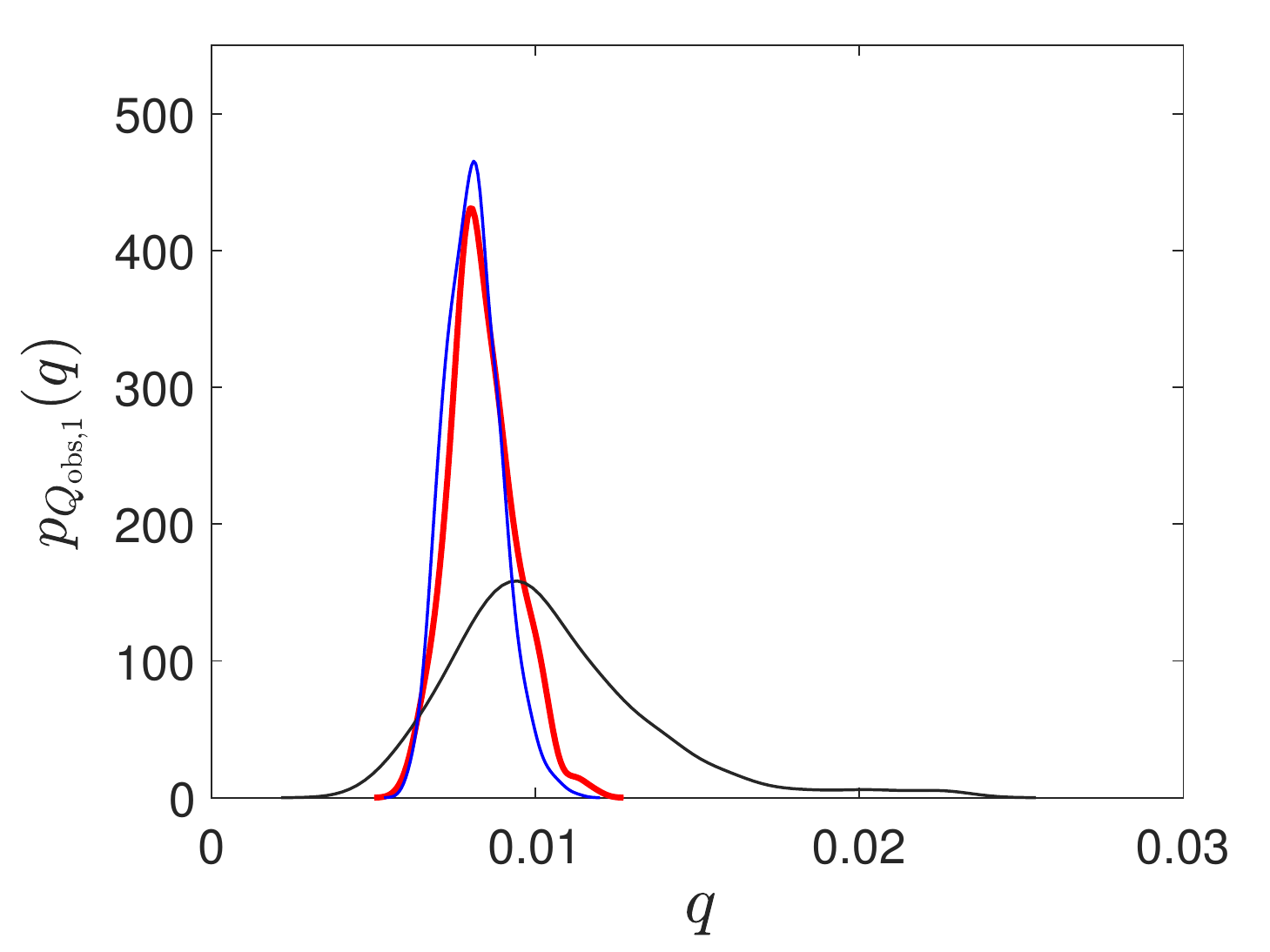}
        \caption{pdf $q\mapsto p_{Q_{\ppobs,1}}(q)$ for $N_d=300$}
        \label{fig:figure10g}
    \end{subfigure}
    \begin{subfigure}[b]{0.25\textwidth}
        \centering
        \includegraphics[width=\textwidth]{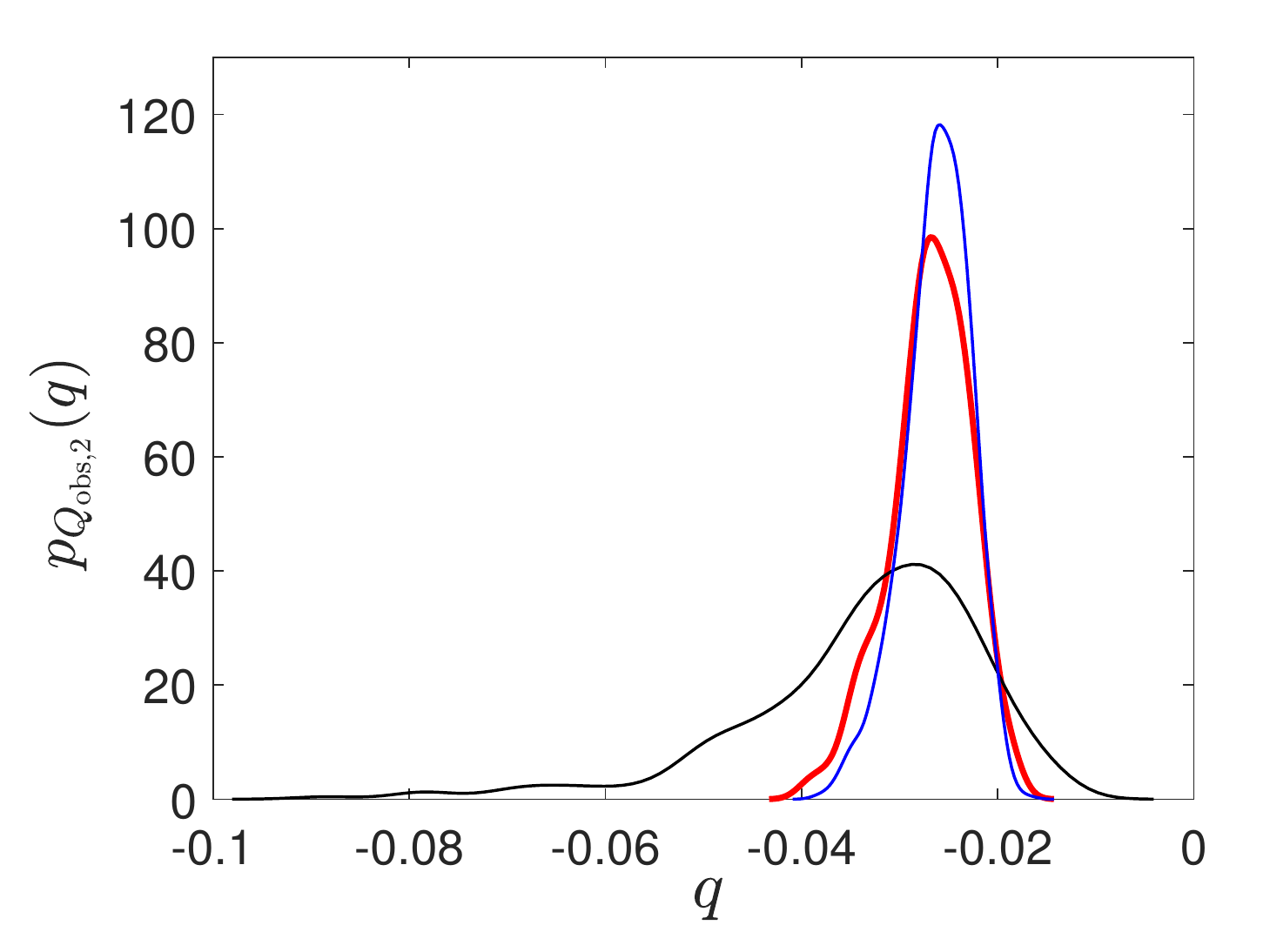}
        \caption{pdf $q\mapsto p_{Q_{\ppobs,2}}(q)$ for $N_d=300$}
        \label{fig:figure10h}
    \end{subfigure}
    \begin{subfigure}[b]{0.25\textwidth}
        \centering
        \includegraphics[width=\textwidth]{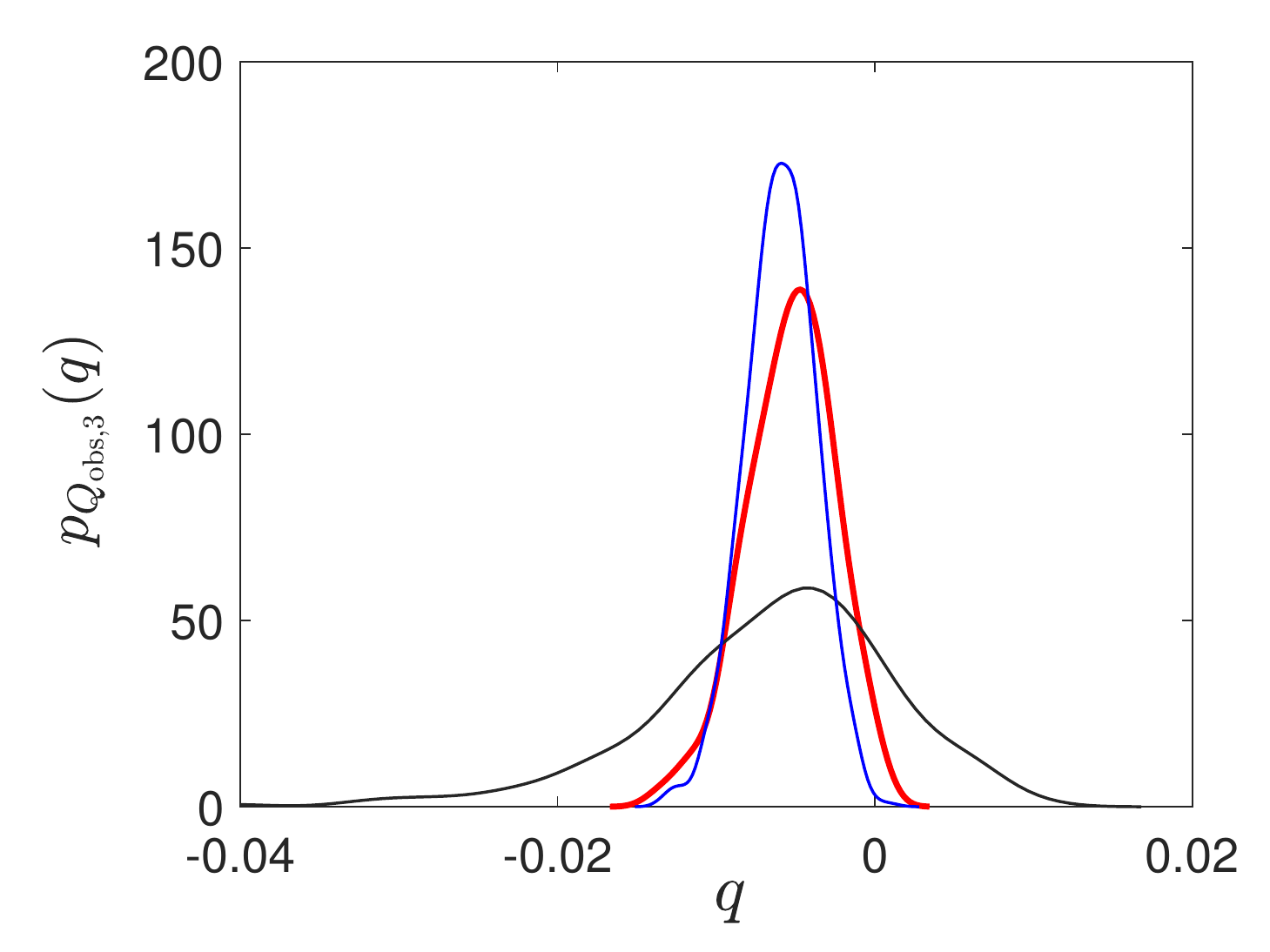}
        \caption{pdf $q\mapsto p_{Q_{\ppobs,3}}(q)$ for $N_d=300$}
        \label{fig:figure10i}
    \end{subfigure}
    \centering
    \begin{subfigure}[b]{0.25\textwidth}
    \centering
        \includegraphics[width=\textwidth]{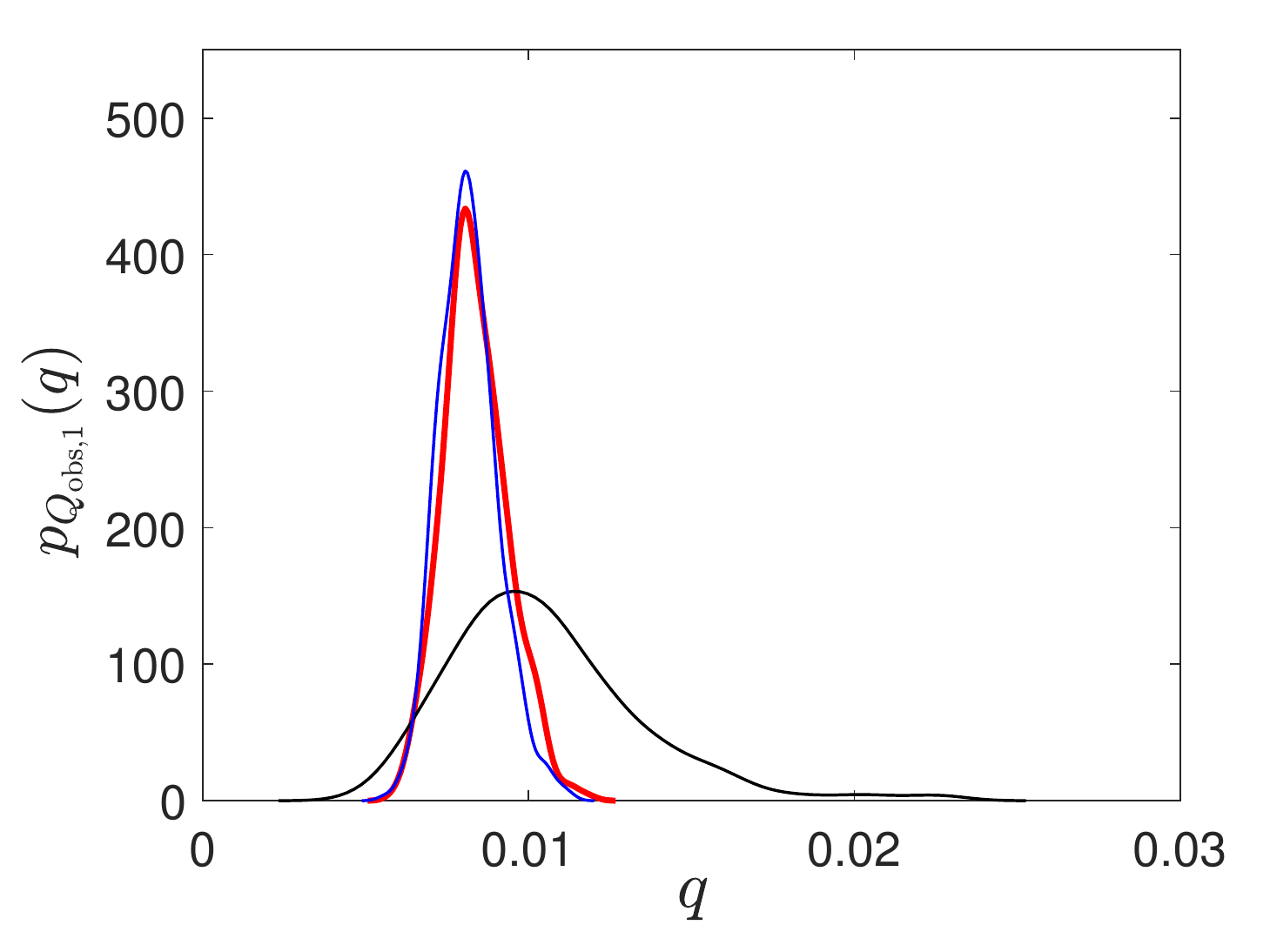}
        \caption{pdf $q\mapsto p_{Q_{\ppobs,1}}(q)$ for $N_d=400$}
        \label{fig:figure10j}
    \end{subfigure}
    \begin{subfigure}[b]{0.25\textwidth}
        \centering
        \includegraphics[width=\textwidth]{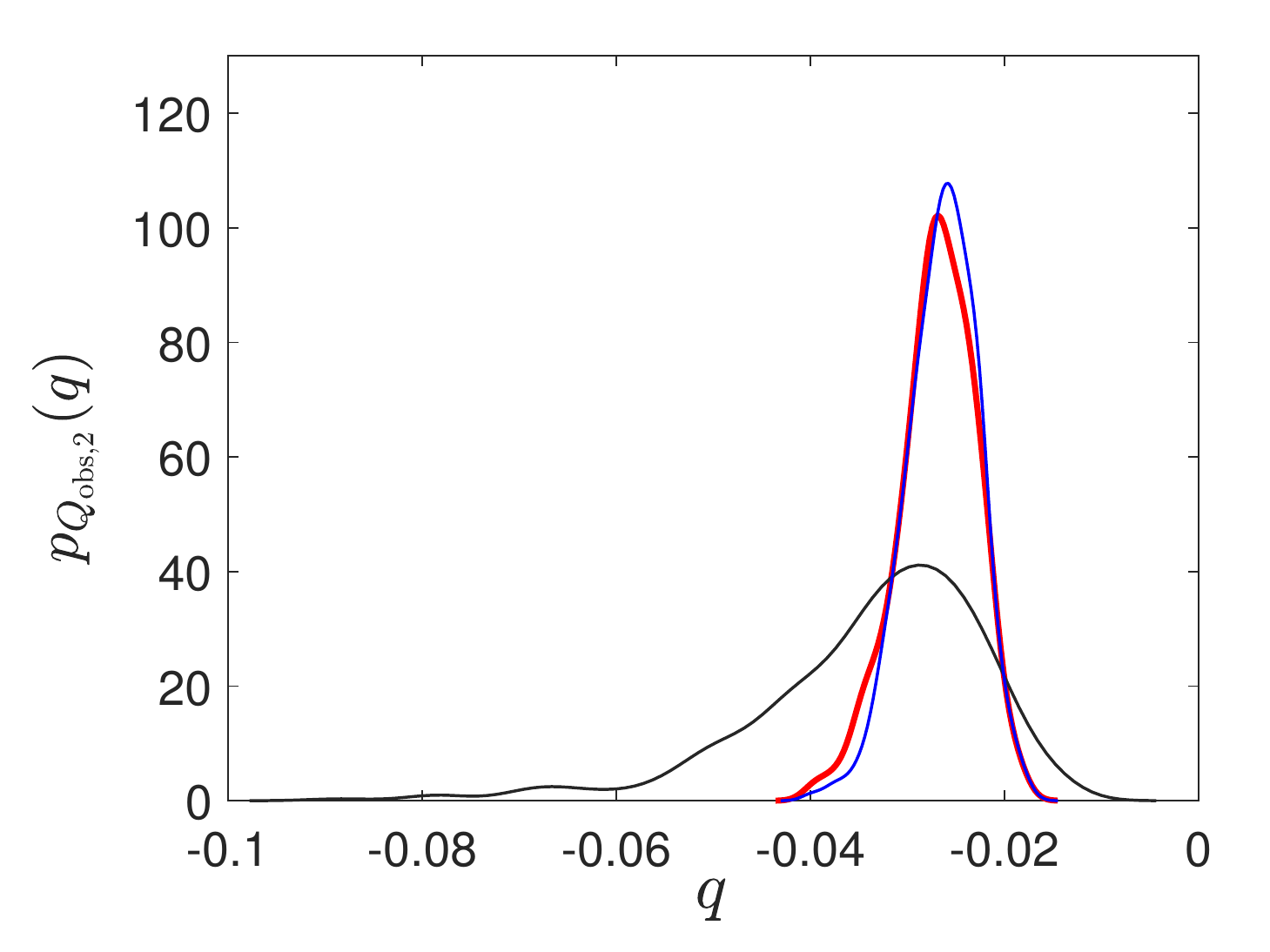}
        \caption{pdf $q\mapsto p_{Q_{\ppobs,2}}(q)$ for $N_d=400$}
        \label{fig:figure10k}
    \end{subfigure}
    \begin{subfigure}[b]{0.25\textwidth}
        \centering
        \includegraphics[width=\textwidth]{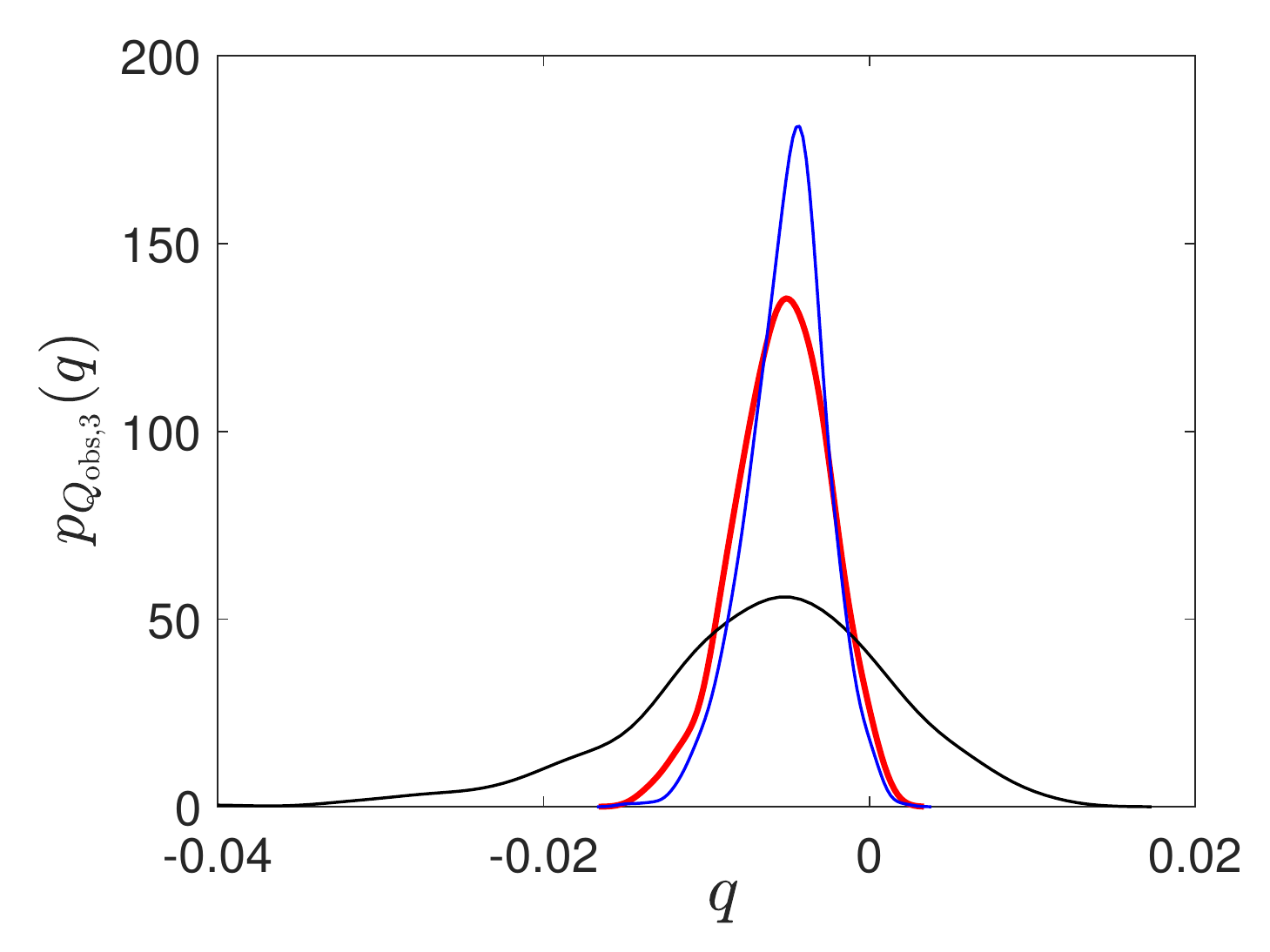}
        \caption{pdf $q\mapsto p_{Q_{\ppobs,3}}(q)$ for $N_d=400$}
        \label{fig:figure10l}
    \end{subfigure}
    \caption{For $N_r=100$, convergence analysis in $N_d$,
         for $N_d=100$ (Figs. (a), (b), and (c)),
         for $N_d=200$ (Figs. (d), (e), and (f)),
         for $N_d=300$ (Figs. (g), (h), and (i)),
     and for $N_d=400$ (Figs. (j), (k), and (l)),
     probability density functions of the random variables $Q_{\pobs,1}$, $Q_{\pobs,2}$, and $Q_{\pobs,3}$,
     for $\bfQ$ estimated with the training set (black line), for $\bfQ^\pQA = \bff(\bfW_\ppost)$ estimated with the posterior learned set of $\bfW_\ppost$ (blue line), and for $\bfQ_\pexp$ corresponding to the reference (red thick line).}
    \label{fig:figure10}
\end{figure}
In addition to the convergence analysis of the mean-square norm  with respect to $N_d$,  Fig.~\ref{fig:figure10} displays the probability density functions of the random variables $Q_{\pobs,1}$, $Q_{\pobs,2}$, and $Q_{\pobs,3}$,
for $\bfQ$ estimated with the training set, for $\bfQ^\pQA = \bff(\bfW_\ppost)$ estimated with $\bfW_\ppost$, and for $\bfQ_\pexp$ corresponding to the reference. Similarly to the convergence of the mean-square norm, this figure shows a clear convergence with respect to $N_d$.
For $N_d=100$ and $N_r=20$ or even $N_r=100$, compared to the reference, the posterior $\bfQ^\pQA$ evaluated with $\bfW_\ppost$ thanks to the knowledge of $\bff$, is less good than the prediction of the posterior $\bfQ_\ppost$.
This is mainly due to the use of the reduced representation for a problem in high dimension.
Fig.~\ref{fig:figure10} shows that the prediction can be improved by increasing the value of $N_d$, that is to say, by increasing the value of $\nu$, which requires to increase the number of points in the training set and consequently, which can induce potential difficulties if the numerical cost for constructing each point of the training set is high.
However, the presented numerical illustration shows that the proposed method allows for integrating,with a good quality, a target set of realizations (i.e. data) in a supervised model, which is defined only by a small number of points in  a training set and for which the target  set of realizations are specified only for the quantities of interest (output) and not for the control variable (input).
Finally, it should be noted that, when the training set is generated using a stochastic boundary value problem, there is also another method as we have proposed and validated in \cite{Soize2021a,Soize2022bb}. It consists, for the generation of the constrained learned set,  to introduce an additional scalar constraint to minimize the norm of the residue of the partial differential equations of the boundary value problem. This procedure can be implemented without difficulty in the methodology presented in this paper, involving only one additional component in the vector-valued function $\bfh^C$ and the vector $\bfb^c$.
\section{Conclusion}
\label{sec:Section8}
In this paper, we have presented a novel functional approach that makes it possible to take into account a target set of realizations
in the Kullback-Leibler minimum principle for constructing a posterior probability measure from a prior probability measure defined by a given training set of realizations. This approach thus allows for integrating a target set of realizations in a supervised model, which is defined only by a small number of points in a training set.
It consists in constructing and analyzing a weak formulation of the Fourier transform of the probability measure (characteristic function) of the observed quantities of interest and to derive from it a finite representation of the functional constraint. On the basis of the positive Hermitian form associated with the Fourier transform of the probability measure, we have constructed and analyzed the properties of a functional family of functions, which only depends on the target set of the given realizations. These properties have allowed us to show the existence and the uniqueness of the posterior probability measure constructed by using  the Kullback-Leibler minimum principle. The numerical aspects have been detailed in order to facilitate the implementation of the algorithms. The presented numerical illustration that is in high dimension demonstrates the efficiency and the robustness of the proposed method.
\appendix
\section{Generation of the training set, target set, and numerical values of the parameters}
\label{appendix:A}
The training set $D_d = \{ \bfx^1,\ldots , \bfx^{N_d}\}$ with $\bfx^j\!\! = \!(\bfq_d^j,\bfw_d^j) \in \RRnx \!\! = \!\RRnq\!\times\!\RRnw$
is made up of $N_d$ independent realizations of random variable $\bfX=(\bfQ,\bfW)$, which are generated by using a stochastic computational model corresponding to the finite element discre\-ti\-za\-tion of a stochastic elliptic boundary value problem for which $n_x=430\, 098$, $n_q=10\, 098$, and $n_w= 420\, 000$. The target set $D_\pexp = \{ \bfq^1_\pexp,\ldots , \bfq^{N_r}_\pexp\}$ is generated using the stochastic computational model with another values of the parameters (see \ref{appendix:A.3}).
\subsection{Definition of the stochastic boundary value problem}
\label{appendix:A.1}
Let $\Omega = ]\,0\, , 1\,[\, \times\, ]\,0\, , 0.2\,[\,\times \,]\,0 \, , 0.1\,[ \, m^3$ be the bounded open set of $\RR^3$, with generic point $\bfomega = (\omega_1,\omega_2,\omega_3)$, and with boundary
$\partial\Omega=\Gamma_0 \cup\Gamma_1\cup\Gamma_2$ in which
$\Gamma_0 =\{\omega_1=1\, , \, 0\leq \omega_2 \leq 0.2 \, , \, 0 \leq \omega_3 \leq 0.1 \}$,
$\Gamma_1 =\{\omega_1=0\, , \, 0\leq \omega_2 \leq 0.2 \, , \, 0 \leq \omega_3 \leq 0.1 \}$,
 and
$\Gamma_2 =\partial\Omega\backslash \{\Gamma_0\cup\Gamma_1 \}$. Let be $\overline\Omega=\Omega\cup\partial\Omega$.
The outward unit normal to $\partial\Omega$ is denoted by $\nn = (\nn_1,\nn_2,\nn_3)$. We use the usual convention of summation on repeated Latin indices. Domain $\Omega$ is occupied by a  heterogeneous and aniso\-tropic elastic random medium for which the elastic properties are defined by the fourth-order tensor-valued non-Gaussian random field $\AAeclair = \bigg \{ \{\AAeclair_{k  m n q}(\bfomega)\}_{k  m n q}, \bfomega\in\Omega\bigg\}$.
Let $\UU=(\UU_1, \UU_2, \UU_3)$ be the $\RR^3$-valued displacement random field defined in $\Omega$.
A Dirichlet condition $\UU=\bfzero$ is given on $\Gamma_0$ while a Neumann condition is given on $\Gamma_1\cup\Gamma_2$.
The stochastic boundary value problem is written, for $k=1,2,3$  and almost surely, as
\begin{align}
- \frac{\partial\sigmasigma_{k m}}{\partial\omega_m} & = 0 \,\,\,\, \hbox{in} \,\,\, \Omega \, , \label{eq:eq140} \\
\UU_k & = 0                      \,\,\,\,\hbox{on}   \,\,\, \Gamma_0                         \, , \label{eq:eq141} \\
\sigmasigma_{km}\, \nn_m & = p_k  \,\,  \hbox{on}  \,\,\, \Gamma_1                         \, , \label{eq:eq142} \\
\sigmasigma_{km}\, \nn_m & = 0    \,\,\,\,  \hbox{on}  \,\,\, \Gamma_2                       \, , \label{eq:eq143}
\end{align}
in which the stress tensor $\sigmasigma$ is related to the strain tensor $\epsilonepsilon$ by
$\epsilonepsilon_{nq} = (\partial\UU_n/\partial\omega_q + \partial\UU_q/\partial\omega_n)/2$ by the constitutive equation,
$\sigmasigma_{km}(\bfomega)= \AAeclair_{kmnq}(\bfomega) \, \epsilonepsilon_{nq}(\UU(\bfomega))$.
For $k=1,2,3$, the applied stresses $p_k$ on $\Gamma_1$ are defined as follows:

\noindent $p_1 = 0$ on $\Gamma_1$, except:

$\quad p_1 = -1.8\times 10^8\, N/m^2$ for $\bfomega\in\{\omega_1=0\, , \, 0\leq \omega_2\leq 0.02 \, , \, 0\leq \omega_3 \leq 0.1\}$.

$\quad p_1 = +9.0\times 10^7\, N/m^2$ for $\bfomega\in\{\omega_1=0\, , \, 0.18\leq \omega_2\leq 0.2 \, , \, 0\leq \omega_3 \leq 0.1\}$.

\noindent $p_2 = 0$ on $\Gamma_1$, except:

$\quad p_2 = +1.0\times 10^7\, N/m^2$ for $\bfomega\in\{\omega_1=0\, , \, \{0\leq \omega_2\leq 0.02\}\cup \{0.18\leq \omega_2\leq 0.20\}
                                 \, , \, 0\leq \omega_3 \leq 0.02\}$.

$\quad p_2 = -1.5\times 10^7\, N/m^2$ for $\bfomega\in\{\omega_1=0\, , \, \{0\leq \omega_2\leq 0.02\}\cup \{0.18\leq \omega_2\leq 0.20\}
                                 \, , \, 0.08\leq \omega_3 \leq 0.1\}$.

\noindent $p_3 = 0$ on $\Gamma_1$, except:

$\quad p_3 = -2.40\times 10^7\, N/m^2$ for $\bfomega\in\{\omega_1=0\, , \, 0\leq \omega_2\leq 0.02 \, , \, 0\leq \omega_3 \leq 0.1\}$.

$\quad p_3 = +2.64\times 10^7\, N/m^2$ for $\bfomega\in\{\omega_1=0\, , \, 0.18\leq \omega_2\leq 0.2 \, , \, 0\leq \omega_3 \leq 0.1\}$.

\noindent Using the matrix representation in Voigt notation, the random elasticity field is rewritten, for $k$, $m$, $n$, and $q$ in $\{1,2,3\}$, as
$[\bfA(\bfomega)]_{\pbfi\pbfj} =\AAeclair_{kmnq}(\bfomega)$ with $\bfi=(k,m)$ with $1\leq k\leq m\leq 3$ and $\bfj=(n,q)$ with $1\leq n \leq q \leq 3$ in which indices  $\bfi$ and $\bfj$ belong to  $\{1,\ldots , 6\}$. The  $\MM_6^+$-valued random field
$\{ [\bfA(\bfomega)] ,\bfomega \in \Omega\}$ is a non-Gaussian, second order, and  statistically homogeneous.
Its mean function is the given $\bfomega$-independent matrix $[\,\underline \bfA\,] =E\{[\bfA(\bfomega)] \}\in\MM_6^+$ corresponding to  a homogeneous isotropic elastic material whose Young modulus is $10^{10}\, N/m^2$ and Poisson coefficient $0.15$ (note that the fluctuations around the mean are those of a heterogeneous  anisotropic elastic material).
The non-Gaussian $\MM_6^+$-valued  random field  $\{ [\bfA(\bfomega)]\, ,\bfomega\in \Omega \}$ is constructed using the stochastic model  \cite{Soize2006,Soize2008,Soize2017b} of random elasticity fields for heterogeneous anisotropic elastic media that are isotropic in statistical mean and exhibit anisotropic statistical fluctuations, for which the parameterization consists of spatial-correlation lengths and of a positive-definite lower bound. The random field $\{[\bfA(\bfomega)],\bfomega\in\Omega\}$  is written as,
\begin{equation}\label{eq:eq144}
[\bfA(\bfomega)] = \frac{1}{1+\epsilon}\, [\underline\LL]^T \, \bigg ( \epsilon \, [I_6] + [\bfG(\bfomega)] \bigg )\, [\underline\LL]
\quad , \quad \forall\, \bfomega\in\Omega\, ,
\end{equation}
in which $[\underline\LL]$ is the upper triangular $(6\times 6)$ real matrix such that $[\,\underline \bfA\,] = [\underline\LL]^T[\underline\LL]$, where $\epsilon$ is a given positive number (which can be chosen arbitrarily small),
and where $\{ [\bfG(\bfomega)],\bfomega\in\RR^3\}$ is a $\MM^+_6$-valued random field (by construction), defined on $(\Theta,\curT,\curP)$, indexed by $\RR^3$. Then $ [\bfG]$ is homogeneous, mean-square continuous, and such that
$E\{[\bfG(\bfomega))]\} = [I_6]$ for all $\bfomega\in \RR^3$. Note that the lower bound $\epsilon\,[\,\underline\bfA\, ]/(1+\epsilon)$ used in Eq.~\eqref{eq:eq144} could be replaced by a more general lower bound $[A_b]$ in $\MM_6^+$ as proposed in \cite{Guilleminot2013a,Soize2017b}.
For all $\bfomega$ fixed in $\RR^3$, the $\MM^+_6$-valued random variable $[\bfG(\bfomega)]$ has been constructed by using the Maximum Entropy Principle under the following available information, $E\{[\bfG(\bfomega)]\} \! = \! [I_6]$ and $E\{ \log (\det[\bfG(\bfomega)] ) \}\! =\! b_G$ with $\vert b_G  \vert \, < \! +\infty$, which has been introduced in order  that the random matrix $[\bfG(\bfomega)]^{-1}$ (that exists almost surely) be such that $E\{\Vert[\bfG(\bfomega)]^{-1}\Vert^2\} \leq $ $E\{\Vert[\bfG(\bfomega)]^{-1}\Vert_F^2\} < +\infty$.
In this construction, for all $\bfomega$ fixed in $\RR^3$,
$[\bfG(\bfomega)] = [g\big ( \, \{ \curG_{mn}(\bfomega), 1 \! \leq \! m \! \leq \! n \! \leq \! 6\}\, \big ) ]$ is a $\MM_6^+$-valued nonlinear function $[g(.)]$ of $6\times(6+1)/2 = 21$ independent normalized Gaussian real-valued random variables denoted by $\{ \curG_{mn}(\bfomega), 1 \leq m \leq n\leq 6\}$ and such that $E\{\curG_{mn}(\bfomega)\} = 0$ and $E\{\curG_{mn}(\bfomega)^2\} = 1$.
The spatial correlation structure of random field $\{[\bfG(\bfomega)],$ $\bfomega\in\RR^3\}$ is introduced by considering  $21$ independent  real-valued random fields $\{\curG_{mn}(\bfomega),\bfomega\in\RR^3\}$  for $1\leq m \leq n \leq 6$,  corresponding to $21$ independent copies of a unique normalized Gaussian homogeneous mean-square continuous real-valued random field $\{\curG(\bfomega),\bfomega\in\RR^3\}$  whose normalized spectral measure is given and has a support that is controlled by three spatial correlation lengths  $L_{c1} = L_{c2} = L_{c3} = 0.4$.
Note that this Gaussian field $\curG$ can be replaced by a non-Gaussian field for taking into account uncertainties in the spectral measure \cite{Soize2021d}.
The constant $b_G$ is eliminated in favor of a hyperparameter $\delta_G > 0$,  which allows for controlling the level of statistical fluctuations of $[\bfG(\bfomega)]$, defined by $\delta_G =(E\{\Vert [\bfG(\bfomega)] - [I_6]\Vert_F^2\} / 6)^{1/2}$,
which is independent of $\bfomega$ and such that $\delta_G= 0.6$.
\subsection{Stochastic computational model for generating the training set $D_d$ and observed quantities of interest}
\label{appendix:A.2}
The stochastic boundary value problem defined by Eqs.~\eqref{eq:eq140} to \eqref{eq:eq143} is discretized by the finite element method. Domain $\Omega$ is meshed with $50\!\times\! 10\!\times\! 5 = 2\,500$ finite elements
using $8$-nodes finite elements. There are $3\,366$ nodes and $10\,098$ dofs (degrees of freedom).
The displacements are locked  at all the $66$ nodes belonging to surface $\Gamma_0$ and therefore, there are $198$ zero Dirichlet conditions. There are $8$ integration points in each finite element. Consequently, there are  $N_p= 20\, 000$ integration points $\bfomega^1,\ldots, \bfomega^{N_p}$.
The $\RRnw$-valued random variable $\bfW$ is generated as follows.
For all $p=1,\ldots, N_p$, let $[ \bfG_p^{\rm{log}} ] = \log_M ( [\bfG(\bfomega^p)] ) \in \MM_6$ in which $\log_M$ is the logarithm of positive-definite matrices. The $\RRnw$-valued random variable $\bfW$ is then defined as the vector that is the reshaping of the upper triangular part of the $N_p$ matrices $\{\, [\bfG_p^{\rm{log}}], p=1,\ldots , N_p\}$.We then have $n_w = 21\times N_p = 420\, 000$.
The finite element discretization of random field $\{\UU(\bfomega) ,\bfomega\in \overline\Omega\,\}$ is the $\RRnq$-valued random variable $\bfQ$ with $n_q= 10\,098$. Consequently $\bfX=(\bfQ,\bfW)$ is a random variable with values in $\RRnx$ with $n_x=n_q+n_w = 430\,098$. The stochastic computational model is then represented by a stochastic linear matrix equation that is solved by using the Monte Carlo numerical simulation method yielding the training set $D_d = \{ \bfx^1,\ldots , \bfx^{N_d}\}$ in which $\bfx^j\!\! = \!(\bfq_d^j,\bfw_d^j) \in \RRnx \!\! = \!\RRnq\!\times\!\RRnw$ is a realization of random variable $\bfX=(\bfQ,\bfW)$, the computed realizations being independent. For studying the convergence properties, the considered values of $N_d$  are $N_d\in\{100, 200, 300, 400 \}$.

The components of the quantity of interest $\bfQ$, which will be observed for presenting the results, are the $3$ components denoted by
$Q_{\pobs,1}$, $Q_{\pobs,2}$, and $Q_{\pobs,3}$ that correspond to the $3$ dofs along directions $\omega_1$, $\omega_2$, and $\omega_3$ of the finite element node of coordinates $(0,0,0.1)$ (located at top corner in which the displacements are significant and result from tension, torsion, and two bendings contributions).
\subsection{Target set of realizations}
\label{appendix:A.3}
The target set $D_\pexp = \{\bfq_\pexp^1,\ldots \bfq_\pexp^{N_r}\}$ is generated using the stochastic boundary value problem defined in Section~\ref{appendix:A.1}  for which  the elasticity matrix $[\underline\bfA^\pexp]$ is the one of a homogeneous and isotropic elastic material with a Young modulus $9\times 10^9\, N/m^2$ and a Poisson coefficient $\nu=0.15$. The level of statistical fluctuations of the random field $\{\bfG^\pexp(\bfomega),\bfomega\in\RR^3\}$ is $\delta_G^\pexp = 0.3$. In order to analyze the convergence with respect to $N_r$, we have considered, in consistency with the values of $N_d$, the intervals $N_r \in [50\, , N_\pexp]$ with
$N_\pexp\in\{100,200,300,400\}$.
\section*{Conflict of interest}

The author declares that he has no conflict of interest.

%-------- REFERENCES-----------------------------------------------------------------

%\bibliographystyle{elsarticle-num}
%\bibliography{reference}

%
%-------- REFERENCES-----------------------------------------------------------------

\end{document}